%% file: main.tex
\documentclass[twoside,11pt]{article}

\usepackage[preprint]{jmlr2e}

\input{preamble.tex}

\usepackage{lastpage}

\ShortHeadings{Nystr\"{o}m Kernel Stein Discrepancy Tests}{Kalinke, Szabó, and Sriperumbudur}
\firstpageno{1}

\newcommand{\lemmaprojectionperspectivetext}{Equivalence of low-rank approximation and projection}

\begin{document}

\title{Nystr\"{o}m Kernel Stein Discrepancy Tests}

\author{\name Florian Kalinke \email florian.kalinke@kit.edu \\
       \addr Chair of Information Systems\\ Karlsruhe Institute of Technology\\
       Am Fasanengarten 5, 76131 Karlsruhe, Germany
       \AND
       \name Zoltán Szabó \email z.szabo@lse.ac.uk \\
       \addr Department of Statistics\\ London School of Economics\\
       Houghton Street, London, WC2A 2AE, UK
       \AND
       \name Bharath K.\ Sriperumbudur \email bks18@psu.edu \\
       \addr Department of Statistics\\The Pennsylvania State University\\
       University Park, PA 16802, USA}

\maketitle

\begin{abstract}%
	Kernel Stein discrepancy (KSD) is among the most popular goodness-of-fit (GoF) measures on general domains with a large number of successful deployments. One of the main applications of KSD is in constructing powerful GoF tests. However, tests relying on the classical U-/V-statistic-based KSD estimators have two major drawbacks. (i) Their runtime scales quadratically in the number of samples. (ii)~Their asymptotic null distribution is computationally intractable in most cases, typically handled by bootstrapping.
    While it is known that the Nystr\"{o}m method permits accelerating KSD estimation with no loss of statistical accuracy under mild conditions, to the best of our knowledge, the fundamental question of its impact on bootstrap-based GoF testing is open; resolving this question is the focus of the current paper. In particular, we prove that the key properties of the quadratic-time bootstrapped KSD-based GoF test (asymptotic level and local consistency) are preserved by its Nystr\"{o}m acceleration. 
    We numerically demonstrate the efficiency of the accelerated KSD estimator and bootstrap in the context of GoF testing of spherical and functional data. 
    Our numerical results show that the Nystr\"{o}m-accelerated method performs statistically on-par with the quadratic-time approach, while requiring substantially smaller runtime.
\end{abstract}


\section{Introduction} \label{sec:introduction}

Testing for goodness-of-fit (GoF) is a fundamental problem in data science and statistics \citep{ingster03goftesting,lehmann21testing}. Given a (fixed known) target distribution $\P_0$ and samples of an (unknown) sampling distribution $\P$, the aim is to decide if the samples come from the target, that is, to test $H_0 : \P=\P_0$ versus $H_1 : \P\neq \P_0$. While famously tackled by the  Kolmogorov-Smirnov test  \citep{kolmogorov33sulla,smirnov48gof} using the empirical distribution function, more recent GoF tests---using, for example, nearest-neighbor-based statistics \citep{bickel83gofnearestneighbor,schilling83gof,schilling83gofextension}, innovation processes \citep{khmaladze88innovation,khmaladze93martingales}, or optimal transport \citep{hallin21wassersteingof}---exist. However, existing tests share at least one of the following drawbacks \citep{hagrass24stein}: (i) $\P_0$ is subject to restrictive distributional assumptions, for example, $\P_0$ is assumed to belong to a parametric family. (ii) They assume data in $\R^d$. (iii) Full knowledge of $\P_0$ is required; they do not allow distributions known only up to a normalizing factor. Such partially known distributions typically arise in Bayesian settings or in restricted Boltzmann machines (RBMs). (iv) They are challenging to compute efficiently.

The highly flexible framework of kernel methods \citep{aronszajn50theory,steinwart08support}, in particular, the associated kernel mean embedding \citep{berlinet04reproducing,smola07hilbert,gretton12kernel,muandet17review}, permits mapping probability distributions into a reproducing kernel Hilbert space (RKHS) and allows comparing two measures by considering the distance of their embeddings in RKHS norm, giving rise to the maximum mean discrepancy (MMD; \citealt{gretton12kernel}). MMD also arises as an integral probability metric \citep{zolotarev83probability,muller97integral} when the underlying function class is chosen to be the unit ball in an RKHS; it is known to be equivalent \citep{sejdinovic13equivalence} to energy distance \citep{baringhaus04new,szekely04testing,szekely05new} (respectively known as $N$-distance; \citealt{zinger92characterization,klebanov05ndistance}). If the kernel associated with the RKHS is characteristic \citep{fukumizu07kernel,sriperumbudur10hilbert}, the mean embedding is injective, and MMD is a metric, which allows designing non-parametric tests that are consistent against all alternatives.

Indeed, kernel mean embeddings allow designing powerful GoF tests \citep{balasubramanian21minimaxgof,hagrass23spectralgof}, addressing (i) and (ii).
A novel approach to tackle GoF testing builds on kernel Stein discrepancies (KSDs; \citealt{chwialkowski16kernel,liu16kernelized}), which combine Stein's method \citep{stein72bound,chen21stein,anastasiou23stein} with the expressivity and computational tractability of RKHSs.
These tests build on a Stein operator---associated to the known $\P_0$---that acts on functions in an RKHS and consider the RKHS-norm of the mean embedding (w.r.t.\ the sampling distribution $\P$) of functions mapped by this operator. One particular Stein operator, the Langevin-Stein operator \citep{gorham15measuring,chwialkowski16kernel,liu16kernelized,oates17control,gorham17measuring}, allows handling not necessarily normalized densities, thereby addressing (i) and (iii).

Since the inception of KSDs, several other GoF tests with corresponding kernel Stein operators have been designed, additionally addressing (ii); for example, for spherical and Riemannian manifolds \citep{xu20directionalksd,xu21gofriemannianmanifold}, (infinite-dimensional) functional data \citep{wynne25fourierksd}, discrete distributions \citep{yang18discrete}, point processes \citep{yang19point}, random graphs \citep{xu21gofgraph}, sequential models \citep{baum23ksd}, and time-to-event data \citep{fernandez20ksdgof}. Related works independent of the underlying domain include 
a sequential GoF testing approach \citep{martinez25sequential},
minimax-optimal KSD-based GoF tests \citep{hagrass24stein}, 
aggregating KSD-based tests over different kernel choices to maximize test power \citep{schrab22ksdagg}, and an analysis of the robustness of KSD-based GoF tests \citep{liu25ksdrobustness}.

These tests enjoy a broad set of real-world applications. They permit model validation in challenging settings, for example, deciding whether a Markov chain was drawn from a given distribution, or improving sample quality \citep{chwialkowski16kernel,gorham17measuring,liu17importancesampling}. They have successfully been applied in the context of the Ising model \citep{yang18discrete}, on network data \citep{yang18discrete,yang19point,fatima25goftest}, for testing if samples come from a generative model \citep{jitkrittum17lineartime,schrab22efficient}, for model comparison \citep{lim19kernel,kanagawa23latent}, and for handling censored data in medical contexts \citep{fernandez20ksdgof}. KSD-based GoF tests on Riemannian manifolds enjoy applications in the analysis of wind data and of vectorcardiogram data \citep{xu21gofriemannianmanifold}.

Despite their success, KSD-based GoF tests suffer from two major drawbacks. First, the classical U- \citep{liu16kernelized} and V-statistic-based (V-KSD; \citealt{chwialkowski16kernel}) estimators have a runtime that scales quadratically with the number of observations, which hinders their application in large-scale settings, violating (iv). Second, the asymptotic distribution of these estimators under the null hypothesis is often analytically intractable. To alleviate this analytical intractability, the limiting distribution is typically approximated using the wild (also known as i.i.d.\ weighted if the samples are i.i.d.) bootstrap \citep{arcones92bootstrap,dehling94bootstrap}, which, however, further increases the total runtime.  

Some mitigations reducing the runtime and targeting (iv) exist. If the kernel function is analytic, an option is using a linear-time estimator \citep[Appendix~5.1]{chwialkowski16kernel}, but this approach may lead to a loss in test power when compared to the quadratic-time counterpart \citep{chwialkowski15fast}. An independent approach consists of replacing the RKHS-norm by an $L_p$-norm ($p=2$: \citealt{jitkrittum17lineartime}; $p\ge2$: \citealt{huggins18random}), which leads to a different statistic and also requires a good choice of parameters to work successfully. A special case of \citet{huggins18random} yields a random Fourier feature-based (RFF; \citealt{rahimi07random,sriperumbudur15optimal,szabo19kernel,chamakh20orlicz}) approximation of KSD. But, as elaborated by \citet{huggins18random}, if RFF is applied to KSD-based GoF tests, it is known \citep[Proposition~1]{chwialkowski15fast} that the resulting statistic can not distinguish a large class of measures. 
An alternative to the RFF-based approach not requiring the sampling and the target distribution to live on $\R^d$ is the Nystr\"{o}m method \citep{nystroem30method,williams01using,rudi15less,chatalic22nystrom,kalinke23nystrommhsic}; it permits the acceleration of the estimation of KSD and KSD-based tests on general domains (N-KSD; \citealt{kalinke25nystromksd}). The analysis of Nystr\"{o}m-accelerated KSD-based GoF tests is the main focus of this work.

In particular, our main \tb{contributions} are as follows.

\begin{enumerate}[label=(\roman*)]
	\item We establish---under a classical sub-Gaussian assumption---that KSD-based GoF tests using the Nystr\"{o}m-accelerated wild bootstrap \citep{kalinke25nystromksd} are locally consistent. Moreover, we show that their separation rate matches that of the quadratic-time approach, while enjoying lesser computational complexity, i.e., the Nystr\"{o}m KSD test is computationally cheap with no loss of testing power compared to that of the exact KSD test.\footnote{A related analysis, but for two-sample testing, has recently been carried out \citep{chatalic25nystroem}. Further, the analysis considers a leverage-score-based Nystr\"{o}m approximation and a permutation-based approximation of the null distribution, which is different from the setting that we consider.}
	For the readers' convenience, our main theoretical results on Nystr\"{o}m-KSD with their quadratic-time counterpart are summarized in Table~\ref{table:overview-main-results}.
	
	\item Extending the testing experiments on $\R^d$ \citep[Section~5]{kalinke25nystromksd}, we present accelerated GoF tests on directional data and functional data, expanding the quadratic-time experiments of \citet{xu20directionalksd} and \citet{wynne25fourierksd}, respectively. These studies numerically validate our theoretical results, that is, we obtain the same power with a reduced runtime.
\end{enumerate}

Along the way, to prove our main results, we establish a few auxiliary results that can be of independent interest, which are summarized in Table~\ref{tab:overview-auxiliary}.

\begin{table}
	\centering
	\begin{tabular}{lll}
		\toprule
		Summary & V-KSD  & N-KSD  \\
		\midrule
		$\sqrt n$-consistency of estimator & Theorem~\ref{thm:v-stat-consistency2}\tablefootnote{This result is known \citep[Theorem~3]{kalinke25nystromksd} under a slightly stronger assumption.} & Theorem~\ref{corr:decay-assumption}\tablefootnote{This result is known \citep[Corollary~1]{kalinke25nystromksd} and included here for self-containedness.}  \\
		Asymptotic distribution of bootstrap & Theorem~\ref{theorem:bootstrap-v-stat} & Theorem~\ref{thm:limit-nystroem-ksd} \\
		Local consistency of GoF test & Theorem~\ref{thm:separation-ksd} & Theorem~\ref{thm:separation-n-ksd}  \\
		\bottomrule
		
	\end{tabular}
	\caption{Overview of our main theoretical results. 
    V-KSD: quadratic-time V-statistic-based KSD estimator; N-KSD: Nystr\"{o}m-accelerated KSD estimator.} \label{table:overview-main-results}
\end{table}

\begin{table}
	\centering
	\begin{tabular}{ll}
		\toprule
		Summary & Result \\
		\midrule
		Properties of Stein operator & Proposition~\ref{lemma:stein-op-properties} \\
		Asymptotic distribution of V-KSD & Theorem~\ref{thm:serfling-v-stat} \\
		Range space equivalences & Lemma~\ref{lemma:range-space-equivalences} \\
		\lemmaprojectionperspectivetext & Lemma~\ref{lemma:projection-perspective} \\
		Consistency of Nystr\"{o}m-accelerated bootstrap  & Theorem~\ref{thm:consistency-nystroem-bootstrap} \\
		\bottomrule
	\end{tabular}
	\caption{Overview of our main auxiliary results.} \label{tab:overview-auxiliary}
\end{table}

The remainder of this article is \tb{structured} as follows. Our notations are introduced in Section~\ref{sec:notations}. We recall KSD, its existing U- and V-statistic-based estimators, and their application to GoF testing in Section~\ref{sec:problem-definition}. The Nystr\"{o}m-accelerated KSD and GoF test, with our theoretical results, are presented in Section~\ref{sec:results}. Experiments demonstrating the efficiency of the Nystr\"{o}m-KSD estimator in non-Euclidean contexts are provided in Section~\ref{sec:experiments}. The proofs of the statements in the main text are provided in Section~\ref{sec:proofs}; auxiliary and external results are collected in the appendices.

\section{Notations} \label{sec:notations}
In this section, we introduce the notations 
$\Np$, %
$[n]$, %
$|S|$, %
$\O$, %
$o$, %
$\Omega$, %
$\omega$,
$\lesssim$, %
$\gtrsim$, %
$\asymp$, %
$\bm 1_{\{\cdot\}}$, %
$\langle \cdot,\cdot\rangle$, %
$\|\cdot\|_2$, %
$\bm 1_d$,  %
$\b A^-$, %
$\nabla$, %
$\Span$, %
$\opnorm{\cdot}$, %
$\L(\H)$, %
$\range$, %
$A^*$, %
$A^{\frac{1}{2}}$, %
$\|A\|_1$, %
$\trace(A)$, %
$\bar{S}$, %
$f\otimes_\H g$, %
$\H \otimes \H$, %
$P_U$, %
$\mathcal B\!\left( \tau_\X \right)$, %
$\mathcal M_1^+$, %
$\P\otimes \Q$, %
$\P^n$, %
$\operatorname{supp}$,
$\delta_x$, %
$\operatorname{Unif}$, %
$\mathcal N$, %
$\E_\P X$, %
$\norm{\cdot}{L^r(\Omega,\P)}$, %
$\mathcal L^r(\Omega,\P)$, %
$L^r(\Omega,\P)$, %
$\norm{\cdot}{\psi_r}$, %
$O_P$, %
$o_P$, %
$\dfrac{\d \P}{\d \Q}$, %
$\chi^2$, %
$\dconv$, %
$\norm{\cdot}{\infty}$, %
$\mC(\X)$, %
$\H_k$, %
$k$, %
$\varphi_k$, %
$\mu_k$, %
$C_{\P,k}$, %
$I$, %
$C_{\P,k,\lambda}$, %
$\I_{\P,k}$, %
$\I_{\P,k}^*$, %
$[\cdot]_{\sim}$, %
$T_{\P,k}$, %
$T_{\P,k}^\theta$. %

The set of positive integers is denoted by $\Np$. For $n \in \Np$, let  $[n] = \{1,2,\ldots,n\}$. The cardinality of a set $S$ is denoted by $|S|$. For some set $D\subseteq \R$ and functions $f,g : D \to \R$, we write $f(x) = \O(g(x))$ if and only if there exist constants $M>0$ and $x_0 \in D$ such that for all $x\ge x_0$ ($x\in D$) it holds that $|f(x)|\le M|g(x)|$; similarly, $f(x) = o(g(x))$ if and only if for all $\epsilon > 0$ there exists a constant $x_0 \in D$  such that $|f(x)|\le \epsilon |g(x)|$ for all $x\ge x_0$. We write $f(x) = \Omega(g(x))$ iff $g(x) = O(f(x))$ and $f(x)=\omega(g(x))$ iff $g(x)=o(f(x))$.
For some set $S$, and functions $f,g :S\to [0,\infty)$, $f(x) \lesssim g(x)$ (resp.\ $f(x) \gtrsim g(x)$) means that there exists $M >0$ such that $f(x) \le M g(x)$ (resp.\ $g(x) \ge Mf(x)$) for all $x \in S$. If $f(x) \lesssim g(x)$ and $f(x)\gtrsim g(x)$, we write $f(x) \asymp g(x)$.
$\bm 1_{\{\cdot\}}$ denotes the indicator function: for a set $A$, $\bm 1_A(x)=1$ if $x\in A$, and $\bm 1_A(x)=0$ otherwise. The inner product of the vectors $\b a$ and $\b b\in \R^d$ is $\langle \b a, \b b\rangle = \sum_{i\in [d]}a_i b_i$; the Euclidean norm of $\b a\in \R^d$ is $\left\|\b a\right\|_2 = \sqrt{\langle \b a, \b a\rangle}$. The $d$-dimensional vector of ones is $\bm 1_d \coloneq (1,\ldots,1)\in\R^d$. For a matrix $\b A \in\R^{d_1\times d_2}$, $\b A^- \in \R^{d_2\times d_1}$ denotes its Moore-Penrose inverse. For a differentiable function $f:\R^d \to \R$, $\nabla f(\b x) = \left[\frac{\partial f(\b x)}{\partial x_i}\right]_{i=1}^d \in \R^d$ ($\b x \in \R^d$).

Let $\H$ be a separable Hilbert space. We write $\Span(M)$ for the linear hull of $M\subseteq \H$. A linear operator $A : \H \to \H$ is called bounded if  $\opnorm{A} \coloneq \sup_{\norm{h}{\H}=1}\norm{Ah}{\H} < \infty$; the set of $\H\rightarrow \H$ bounded linear operators is denoted by $\L(\H)$. We write $\range (A) = \{Ah : h\in \H\}$ for the range of $A\in\L(\H)$.  An $A \in \L(\H)$ is called positive (shortly $A\ge 0$) if  it is self-adjoint ($A^*=A$, where $A^*\in \L(\H)$ is defined by $\langle Af,g \rangle_{\H} = \langle f,A^*g \rangle_{\H}$ for all $f,g\in \H$), and $\langle Ah,h\rangle_{\H} \ge 0$ for all $h\in \H$. If $A\ge 0$, then there exists a unique $B \ge 0$ such that $B^2 = A$; we write $B = A^{\frac{1}{2}}$ and call $B$ the square root of $A$. An $A \in \L(\H)$ is called trace-class if $\|A\|_1\coloneq\sum_{j\in J} \langle (A^*A)^{\frac{1}{2}}e_j,e_j \rangle_{\H} <\infty$ for some countable orthonormal basis (ONB) $(e_j)_{j\in J}$ of $\H$, and in this case $\trace(A)\coloneq\sum_{j\in J} \langle Ae_j,e_j \rangle_{\H} <\infty$.\footnote{The trace-class property and the value of $\trace(A)$ is independent of the specific ONB chosen. The separability of $\H$ implies the existence of a countable ONB in it.} If $A \in  \L(\H)$ is positive, then $\|A\|_1=\trace(A)$. It is known that $ \opnorm{A} \le \|A\|_1$. For a self-adjoint trace-class operator $A$ with eigenvalues $(\lambda_j)_{j\in J}$, $\trace(A)=\sum_{j\in J} \lambda_{j}$. An operator $A \in \L(\H)$ is called compact if $\overline{\{Ah :  \left\|h\right\|\le 1, h\in\H\}}$ is compact, where $\overline{\cdot}$ denotes the closure. A trace-class operator is compact.
We write $f\otimes_\H g\in \H\otimes \H$ for the tensor product of $f,g \in \H$ and $\H\otimes \H$ denotes the tensor product Hilbert space. Particularly, $f\otimes_\H g : \H \to \H \in \L(\H)$ defines a rank-one operator by $h\mapsto f\ip{g,h}{\H}$; $\H \otimes \H = \overline{\Span}\!\left(f \otimes_\H g\,:\, f,g\in \H\right)$, where the closure is meant w.r.t.\ to the linear extension of the inner product $\langle f_1\otimes g_1, f_2 \otimes g_2\rangle_{\H \otimes \H} = \langle f_1,f_2\rangle_\H \langle g_1,g_2\rangle_\H $. For a closed linear subspace $U \subseteq \H$, we denote by $P_Uh \in U$ ($h\in\H$) the orthogonal projection of $h$ onto~$U$.

Let $\left( \X, \tau_\X \right)$ be a topological space and $\mathcal B\!\left( \tau_\X \right)$ the Borel sigma-algebra induced by $\tau_\X$. We write $\mathcal M_1^+(\X)$ for the set of probability measures defined w.r.t.\ the measurable space $\left( \X,\mathcal B\left( \tau_\X \right) \right)$. For $\P,\Q \in \mathcal M_1^+(\X)$, $\P\otimes\Q \in \mathcal M_1^+(\X\times \X)$ denotes the product measure, $\P^n \coloneq \otimes_{i=1}^n \P$ is the $n$-fold product of $\P$, and $\operatorname{supp}(\P)$ denotes the support of $\P$. The Dirac measure at $x \in \X$ is  $\delta_x \in \mathcal M_1^+(\X)$. For a set $A\subseteq\X$ with cardinality $|A|<\infty$, $\operatorname{Unif(A)} = \frac{1}{|A|}\sum_{a\in A}\delta_a$ is the discrete uniform measure. $\mathcal N(\mu,\sigma^2)$ denotes the normal distribution with mean $\mu$ and variance $\sigma^2$.
For a Hilbert space-valued random variable $X : \left( \Omega, \mathcal F, \P \right) \to \left( \H, \mathcal B\left( \tau_\H \right) \right)$, we use the shorthand $\E_{\P} X \coloneq \E_{X\sim \P}X \coloneq \int_{\Omega}X(\omega) \d \P(\omega)$, where the integration is meant in Bochner's sense \citep[Chapter~II.2]{diestel77vector}. For a real-valued random variable $X : \left( \Omega, \mathcal F, \P \right) \to \left( \R, \mathcal B\left( \tau_\R \right) \right)$ and $r\ge1$, $\norm{X}{L^r(\Omega,\P)} \coloneq \left[\int_\Omega|X(\omega)|^r\d\P(\omega)\right]^{\frac{1}{r}}$ and $\norm{X}{\psi_r} \coloneq \inf\left\{t>0 : \E_{\P}\exp\left( \frac{|X|^r}{t^r}\right) \le 2 \right\}$; $L^r(\Omega,\P)\coloneq\{X : \left( \Omega, \mathcal F, \P \right) \to \left( \R, \mathcal B\left( \tau_\R \right) \right)\,:\, \norm{X}{L^r(\Omega,\P)}<\infty\}$; for $r\in\{1,2\}$, $X$ is called sub-exponential (resp.\ sub-Gaussian) iff $\psione{X} < \infty$ (resp.\ $\psitwo{X} < \infty$). $\norm{\cdot}{\psi_r}$ is monotone for $r\ge 1$: for $0\le X \le X'$, one has $\norm{X}{\psi_r} \le \norm{X'}{\psi_r}$ where $X$ and $X'$ are real-valued random variables. For $r\ge 1$, $\mathcal L^r(\Omega,\P)$ denotes the space of real-valued measurable functions on $\Omega$ whose $r$-th absolute power is integrable w.r.t.\ $\P$; $L^r(\Omega,\P)$ is the space of equivalence classes in $\mathcal L^r(\Omega,\P)$, where two functions are considered to be identical if they are equal $\P$-almost everywhere.
Given a (non-random) sequence $(r_n)_n >0$ and a sequence of real-valued random variables $(X_n)_n$, we write $X_n = O_P(r_n)$ (resp. $X_n = o_P(r_n)$), if $\left(\frac{X_n}{r_n}\right)_n$ is bounded in probability (resp.\ converges to zero in probability).
If $\P \ll \Q$ ($\P,\Q\in\mathcal M_1^+(\X)$), that is, $\P$ is absolutely continuous w.r.t.\ $\Q$, we write the Radon-Nikodym derivative of $\P$ w.r.t.\ $\Q$ as $\dfrac{\d \P}{\d \Q}$, and their $\chi^2$-divergence is
\begin{align}
	\chi^2(\P,\Q) = \norm{\dfrac{\d \P}{\d \Q}-1}{L^2(\X,\Q)}^2. \label{eq:chi-square-definition}
\end{align}
For a sequence of random variables $X_n\sim\P_{n}\in\mathcal M_1^+(\X)$ taking values in a metric space~$\X$, $X_n\dconv X\sim \P \in\mathcal M_1^+(\X)$ indicates their distributional (weak) convergence. 

The supremum norm of a function $f:\X \to \R$ is $\norm{f}{\infty} \coloneq \sup_{x\in\X}|f(x)|$. For a topological space $(\X,\tau_\X)$, let $\mC(\X)$ stand for the set of continuous real-valued functions on $\X$ endowed with the supremum norm $\norm{\cdot}{\infty}$. Let $\H_k$ denote the reproducing kernel Hilbert space associated with a kernel $k: \X \times \X \to \R$; $\H_k$ is the Hilbert space of $\X \to \R$ functions, where (i) $k(\cdot,x) \in \H_k$ for all $x \in \X$,\footnote{$k(\cdot,x)$ stands for the function $\X \ni x' \mapsto k(x',x) \in \R$ while keeping $x \in \X$ fixed.} and (ii) $f(x) = \langle f,k(\cdot,x)\rangle_{\H_k}$ for all $f\in \H_k$ and $x \in \X$. Equivalently, a function $k: \X \times \X \rightarrow \R$ is called kernel if there exists a Hilbert space $\H$ and a feature map $\varphi:\X \to \H$ such that $k(x,x') = \langle \varphi(x), \varphi(x')\rangle_{\H}$ for all $x,x'\in \X$.  The mapping $\varphi_k: \X \to \H_k$, $x \mapsto k(\cdot,x)$ is called the canonical feature map. There is a one-to-one correspondence between kernels and RKHSs, and one can choose $\H=\H_k$ and $\varphi=\varphi_k$ in the definition of kernels.
For $\P\in\mathcal M_1^+(\X)$, the (kernel) mean embedding of $\P$ w.r.t.\ $k$ is defined as
\begin{align}
	\mu_k(\P) = \E_\P k(\cdot,X) = \int_\X k(\cdot,x)\d\P(x) \in \H_k.\label{eq:mean-embedding}
\end{align}
The mean embedding exists if $\int_{\X}\norm{k(\cdot,x)}{\H_k}\d\P(x) = \int_{\X}\sqrt{k(x,x)}\d \P(x) < \infty$ \citep[p.~45, Theorem~2]{diestel77vector}. 
Similarly, one can define the covariance operator $C_{\P,k} : \H_k  \to \H_k$ of $\P\in\mathcal M_1^+(\X)$ w.r.t.\ $k$ as
\begin{align}
	C_{\P,k} = \int_{\X}k(\cdot,x)\otimes_{\H_k}k(\cdot,x)\d\P(x), \label{eq:def-cov-operator}
\end{align}
which exists if $\int_{\X}\norm{k(\cdot,x)}{\H_k}^2\d\P(x) =\int_{\X} k(x,x)\d \P(x) <\infty$; $C_{\P,k}$ is a positive trace-class operator. We define $C_{\P,k,\lambda}\coloneq C_{\P,k} + I\lambda$ for $\lambda >0$, where $I$ denotes the identity operator.

For a measurable space $\X$, a kernel $k:\X \times \X \to \R$ and a probability measure $\P \in \mathcal M_1^+(\X)$, let $\I_{\P,k} : \H_{k} \to L^2(\X,\P)$, $g\mapsto [g]_\sim$ be the inclusion, where $[g]_\sim$ denotes the equivalence class of $g$ in $L^2(\X,\P)$. Then it can be shown \citep[Theorem~4.26]{steinwart08support} that $\I_{\P,k}^* : L^2(\X,\P) \to \H_{k}$, $f\mapsto \int_\X k(\cdot,x)f(x)\d \P(x)$. Define the integral operator as $T_{\P,k} = \I_{\P,k}\I_{\P,k}^*$.
If $\E_{\P} k(X,X) <\infty $, then $T_{\P,k}$ is positive and trace-class \citep[Theorem~4.27]{steinwart08support}. Hence, the spectral theorem implies that there exists a countable orthonormal system (ONS) $\big(\tilde{\phi}_j\big)_{j\in J}\subset L^2(\X,\P)$ and  $(\lambda_j)_{j\in J} \subset \R$ converging to zero\footnote{This means that either $|J|<\infty$, or $\lim_{j\to \infty} \lambda_j = 0$ if $J$ is countable.} such that
$\lambda_1\ge\lambda_2\ge\ldots>0$ and 
\begin{align}
	T_{\P,k} = \sum_{j \in J} \lambda_j \tilde \phi_j \otimes_{L^2(\X,\P)} \tilde \phi_j,\label{eq:spectral-decomposition}
\end{align}
where $\left(\lambda_j\right)_{j\in J}$ are the eigenvalues and $\big(\tilde{\phi}_j\big)_{j\in J}$ are the corresponding eigenvectors of $T_{\P,k}$. 
In particular, \eqref{eq:spectral-decomposition} implies that $\big(\tilde{\phi}_j\big)_{j\in J}\subset L^2(\X,\P)$ forms an ONB of $\overline{\range\!\left(T_{\P,k}\right)}$. 
If $k$ is continuous and $\E_{\P} k(X,X) <\infty$, it is known \citep{steinwart12mercer} that in \eqref{eq:spectral-decomposition} one can choose continuous representatives $(\phi_j)_{j\in J}  \subset \H_k \subset \mC(\X)$ such that $\mathfrak I_{\P,k}\phi_j = [\phi_j]_\sim = \tilde \phi_j \in L^2(\X,\P)$.
For $\theta\ge0$, the fractional power of $T_{\P,k}$ is defined as
\begin{align}
	T_{\P,k}^\theta = \sum_{j \in J} \lambda_j^\theta \tilde \phi_j \otimes_{L^2(\X,\P)} \tilde \phi_j. \label{eq:integral-operator}
\end{align}
If $\theta < 0$, then $T_{\P,k}^\theta$ is defined\footnote{See the discussion by \citet{cucker02learning} before Theorem~3 in Chapter~II, Section~2.} as in \eqref{eq:integral-operator}---taking into account the fact the $\lambda_j>0$ for all $j\in J$---but on the subspace
\begin{align}
S_{T_{\P,k}^\theta}\coloneq\left\{\sum_{j\in J}a_j\tilde \phi_j : \sum_{j\in J}\left(a_j\lambda_j^\theta\right)^2 < \infty\right\} \subset L^2(\X,\P). \label{eq:S-def}
\end{align}

\section{Background} \label{sec:problem-definition}

In this section, we recall the general KSD framework (Section~\ref{sec:background-ksd}) of \cite{hagrass24stein} and KSD-based goodness-of-fit (GoF) testing together with an existing result on the consistency of KSD-based GoF-testing against fixed alternatives in Euclidean spaces (Section~\ref{sec:background-gof}).

\subsection{Kernel Stein Discrepancy} \label{sec:background-ksd}
Let $X\sim \Q \in \mathcal M_1^+(\X)$, $\H$ a Hilbert space of functions on $\X$, and $\Psi_\Q:\X \to \H$ such that 
\begin{align}
\E_\Q\Psi_\Q(X) = 0
\end{align}
holds.\footnote{The existence of the l.h.s.\ requires that $\E_\Q\norm{\Psi_\Q(X)}{\H} < \infty$ by the properties of the Bochner integral. KSDs can also be defined using the Pettis integral \citep{barp24targeted}, which, for simplicity, we do not consider in this paper.} Define the Stein operator $\mathcal T_\Q$ on $\H$ associated to $\Q$ (via $\Psi_\Q$) as 
\begin{align}
	\left( \mathcal T_\Q f \right)(x) = \ip{\Psi_\Q(x),f}{\H},\quad (f\in \H). \label{eq:mean-zero-property}
\end{align}
The Stein operator $\mathcal T_\Q$ is linear by the linearity of the inner product. The following proposition (proved in Section~\ref{sec:proof-lemma-stein-op-properties}) shows that $\mathcal T_\Q$ inherits various properties from $\Psi_\Q$.

\begin{proposition}[Properties of the Stein operator]\label{lemma:stein-op-properties}
	Let $\mathcal T_\Q$ be as in \eqref{eq:mean-zero-property}. Then,
	\begin{enumerate}[label=(\roman*)]
		\item if $x \mapsto \norm{\Psi_\Q(x)}{\H}\in \mathcal L^r(\X,\Q')$ ($r\in[1,\infty)$, $\Q'\in\mathcal M_1^+(\X)$), it holds that $\range(\mathcal T_\Q) \subseteq \mathcal L^r(\X,\Q')$;
		\item if $\norm{x \mapsto \norm{\Psi_\Q(x)}{\H}}{\infty} < \infty$, it holds that $\norm{\mathcal T_\Q f}{\infty} < \infty$ for all $f\in\H$;
		\item if $\X$ is a metric space and $\Psi_\Q$ is Hölder continuous, it holds that $\mathcal T_\Q f$ is Hölder continuous with the same parameters for all $f\in\H$.
	\end{enumerate}
	In particular, if the assumptions in (i) (resp.\ (ii)) and (iii) are satisfied, then $\range(\mathcal T_\Q)$ contains $r$-integrable (resp.\ bounded) Hölder continuous functions on $\X$.
\end{proposition}

We now proceed with the construction of KSD and note that the operator $\mathcal T_\Q$ satisfies
\begin{align}
	\E_\Q\left[ \left( \mathcal T_\Q f\right)(X) \right] &= \ip{\E_\Q \Psi_\Q(X),f}{\H} = 0,
\end{align}
by interchanging the inner product with the expectation \citep[(A.32)]{steinwart08support} and  using that $\E_\Q\Psi_\Q(X) = 0$.
Consider a fixed (known) target measure $\P_0 \in\mathcal M_1^+(\X)$ and an unknown sampling measure $\P\in\mathcal M_1^+(\X)$. The KSD of $\P_0$ and $\P$ is then defined as the integral probability metric leveraging the above construction and substituting $\Q$ by $\P_0$,
\begin{align}
	\D &\coloneq \sup_{\substack{f\in\H\\\norm{f}{\H} \le 1}}\big|\underbrace{\E_{\P_0}\!\left( \mathcal T_{\P_0} f \right)(X)}_{=0} - \E_{\P}\!\left( \mathcal T_{\P_0} f \right)(X)\big| 
	\stackrel{(a)}{=} \sup_{\substack{f\in\H\\\norm{f}{\H} \le 1}} \E_{\P}\!\left( \mathcal T_{\P_0} f \right)(X) \\
    &\stackrel{\eqref{eq:mean-zero-property}}{=}\sup_{\substack{f\in\H\\\norm{f}{\H} \le 1}}  \E_{\P}\ip{\Psi_{\P_0}(X),f}{\H}  \stackrel{(b)}{=} \sup_{\substack{f\in\H\\\norm{f}{\H} \le 1}}  \ip{\E_{\P}\Psi_{\P_0}(X),f}{\H} \stackrel{(c)}{=} \norm{\E_{\P} \Psi_{\P_0}(X)}{\H} \\
    & \stackrel{(d),(b),\eqref{eq:K_0-def}}{=} \sqrt{\E_{\P\otimes\P}K_0(X,X')}
    \stackrel{(d),(b),(e)}{=}\norm{\int_\X K_0(\cdot,x)\d \P(x)}{\H_{K_0}} \stackrel{\eqref{eq:mean-embedding}}{=}\norm{\mu_{K_0}(\P)}{\H_{K_0}}. \label{eq:population-ksd}
\end{align}
(a) follows from the homogeneity of $\mathcal T_{\P_0}$ and the expectation, and using the symmetry of the unit ball in $\H$. We change the order of the expectation and the inner product in (b), rely on the self-duality of the Hilbert norm in (c), use that the norm in a Hilbert space is induced by its inner product in (d), and apply the reproducing property \eqref{eq:repr-prop-of-K_0} in (e). Further, we use the notation
\begin{align}
K_0(x,x') \coloneq \ip{\Psi_{\P_0}(x),\Psi_{\P_0}(x')}{\H} \text{ for all }x,x' \in \X. \label{eq:K_0-def}
\end{align}
As $K_0$ is a kernel, there exists an associated RKHS $\H_{K_0}$ for which $K_0$ is the (reproducing) kernel. Hence, for any $x,x'\in\X$ it holds that 
\begin{align}
	K_0(x,x') = \ip{\fm x, \fm{x'}}{\H_{K_0}}. \label{eq:repr-prop-of-K_0}
\end{align}
We note that $\Psi_{\P_0}(x) \in \H$ and $K_0(\cdot,x) \in \H_{K_0}$ ($x\in\X$) but both yield the same kernel $K_0$.

We illustrate these definitions with an example after we collect our requirements in the following assumption.

\begin{assumption}[Well-definedness of KSD, regularity of $K_0$] \label{ass:integrability} $(\X,\tau_\X)$ is a separable topological space, the Stein kernel $K_0 : \X \times \X \to \R$ is continuous and $\E_{\P_0}K_0(X,X)<\infty$, and $\P \in \mathcal S_1 \coloneq \left\{\P \in \mathcal M_1^+(\X) : \E_\P \sqrt{K_0(X,X)} < \infty\right\}$.
\end{assumption}

We make the following remarks.
\begin{remark}\label{remark:integrability}~
\begin{enumerate}[label=(\alph*)]
    \item The continuity of $K_0$ ensures the separability of $\H_{K_0}$ due to the separability of $\X$ \citep[Lemma~4.33]{steinwart08support}. \label{item:HK0-sep}
    \item \label{item:compact-op} Assumption~\ref{ass:integrability} implies that $\integralop \coloneqq T_{\P_0,K_0}$ is self-adjoint and trace-class.
    \item\label{item:existence} The assumption $\P\in\mathcal S_1$  ensures the existence of the mean embedding $\kme \P$ as\begin{align}
        \D \stackrel{\eqref{eq:population-ksd}}{=} \norm{\int_\X K_0(\cdot,x)\d \P(x)}{\H_{K_0}} \stackrel{(i)}{\le} \E_\P\rnorm{\fm X} \stackrel{(ii)}{=} \E_\P\sqrt{K_0(X,X)},
    \end{align}
    where (i) holds by the Jensen inequality, (ii) is implied by the fact that the norm in a Hilbert space is induced by its inner product and the reproducing property.
	\item \label{remark:item:e-p0-equals-0} The KSD construction by \eqref{eq:population-ksd} implies that $\E_{\P_0}\fm X = 0$.
\end{enumerate}
    
\end{remark}

One example of a KSD on $\X = \R^d$ is the following.

\begin{example}[Langevin-Stein KSD on $\R^d$]  \label{example:langevin-stein}
    Let $\H_k$ be an RKHS on $\R^d$ with reproducing kernel $k : \R^d\times \R^d \to \R$, $\H_{k}^d=\times_{i=1}^d\H_k$ the product RKHS  with inner product $\ip{\b f,\b g}{\H_k^d} = \sum_{i=1}^d\ip{f_i,g_i}{\H_k}$ for $\b f=\left( f_i \right)_{i=1}^d, \b g = \left( g_i \right)_{i=1}^d \in \H_k^d$.
    Assume that $k$ is twice continuously differentiable and that $\P_0$ is absolutely continuous w.r.t.\ the Lebesgue measure with density $p_0$. Further, assume that $p_0$ is continuously differentiable and has support $\R^d$ (in other words, $p_0(\b x) >0$ for all $\b x \in \R^d$), and $\lim_{\norm{\b x}{2}\to\infty}f(\b x)p_0(\b x) = 0$ for all $f \in \H_k$.
    Boundedness of $p_0$ and $\lim_{\norm{\b x}{2}\to\infty}f(\b x) = 0$ for all $f \in \H_k$ are sufficient for the last property to hold. 
    The Langevin-Stein operator \citep{gorham15measuring,oates17control} acting on $\b f \in \H_k^d$ is defined as $
    \left(\mathcal T_{\P_0} \b f\right)(\b x) := \langle \nabla_{\b x} [\log p_0(\b x)], \b f(\b x)\rangle + \sum_{i=1}^d \dfrac{\partial f_i(\b x)}{\partial x_i}= \langle \Psi_{\P_0}(\b x) ,\b f\rangle_{\H_k^d}$ with $\Psi_{\P_0}(\b x)=\nabla_{\b x}[\log p_0(\b x)] k(\cdot,\b x)+\nabla_{\b x} k(\cdot,\b x) \in \H_k^d$ for $\b x \in \R^d$. Using the reproducing property, this form of $\Psi_{\P_0}$ gives rise to \citep{chwialkowski16kernel,liu16kernelized}
    \begin{align}
         K_0(\b x, \b y) &= \ip{\nabla_{\b x}\log p_0(\b x),\nabla_{\b y}\log p_0(\b y)}{}k(\b x, \b y) +
		 \ip{\nabla_{\b y}\log p_0(\b y),\nabla_{\b x}k(\b x, \b y)}{}  \\
  &\quad + \ip{\nabla_{\b x}\log p_0(\b x),\nabla_{\b y}k(\b x, \b y)}{} + \sum_{i \in [d]}\dfrac{\partial^2k(\b x,\b y)}{\partial x_i\partial y_i},
    \end{align}
    with $\b x,\b y \in \R^d$ and  the corresponding KSD is $D^2_{\P_0}(\P) = \rnorm{\kme\P}^2$ by \eqref{eq:population-ksd}. As $K_0$ only depends on the derivative of the score function of $p_0$ (that is, $\nabla_{\b x}\log p_0(\b x)$), knowledge of $p_0$ up to a normalizing factor is enough. To sum up, in this case one has the choice $\X = \R^d$, $\H = \H_k^d$, and $\Psi_{\P_0}=\nabla_{\b x}[\log p_0(\b x)] k(\cdot,\b x)+\nabla_{\b x} k(\cdot,\b x)$ in the general KSD construction.
\end{example}

We refer to \citet[Examples 2--3]{hagrass24stein} and Section~\ref{sec:experiments} for illustrative examples on non-Euclidean domains.

\subsection{KSD for Goodness-of-Fit Testing} \label{sec:background-gof}
Recall that the goal of GoF testing is to validate if $\P_0=\P$, where $\P_0$ is assumed to be known, and $\P$ is only observable through i.i.d.\ samples $(X_i)_{i=1}^n  \sim \P^n$.\footnote{By Example~\ref{example:langevin-stein}, even partial knowledge of $\P_0$ can suffice.}
The following assumption allows using KSD for GoF testing.
\begin{assumption}[Validness of KSD on $\mathcal S_1$]\label{ass:validity-ksd}
	For any $\P\in\mathcal S_1$, it holds that $\D =0$ iff $\P_0 = \P$.
\end{assumption}
\begin{remark}
	Assumption~\ref{ass:validity-ksd} is satisfied if, for instance, $\mu_{K_0}$ is injective on $\mathcal S_1$. Alternatively, in the setting of Example~\ref{example:langevin-stein}, assume that $\P \in \mathcal S_1$ has probability density function $p$; then $\E_\P \norm{\nabla\log\frac{p_0(X)}{p(X)}}{2}^2 < \infty$ and $c_0$-universality \citep[Definition~4.1]{carmeli10vector} of $k$ are sufficient \citep[Theorem~2.2]{chwialkowski16kernel}.
\end{remark}
Given Assumption~\ref{ass:validity-ksd}, one may use \eqref{eq:population-ksd} to construct a GoF test by considering  
\begin{align}
	H_0 : \D = 0 && \text{vs.} && H_1 : \D \neq 0, \label{eq:hypothesis}
\end{align}
which is then equivalent to considering $H_0 : \P_0 = \P$ vs.\ $H_1: \P_0 \neq \P$.
In the setting of Example~\ref{example:langevin-stein}, in particular $\X=\R^d$, \citet{chwialkowski16kernel} suggest a V-statistic-based estimator of \eqref{eq:population-ksd}, obtained by replacing $\P$ with the associated empirical measure $\hat \P_n \coloneq \frac1n\sum_{i=1}^n\delta_{X_i}$ and squaring, which takes the form
\begin{align}
	D_{\P_0}^2\Big(\hat \P_n\Big) = \frac{1}{n^2}\sum_{i,j=1}^{n}K_0(X_i,X_j). \label{eq:v-stat-ksd}
\end{align}
Alternatively, a U-statistic-based estimator can be defined by omitting the $i=j$ terms \citep{liu16kernelized}, that is,
	\begin{align}
		U_{\P_0}^2\Big(\hat \P_n\Big) = \frac{1}{n(n-1)}\sum_{\substack{i,j=1\\i\neq j}}^{n}K_0(X_i,X_j). \label{eq:u-stat-ksd}
	\end{align}
In this work, for simplicity, we limit our analysis to the V-statistic-based estimator \eqref{eq:v-stat-ksd}, which has recently been shown to be optimal in the minimax sense under mild conditions \citep{cribeiroramallo25minimaxlowerboundkernel}.

While KSD has many desirable properties for testing goodness-of-fit, the classical estimators \eqref{eq:v-stat-ksd} and \eqref{eq:u-stat-ksd} have (i) a null distribution that is analytically intractable and (ii) a runtime cost of $\O\!\left( n^2 \right)$. To tackle (i), \citet{chwialkowski16kernel} propose to consider the wild bootstrap\footnote{The wild bootstrap is also referred to as i.i.d.\ weighted bootstrap \citep{dehling94bootstrap}. One can extend the bootstrap idea to consider non-i.i.d.\ $(X_i)$ sequences \citep{leucht13dependent}; this adaptation is then called wild dependent bootstrap.\label{fn:dehling-mikosch}}\textsuperscript{,}\footnote{In fact, \citet{chwialkowski16kernel} consider the wild dependent bootstrap. We assume that the $X_i$-s are i.i.d.\ throughout this work and therefore state the simplified case.\label{footnote:wild-bootstrap}}\textsuperscript{,}\footnote{An alternative approach \citep{gretton09fast}, requiring $\O\!\left(n^3\right)$ computations in practice, is to estimate the $\lambda_i$-s of the truncated sum in Theorem~\ref{thm:serfling-v-stat}\ref{thm:conv:null}.} 
\begin{align}
	B_n^2 = \frac{1}{n^2}\sum_{i,j=1}^nR_iR_j K_0(x_i,x_j), \label{eq:ksd-bootstrap}
\end{align}
where $(x_i)_{i=1}^n$ are fixed, and $\left( R_i \right)_{i=1}^n\sim\rho^n$ are Rademacher random variables, that is, 
\begin{align}
\rho(R_i = 1) = \rho(R_i = -1) = 1/2\quad (i\in[n]). \label{eq:rho-def}
\end{align}
Assuming that $\X = \R^d$, \citet[Proposition~3.2]{chwialkowski16kernel} establish  that under $H_0$, the quantiles of $nB_n^2$ and $nD_{\P_0}^2\!(\hat \P_n)$ match asymptotically, which motivates their following test procedure.

\begin{enumerate}
	\item Calculate the test statistic \eqref{eq:v-stat-ksd}.
	\item Obtain $c_b$ bootstrap samples $\{B_{n,i}\}_{i=1}^{c_b}$ from \eqref{eq:ksd-bootstrap} and estimate the $1-\alpha$ empirical quantile of these samples.
	\item Reject the null hypothesis if \eqref{eq:v-stat-ksd} computed in step~1 exceeds the quantile obtained in step~2.
\end{enumerate}
	
\eqref{eq:v-stat-ksd} requires $\O\!\left( n^2 \right)$ computations, obtaining $c_b$ bootstrap samples costs $\O\! \left( c_bn^2 \right)$, which inhibits the application of KSD for large-scale goodness-of-fit testing.

\section{Accelerated Goodness-of-Fit Testing with KSD} \label{sec:results}

This section collects our results. We start by obtaining the key properties of the quadratic-time KSD estimator (consistency, asymptotic distribution) and GoF test (validity of bootstrap, local consistency) in  Section~\ref{sec:quadratic-time}. Next, we establish the corresponding results (see also the overview in Table~\ref{table:overview-main-results}) of the Nystr\"{o}m-accelerated estimator (Section~\ref{sec:Nyström-estimator}) and GoF test (Section~\ref{sec:Nyström-bootstrap}), obtaining the same rates throughout. These results indicate that Nystr\"{o}m-accelerated GoF testing with KSD is possible without any noticeable loss in testing power, which we also validate numerically in Section~\ref{sec:experiments}.

\subsection{Quadratic-time Estimator and Wild Bootstrap}\label{sec:quadratic-time}

Our first goal is to use the limiting distribution of \eqref{eq:v-stat-ksd} to obtain KSD-based GoF tests, and to settle their validity against fixed and local alternatives on general domains. We further establish their separation rate. While of independent interest, these results also allow us to put our later results on the Nystr\"{o}m-based accelerations into perspective. 

For self-containedness, we start with the following result (proved in Section~\ref{sec:proof-consistency-v-stat}) by showing the $\sqrt n$-consistency of the estimator \eqref{eq:v-stat-ksd}, slightly weakening the sub-Gaussian assumption on $\rnorm{\fm X}$ of the similar result \citep[Theorem~3]{kalinke25nystromksd} to a sub-exponential assumption.

\begin{theorem}[$\sqrt n$-consistency of quadratic-time KSD estimator]\label{thm:v-stat-consistency2}
	Let Assumption~\ref{ass:integrability} hold, suppose that
		$\norm{\rnorm{\fm X}}{\psi_1}< \infty$,
    and define the centered Stein feature map 
\begin{equation}
    \fmc x \coloneq \fm x - \E_\P K_0(\cdot,X) \text{ for } x\in\X. \label{eq:centered-stein-fm}
\end{equation}
	Then, for any $\delta \in (0,1)$, it holds with $\P$-probability of at least $1-\delta$ that
	\begin{align}
		\rnorm{\kme \P - \kme{\hat \P_n}} \lesssim \frac{2K\log(2/\delta)}{n} + \sqrt{\frac{2K^2\log(2/\delta)}{n}}, \label{eq:sqrt-n-sub-exp-rate}
	\end{align}
	with constant $K = \psione{\rnorm{\fmc{X}}}$ (depending on $\P$ and $K_0$). In particular, it holds with the same probability that 
	\begin{align}
		\left|\D - D_{\P_0}\big( \hat \P_n \big)\right| \lesssim \frac{2K\log(2/\delta)}{n} + \sqrt{\frac{2K^2\log(2/\delta)}{n}}.
	\end{align}
\end{theorem}

\begin{remark}
	The main take-away of Theorem~\ref{thm:v-stat-consistency2} is that for sub-exponential $\rnorm{\fm X}$, $\D - D_{\P_0}\big(\hat \P_n \big) = O_P\!\left(n^{-1/2}\right)$. We include \eqref{eq:sqrt-n-sub-exp-rate} for later use.
\end{remark}

Our next result, the weak limit of \eqref{eq:v-stat-ksd}, follows from known results on V-statistics \citep{serfling80approximation,shao03mathematicalstatistics}, and is proved in Section~\ref{sec:proof-serfling-v-stat}.

\begin{theorem}[Asymptotic distribution of quadratic-time KSD estimator]\label{thm:serfling-v-stat} Let Assumptions~\ref{ass:integrability}--\ref{ass:validity-ksd} hold and $\fmc x$ with $x\in\X$ be as in \eqref{eq:centered-stein-fm}. Assume that $\E_\P K_0(X,X) < \infty$.  Then, 
	\begin{enumerate}[label=(\roman*)]
		\item \label{thm:conv:alternative} if $\P_0\neq\P$, $\sqrt{n}\left(D_{\P_0}^2\Big(\hat \P_n\Big)-D_{\P_0}^2(\P)\right) \dconv \mathcal N\!\left(0,4\rnorm{\covc^{1/2}\kme\P}^2\right)$, assuming that $\mu_{K_0}(\P)(X)$ is non-degenerate\footnote{\label{fn:degenerate}$\mu_{K_0}(\P)(X)$ is non-degenerate if there exists no $c\in\R$ such that $\mu_{K_0}(\P)(X) = c$ holds $\P$-almost surely.} and $\norm{\rnorm{\fm X}}{\psi_1} < \infty$, and,
		\item \label{thm:conv:null} if $\P_0=\P$, $nD_{\P_0}^2\Big(\hat \P_n\Big) \dconv \sum_{j\in J}\lambda_jZ_j^2$, where $(\lambda_j)_{j\in J}$ are the eigenvalues of $\integralop$ and $(Z_j)_{j\in J}$ are independent standard normal random variables, provided that $\covn \neq 0$.
	\end{enumerate}
\end{theorem}

\begin{remark} We note that the above result considers the V-statistic-based estimator \eqref{eq:v-stat-ksd} on general domains. Existing works either state the result only implicitly on $\mathbb R^d$ \citep{chwialkowski16kernel}, for the U-statistic-based estimator \eqref{eq:u-stat-ksd} \citep{liu16kernelized}, or directly bootstrap the null distribution \citep{schrab22ksdagg,hagrass24stein}.
\end{remark}

Before leveraging Theorem~\ref{thm:serfling-v-stat} to  construct a test, we recall that a test $S_n : \X^n \to \{0,1\}$ for $H_0$ vs.\ $H_1$ is a decision function, rejecting $H_0$ if $S_n = 1$. $S_n$ is said to have asymptotic level $\alpha \in (0,1)$ if $\E_{\P_0} S_n \to \alpha$ as $n\to \infty$. The type II error of $S_n$ w.r.t.\ to a class of alternatives $\mathcal P \subseteq \mathcal S_1$  is 
\begin{align}
	\beta(S_n,\mathcal P) \coloneq \sup_{\P\in\mathcal P}\E_\P[1-S_n(X_1,\ldots,X_n)], \label{eq:type-ii-error}
\end{align}
and $S_n$ is called consistent w.r.t.\ $\mathcal P$ if $\beta(S_n,\mathcal P) \to 0$ as $n\to\infty$. Hence, to test if $H_0 : \P_0 = \P$ holds at level $\alpha\in(0,1)$, it is natural to consider 
\begin{align}
	S_n \coloneq S_n(X_1,\ldots,X_n) \coloneq \bm 1_{\left\{nD_{\P_0}^2\!(\hat \P_n) > q_{W,1-\alpha}\right\}}, \label{eq:test-definition}
\end{align}
with $q_{W,1-\alpha}$ denoting the $1-\alpha$-quantile of 
\begin{align}
	W\coloneq \sum_{j\in J}\lambda_jZ_j^2 \label{eq:def-limit-h0-w}
\end{align}
and the sum is as in Theorem~\ref{thm:serfling-v-stat}\ref{thm:conv:null}. This construction immediately implies that \eqref{eq:test-definition} is an asymptotic level $\alpha$ test. 
Under a fixed alternative $\P\in\mathcal S_1 \setminus \{\P_0\}$, $\sqrt n D_{\P_0}^2\!(\hat \P_n)$ is asymptotically normal by Theorem~\ref{thm:serfling-v-stat}\ref{thm:conv:alternative}; hence, $nD_{\P_0}^2\!(\hat \P_n)$ diverges and $\beta(S_n,\{\P\}) \to 0$ as $n\to \infty$, that is, the test \eqref{eq:test-definition} is consistent against fixed alternatives.

We note that $q_{W,1-\alpha}$ can be replaced with any consistent estimator of the corresponding quantile. Indeed, our next result (proved in Section~\ref{sec:proof-limit-v-stat}) shows that the wild bootstrap\textsuperscript{\ref{fn:dehling-mikosch}} allows us to approximate the null distribution in Theorem~\ref{thm:serfling-v-stat}\ref{thm:conv:null} if $\X$ is a separable metric space.

\begin{theorem}[Bootstrap of V-statistic-based KSD] \label{theorem:bootstrap-v-stat}
	Let Assumption~\ref{ass:integrability} hold. Suppose 
    $\X$ is a (separable) metric space, $\left( x_i \right)_{i=1}^\infty \subseteq \operatorname{supp}(\P_0)$ is fixed, $\left( R_i \right)_{i=1}^n\sim\rho^n$ with $\rho$ as in \eqref{eq:rho-def}. Let $B_n^2$ be as in \eqref{eq:ksd-bootstrap}. Then, under $H_0$,
    \begin{align}
    nB_n^2 \dconv W
    \end{align}
    as $n \to \infty$, with $W$ defined in \eqref{eq:def-limit-h0-w}.
\end{theorem}

\begin{remark}
	Note that Theorem~\ref{theorem:bootstrap-v-stat} is similar to \citet[Proposition 3.2]{chwialkowski16kernel} with  the latter allowing for dependent sequences $\left( X_i \right)_i$. However, \citet[Proposition 3.2]{chwialkowski16kernel} relies on \citet[Theorem~2.1]{leucht13dependent}, which only holds on $\R^d$, and therefore, \citet[Proposition 3.2]{chwialkowski16kernel} is limited to $\R^d$. Theorem~\ref{theorem:bootstrap-v-stat} substantially weakens the latter requirement by handling any separable metric space. 
\end{remark}

As we established above that the test \eqref{eq:test-definition} is consistent against fixed alternatives, the interest in analyzing KSD-based tests is in considering a sequence of alternatives $\mathcal P_n$ containing probability measures that become more ``similar'' to $\P_0$ as $n\to\infty$, that is, the testing problem becomes more difficult as the sample size increases and the goal is to show that even
\begin{align}
\beta(S_n,\mathcal P_n) \to 0
\end{align}
as $n\to \infty$. The analysis of GoF-testing with KSD takes place in the operator-theoretic framework, which we quickly recall.

Let $\P \in \mathcal S_1$, $\P\ll\P_0$, and $u_\P \coloneq \tfrac{\d \P}{\d \P_0} - 1$ (we refer to Remark~\ref{remark:class-of-alternatives}\ref{item:absolute-continuity} below for a discussion). Using the decomposition $\integralop = \sum_{j\in J}\lambda_j\tilde\phi_j \otimes_{L^2(\X,\P_0)} \tilde \phi_j$ (implied by the spectral theorem given Remark~\ref{remark:integrability}\ref{item:compact-op}) and that $\E_{\P_0}\fm X = 0$ (see Remark~\ref{remark:integrability}\ref{remark:item:e-p0-equals-0}), it can be shown \citep[Proposition~1]{hagrass24stein} under Assumption~\ref{ass:integrability} that 
\begin{align}
	\D = \norm{\integralop^{1/2}u_\P}{L^2(\X,\P_0)}. \label{eq:explicit-statistic}
\end{align}
This reformulation shows that $u_\P$ links $D_{\P_0}(\P)$ and $\integralop$, which will be crucial in the analysis of the test. Indeed, adapting the framework considered in \citet[(9)]{hagrass24stein}, define the class of alternatives $\mathcal P_n \subseteq \mathcal S_1$ that satisfies a certain smoothness assumption and is separated from $\P_0$ in terms of $\chi^2$-divergence as 
\begin{align}
	\mathcal P_n \coloneq \mathcal P_n(\Delta_n, \theta) \coloneq &\Big\{\P \in \mathcal M_1^+(\X) : \E_\P K_0^r(X,X) \le cr!\kappa^r \text{ for all } r\ge 2,\\
    &\qquad 	\P \ll \P_0,\; u_\P = \dfrac{\d \P}{\d \P_0} -1 \in \range\!\left(\integralop^\theta\right),\;
	   \Delta_n \le \chi^2(\P,\P_0) <\infty \Big\},\hspace{0.45cm} \label{eq:alternatives}
\end{align}
for some (fixed) $c,\kappa >0$ and $\theta, \Delta_n > 0$ ($n\in\mathbb N_{>0}$).
\begin{remark}\label{remark:class-of-alternatives}~ \begin{enumerate}[label=(\alph*)]
    \item \tb{Relationship to $\mathcal S_1$.} By  the imposed moment conditions in \eqref{eq:alternatives} and the monotonicity of $L^p$-norms in terms of $p$, it holds that $\E_\P\sqrt{K_0(X,X)} < \infty$ given that $\P\in\mathcal P_n$; hence, $\mathcal P_n \subseteq \mathcal S_1$ (with $\mathcal S_1$ defined in Assumption~\ref{ass:integrability}). We also refer to Remark~\ref{remark:integrability}\ref{item:existence} detailing the finiteness of $\D$ for any $\P\in \mathcal S_1$.
	\item \tb{Moment conditions.} The fixed $c,\kappa > 0$ of the moment conditions on $K_0$ in \eqref{eq:alternatives} imply that Bernstein-type concentration inequalities hold uniformly over $\P \in \mathcal P_n$, used in the proof of our next result (Theorem~\ref{thm:separation-ksd}). In particular, these allow relaxing the uniform boundedness conditions on the eigenfunctions of the Mercer decomposition of $K_0$, which are present in related works \citep{balasubramanian21minimaxgof,hagrass24stein}. We further elaborate this point in Remark~\ref{remark:consistency-ksd}\ref{item:comparison-hagrass} and in the proof of Theorem~\ref{thm:separation-ksd} in Section~\ref{sec:proof-separation-ksd}. 
    \item \label{item:absolute-continuity} \tb{Absolute continuity.} The assumption $\P\ll \P_0$ implies the existence of $\tfrac{\d \P}{\d \P_0}$. Furthermore, $\chi^2(\P,\P_0) = \norm{u_\P}{L^2(\X,\P_0)}^2$ by using the definition of the $\chi^2$-divergence in \eqref{eq:chi-square-definition}; we recall that $\P_0$ is fixed and emphasize the $\P$-dependence of $u_\P$ using the subscript. Notice also that the stated equivalence and \eqref{eq:alternatives} imply that $\norm{u_\P}{L^2(\X,\P_0)}^2 < \infty$ for $\P\in \mathcal P_n$.
	\item \tb{Smoothness assumption.} In the analysis of tests on Euclidean spaces, the densities $\tfrac{\d \P}{\d \P_0}$ are typically assumed to lie in certain Sobolev or Hölder classes, that is, they satisfy some smoothness condition. The range space assumption $u_{\P} \in \range \!\left(\integralop^\theta\right)$ in \eqref{eq:alternatives} can be interpreted as a corresponding assumption on general domains, which also considers the interplay between $\P$ and $K_0$. Such an assumption is typical in the analysis of kernel-based algorithms in regression, testing, and inverse problems \citep{caponnetto07optimal,cucker07approximation,rudi15less,hagrass23spectralgof,hagrass24stein,blanchard18optimalrates,devito21regularization}.
	\item \tb{Impact of $\theta$.}\label{remark:item:influence-theta} A larger value of parameter $\theta > 0$ corresponds to a smoother $u_\P$. In case of $\theta \in (0,1/2]$, one can interpret the smoothness in terms of interpolation spaces: $\range\! \left(\integralop^\theta\right)$ corresponds to \citep[Theorem~4.6]{steinwart12mercer} the real interpolation of $L^2(\X,\P_0)$ (in the limit of $\theta = 0)$ and $\left[\H_{K_0}\right]_{\sim} \coloneq \left\{ [f]_{\sim} : f \in \H_{K_0}\right\}$ ($\theta = 1/2$). We again refer to \citet{cucker07approximation} for the connection of learning and approximation theory. 
\end{enumerate}
\end{remark}

Our next result (proved in Section~\ref{sec:proof-separation-ksd}) shows the consistency of KSD against local alternatives. Together with \citet[Theorem~7]{hagrass24stein}, the result yields the precise separation rate of KSD estimation, which we elaborate in the remark following the result.

\begin{theorem}[Local consistency of quadratic-time KSD test] \label{thm:separation-ksd} 
	Let Assumptions~\ref{ass:integrability}--\ref{ass:validity-ksd} hold, assume $\theta > 0$,  $n^{\frac{2\theta}{2\theta+1}}\Delta_n \to \infty$ as $n\to \infty$, and
	\begin{align}
		\sup_{\P\in \mathcal P_n}\norm{\integralop^{-\theta}u_\P}{L^2(\X,\P_0)} < \infty.\label{eq:sup-assumption}
	\end{align}
	Then $\beta(S_n,\mathcal P_n) \to 0$ as $n\to \infty$.
\end{theorem}

Before we comment on Theorem~\ref{thm:separation-ksd}, we state an auxiliary lemma (proved in Section~\ref{sec:proof-lemma-range-space-equivalences}), unifying different assumptions on $u_\P$.

\begin{lemma}[Range space equivalences]\label{lemma:range-space-equivalences}
	Suppose that $\integralop$ has the spectral decomposition $\integralop = \sum_{j\in J} \lambda_j \tilde \phi_j\otimes_{L^2(\X,\P_0)} \tilde \phi_j$, $\theta>0$, and $\norm{u_\P}{L^2(\X,\P_0)} < \infty$. Then (i) $u_\P \in \range\!\left( T_{\P_0}^\theta \right)$, (ii) $\sum_{j\in J}\lambda_j^{-2\theta}  \big\langle u_\P, \tilde \phi_j\big\rangle^2_{L^2(\X,\P_0)} < \infty$, and (iii) $u_\P \in S_{T_{\P_0}^{-\theta}}$ are equivalent. Further, if any (and thus all) of (i)--(iii) hold, $\norm{\integralop^{-\theta}u_\P}{L^2(\X,\P_0)}^2 = \sum_{j\in J} \lambda_j^{-2\theta}\big\langle u_\P, \tilde \phi_j\big\rangle^2_{L^2(\X,\P_0)} $.
\end{lemma}

\begin{remark} \label{remark:consistency-ksd}~
	\begin{enumerate}[label=(\alph*)]
        \item \tb{Well-definedness of $\integralop^{-\theta}u_\P$ in \eqref{eq:sup-assumption}.} In $\mathcal P_n = \mathcal P_n(\Delta_n,\theta)$, the assumption $u_\P\in\range\big(\integralop^\theta\big)$ was imposed, which by Lemma~\ref{lemma:range-space-equivalences} is equivalent to $u_\P \in S_{T_{\P_0}^{-\theta}}$. The latter, by definition \eqref{eq:S-def}, implies that $\integralop^{-\theta} u_\P$ in \eqref{eq:sup-assumption} is well-defined.
		\item \tb{Intuition on the assumption in \eqref{eq:sup-assumption}.} By Lemma~\ref{lemma:range-space-equivalences}, \eqref{eq:sup-assumption} has the equivalent formulation
		\begin{align}
			\sup_{\P\in \mathcal P_n}\left(\sum_{j\in J}\lambda_j^{-2\theta}\ip{u_\P,\tilde\phi_j}{L^2(\X,\P_0)}^2\right)^{1/2} < \infty,
		\end{align}
		which, given that $\lambda_j^{-2\theta}$ diverges as $j\to \infty$, requires the Fourier coefficients of $u_\P$ ($\big\langle u_\P,\tilde\phi_j \big\rangle_{L^2(\X,\P_0)}$) to decay fast enough. In other words, \eqref{eq:sup-assumption} restricts the class of alternatives $\mathcal P_n$ such that $u_\P$ is sufficiently smooth.
		\item \tb{Influence of $\theta$.} As per the discussion in Remark~\ref{remark:class-of-alternatives}\ref{remark:item:influence-theta}, a larger $\theta$ corresponds to a smoother $u_\P$. Indeed, the function $f : \theta \in \R_{>0} \mapsto \frac{2\theta}{2\theta+1}$ is strictly increasing with $\lim_{\theta\to 0}f(\theta) = 0$ and $\lim_{\theta\to \infty}f(\theta) = 1$ and 
        Theorem~\ref{thm:separation-ksd} shows that with larger $\theta$, $\Delta_n$ can shrink faster while still obtaining a consistent test with a separation rate of $n^{\frac{2\theta}{2\theta+1}}$; the testing problem becomes easier. 
		
		\item \tb{Comparison to \citet[Theorem~7]{hagrass24stein}.}\label{item:comparison-hagrass} In the case of the U-statistic-based estimator \eqref{eq:u-stat-ksd}, the authors showed that for either (i) $\theta>1$ or (ii) $\theta>0$ and 
        \begin{equation} 
        \sup_{j\in J}\norm{\phi_j}{\infty} < \infty \label{eq:uniform-boundedness}
        \end{equation}
        with the $\phi_j$-s being the continuous representatives of the $\tilde \phi_j$-s, one has that
		\begin{align}
			\limsup_{n\to\infty}\sup_{\P\in\mathcal P_n} \E_{\P}[1-S_n(X_1,\ldots,X_n)] > 0, \label{eq:type-ii-error-n-dependent}
		\end{align}
		for $n^{\frac{2\theta}{2\theta+1}}\Delta_n \to 0$ as $n\to \infty$, where \eqref{eq:uniform-boundedness} implies that $K_0$ is bounded by the Mercer decomposition \citep[Lemma~2.6]{steinwart12mercer}; we refer also to \citet[Remark~4(i)]{hagrass24stein}.
		In other words, if one grows $\mathcal P_n$ at this rate w.r.t.\ $n$, the test $S_n$ is asymptotically \emph{not} consistent. We make a few comments. First, we note that, as per \citet[Example~1]{kalinke25nystromksd} or \citet[Remark~4(i)]{hagrass24stein}, the boundedness of $K_0$ is too restrictive in the KSD setting, as it is virtually never satisfied. Second, the existing result with no such boundedness assumption only handles the case $\theta>1$, severely restricting the class of alternatives (see Remark~\ref{remark:class-of-alternatives}\ref{remark:item:influence-theta}). Hence, our result (Theorem~\ref{thm:separation-ksd}) differs in two notable ways: (i) We give a positive result, that is, we show that with $n^{\frac{2\theta}{2\theta+1}}\Delta_n \to \infty$ as $n\to \infty$, the considered $S_n$ is consistent. (ii) We consider the full range of $\theta>0$, that is, our result also applies to densities which are not necessarily smoother than functions in $\H_{K_0}$.
		As a take-away and ignoring that \citet[Theorem~7]{hagrass24stein} considers U-statistics while Theorem~\ref{thm:separation-ksd} considers V-statistics, $n^{\frac{2\theta}{2\theta+1}}$ is the tipping point: if one grows $\mathcal P_n$ any faster, the test is not consistent, and if one grows $\mathcal P_n$ any slower, the test is consistent.

		\item \tb{Comparison to \citet[Theorem~1]{balasubramanian21minimaxgof}.} While we consider testing for GoF with KSD, in the case of testing for GoF with the related maximum mean discrepancy using kernel $k$, \citet[Theorem~1]{balasubramanian21minimaxgof} showed that for a smoothness of $\theta=1/2$ (corresponding to $u_\P\in \left[\H_k\right]_\sim$) and with a uniform boundedness assumption on the eigenfunctions corresponding to the spectral decomposition of $T_{\P,k}$ (comparable to \eqref{eq:uniform-boundedness}), the separation rate of their test is $\sqrt n$. For this value of $\theta$, the rate matches that of Theorem~\ref{thm:separation-ksd}---but we consider a broader range of $\theta$ and impose no boundedness condition on the $\phi_i$-s. We further elaborate on the differences in the proof of Theorem~\ref{thm:separation-ksd} in Section~\ref{sec:proof-separation-ksd}.
	\end{enumerate}
\end{remark}

Having settled the consistency of \eqref{eq:test-definition} against local alternatives, the following sections elaborate Nystr\"{o}m-based accelerations of \eqref{eq:v-stat-ksd} and \eqref{eq:ksd-bootstrap}, and show that these preserve the statistical behavior detailed above for a sub-class of $\mathcal P_n$ (detailed in Theorem~\ref{thm:separation-n-ksd}).

\subsection{Nystr\"{o}m-accelerated KSD Estimator} \label{sec:Nyström-estimator}
To mitigate the quadratic runtime cost of \eqref{eq:v-stat-ksd}, \citet{kalinke25nystromksd} proposed to use a Nystr\"{o}m-based acceleration for estimating \eqref{eq:v-stat-ksd} called N-KSD, which we recall next.

Let $\left( X_i \right)_{i=1}^n\sim\P^n$,
\begin{align}
\Lambda \coloneq \operatorname{Unif}\!\left( [n] \right), \label{eq:Lambda-def}
\end{align}
$\H_{K_0,m} = \Span\!\left(\fm {X_{I_j}} : j\in[m]\right) \subset \H_{K_0}$, where $\left( I_j \right)_{j=1}^m \sim\Lambda^m$ with $m$ denoting the number of Nystr\"{o}m points. They proposed to approximate \eqref{eq:v-stat-ksd} by the (orthogonal) projection of $\kme{\hat \P_n}$ onto $\H_{K_0,m}$, taking the form
\begin{align}
	D^2_{\P_0}\left( \hat\P_n \right) \approx \rnorm{P_{\H_{K_0,m}}\kme{\hat \P_n}}^2 = \bm\beta\T\Kmm^{-}\bm\beta \eqcolon \tilde D^2_{\P_0}\!\left( \hat\P_n \right),\label{eq:nystroem-ksd-estimator}
\end{align}
with $\bm\beta = \frac1n\Kmn \bm1_{n}\in \R^m$, 
\begin{align}
    \Kmm &= \left[K_0\!\left(X_{I_i},X_{I_j}\right)\right]_{i,j=1}^m \in \R^{m\times m},\text{ and} \label{eq:k0-kmm} \\
    \Kmn &= \left[K_0\!\left(X_{I_i},X_j\right)\right]_{i,j=1}^{m,n} \in \R^{m\times n}. \label{eq:k0-kmn}
\end{align}
We abbreviate $P_{\H_{K_0,m}} \eqcolon \projm$ in the following and note that $\Knn$, $\Kmm$, $\Knm$, and $\Kmn$ refer to different objects (they do not necessarily coincide for $n=m$).

\begin{remark}\label{remark:positive-definite}
	Note that $\Kmm^- \succcurlyeq 0$ as is evident by considering that the pseudoinverse can be computed by inverting the (strictly) positive eigenvalues in the singular value decomposition (and keeping the zero ones) of the positive semi-definite $\Kmm \succcurlyeq 0$. This observation implies that $\tilde D_{\P_0}\!\left( \hat \P_n \right)$ is well-defined.
\end{remark}

Imposing a sub-Gaussian assumption on the centered Stein feature map $\fmc x$ ($x\in\X$), defined in \eqref{eq:centered-stein-fm}, permits bounding the error in \eqref{eq:nystroem-ksd-estimator}, and, as we will see, is one of the key ingredients for showing the consistency of Nystr\"{o}m-accelerated GoF testing with KSD.

\begin{assumption}\label{ass:sub-gaussian}
	Let $X\sim\P$ and assume that the centered Stein feature map $\fmc x$ satisfies
	\begin{align}
		\norm{\ip{\fmc X, u}{\H_{K_0}}}{\psi_2} \lesssim \norm{\ip{\fmc X, u}{\H_{K_0}}}{L^2(\P)}
	\end{align}
     for all $u\in \H_{K_0}$.
\end{assumption}

We recall a few known facts regarding this assumption.

\begin{remark}~ \label{remark:connections-sub-gauss-exp}
	\begin{enumerate}[label=(\alph*)]
		\item \tb{Necessity of Assumption~\ref{ass:sub-gaussian}.} In typical cases, the Stein kernel $K_0$ is unbounded; see \citet[Example~1]{kalinke25nystromksd} or \citet[Remark 4(i)]{hagrass24stein} for concrete examples. The sub-Gaussian assumption on the feature map enables the analysis of the Nystr\"{o}m-accelerated KSD estimator \eqref{eq:nystroem-ksd-estimator} as shown by \citet{kalinke25nystromksd}. \citet{della21regularized} imposes a similar requirement for the analysis of empirical risk minimization on random subspaces.
		\item \tb{Sub-Gaussianity of RKHS-norm.} Assumption~\ref{ass:sub-gaussian} yields the sub-Gaussianity of $\rnorm{\fmc X}$ (\citealt[Lemma~B.3]{kalinke25nystromksd}; recalled in Lemma~\ref{lemma:sub-gauss-norm}), which implies by
        \begin{align*}
            \psitwo{\rnorm{\fm X}} &\overset{\eqref{eq:centered-stein-fm}}{=} \psitwo{\rnorm{\fmc X + \E_{\P}\fm X}} \\
            &\overset{\hphantom{\eqref{eq:centered-stein-fm}}}{\le} \psitwo{\rnorm{\fmc X}} + \psitwo{\rnorm{\E_{\P}\fm X}} 
        \end{align*}
        the sub-Gaussianity of $\rnorm{\fm X}$, that is, 
		\begin{align}
			\psitwo{\rnorm{\fm X}} < \infty. \label{eq:sub-gauss-rkhs-norm}
		\end{align}
		Further, \citet[Example~3]{kalinke25nystromksd} details how \eqref{eq:sub-gauss-rkhs-norm} can be verified in certain cases.
		\item \tb{$\sqrt n$-consistency of KSD estimator.} The weaker requirement \eqref{eq:sub-gauss-rkhs-norm} yields the $\sqrt n$-consistency of the quadratic-time KSD estimator \eqref{eq:v-stat-ksd}. In other words,
		\begin{align}
			\D - D_{\P_0}\!\left( \hat \P_n \right) = \O_P\!\left( n^{-1/2} \right) \label{eq:sqrt-n-consistency}
		\end{align}
		given that \eqref{eq:sub-gauss-rkhs-norm} holds \citep[Theorem~3]{kalinke25nystromksd}.
        \item \tb{Relaxation of decay assumption on $\rnorm{\fm X}$.} \label{remark:item:relaxation}
        Weakening the assumption $\psitwo{\rnorm{\fm X}} < \infty$ in \eqref{eq:sub-gauss-rkhs-norm} to $\psione{\rnorm{\fm X}} < \infty$ is enough to guarantee \eqref{eq:sqrt-n-consistency}, as shown above in Theorem~\ref{thm:v-stat-consistency2}.
	\end{enumerate}
\end{remark}

The following known result \citep[Corollary~1]{kalinke25nystromksd} recalls that under suitable assumptions on the rate of decay of the eigenvalues of $C_{\P_0,\bar K_0}$ and a sub-Gaussian assumption on $\bar K_0(\cdot,X)$, \eqref{eq:nystroem-ksd-estimator} yields a computational gain over \eqref{eq:v-stat-ksd} without loss in statistical accuracy.

\begin{theorem}[Consistency of N-KSD estimator]\label{corr:decay-assumption} Let Assumptions~\ref{ass:integrability} and~\ref{ass:sub-gaussian} hold and assume that $\covc \neq 0$,
\begin{equation}
\mathcal N_{\bar K_0}(\lambda) \coloneq \trace\!\left( \covcr^{-1}\covc \right), \quad  \lambda > 0, \label{eq:centered-eff-dim}
\end{equation}
and that the spectrum of the covariance operator $C_{\P,\bar K_0}$ decays either (i) polynomially, implying that $\mathcal N_{\bar K_0}(\lambda) \lesssim \lambda^{-\gamma}$ for some $\gamma\in(0,1]$, or
	(ii) exponentially, implying that $\mathcal N_{\bar K_0}(\lambda) \lesssim \log(1+\tilde c/\lambda)$ for some $\tilde c>0$.
	Then it holds that
	  \begin{align}
		D_{\P_0}\!\left( \hat\P_n \right) - \tilde D_{\P_0}\!\left( \hat\P_n \right) = \O_P\!\left(n^{-1/2}\right),
	  \end{align}
	  given that the number $m$ of Nystr\"{o}m points satisfies 
	  \begin{enumerate}[label=(\roman*)]
		\item $m=\Omega\!\left( n^{\frac{1}{2-\gamma}}\log^{\frac{1}{2-\gamma}}n\right)$ in the polynomial decay case, or
		\item $m=\Omega\!\left(\sqrt{n}\log (n)\right)$ in the exponential decay case.
	  \end{enumerate}
	\end{theorem}

\begin{remark}~\label{remark:consistency-nksd}
\begin{enumerate}[label=(\alph*)]
\item \tb{Minimax optimality.} In the setting of Theorem~\ref{corr:decay-assumption}, the Nystr\"{o}m-accelerated KSD estimator has the same convergence rate as the quadratic-time estimator \eqref{eq:v-stat-ksd} and is known to be minimax optimal \citep{cribeiroramallo25minimaxlowerboundkernel}.

\item \tb{Asymptotic speedup.} To summarize the result, recall that the KSD estimator \eqref{eq:nystroem-ksd-estimator} has a runtime cost of $\O\!\left( nm+ m^3 \right)$, meaning that a speedup can be achieved for $m = o\!\left( n^{2/3} \right)$.
This condition can be satisfied  while matching the convergence rate of the quadratic-time estimator \eqref{eq:v-stat-ksd}, for instance, if the decay of the spectrum of the covariance operator is either polynomial and $\gamma < 1/2$, or exponential.
\item \tb{Condition $\covc\neq0$.} \label{remark:item:cov-op-non-zero} By Lemma~\ref{lemma:equivalent-conditoin-C-zero}, an equivalent condition for $\covc\neq0$ is that there exists $A\in \mathcal B(\tau_\X)$ with $\P(A) > 0$ such that $\bar K_0(x,x) > 0$ for all $x\in A$. 
\end{enumerate}
\end{remark}

\subsection{Nystr\"{o}m-accelerated Wild Bootstrap} \label{sec:Nyström-bootstrap}
While the last section recalled the validity of the Nystr\"{o}m-accelerated KSD estimator, this section is dedicated to the analysis of the Nystr\"{o}m-accelerated wild bootstrap.

To introduce the acceleration, let $R = \left( R_i \right)_{i=1}^n\sim\rho^n$ be Rademacher random variables, that is, $\rho(R_i = 1) = \rho(R_i = -1) = 1/2$ ($i\in[n]$), and note that one can equivalently express \eqref{eq:ksd-bootstrap} as
\begin{align}
	B_n^2 = \frac{1}{n^2}R\T\Knn R && \text{with} && \Knn = \left[ K_0(x_i,x_j) \right]_{i,j=1}^n \in \R^{n\times n}.
\end{align}
This formulation and using the low-rank kernel matrix approximation \citep{williams01using} $\Knn \approx \Knm \Kmm^- \Kmn$, with $\Kmm$ and $\Kmn$ defined according to \eqref{eq:k0-kmm} and \eqref{eq:k0-kmn}, respectively, and $\Knm \coloneq \Kmn\T \in \R^{n\times m}$, inspires the Nystr\"{o}m approximation \citep{kalinke25nystromksd}
\begin{align}
	B_n^2 \overset{(\dagger)}{\approx} \frac{1}{n^2} R\T \Knm \Kmm^- \Kmn R \eqcolon \tilde B_n^2. \label{eq:nystroem-bootstrap}
\end{align}
The error induced in $(\dagger)$ and its impact on the resulting GoF test are the main focus of the remainder of this work.

Our next lemma (proved in Section~\ref{sec:proof-projection-perspective}) connects the low-rank matrix approximation perspective \eqref{eq:nystroem-bootstrap} to the projection perspective detailed in Section~\ref{sec:Nyström-estimator}, on which we base our analysis.

\begin{lemma}[\lemmaprojectionperspectivetext]\label{lemma:projection-perspective}
	Suppose Assumption~\ref{ass:integrability} holds. Let $\left( X_i \right)_{i=1}^n\sim\P^n$, $\left( R_i \right)_{i=1}^n\sim\rho^n$ with $\rho$ as in \eqref{eq:rho-def}, $\left( I_j \right)_{j=1}^m\sim\Lambda^m$ with $\Lambda$ as in \eqref{eq:Lambda-def}, $\G = \frac1n\sum_{i=1}^nR_i \delta_{X_i}$, $\H_{K_0,m} = \Span\!\left(\fm {X_{I_j}} : j\in[m]\right)$, and $\tilde B_n^2$ as defined  in  \eqref{eq:nystroem-bootstrap}. Then 
       \begin{align}
		\tilde B_n^2 &= \rnorm{\projm\kme \G}^2,&
        \kme \G & \coloneq \frac{1}{n}\sum_{i=1}^n R_i K_0(\cdot,X_i),
	   \end{align}
	with $\projm$ defined above Remark~\ref{remark:positive-definite}.
\end{lemma}

In particular, Lemma~\ref{lemma:projection-perspective} allows us to show that the Nystr\"{o}m-based wild bootstrap \eqref{eq:nystroem-bootstrap} is a consistent estimator of the wild bootstrap \eqref{eq:ksd-bootstrap}, as we establish next. The following result is proved in Section~\ref{sec:proof-consistency-nystroem-bootstrap}.

\begin{theorem}[Consistency of Nystr\"{o}m bootstrap] \label{thm:consistency-nystroem-bootstrap}
Let $\mathcal N_{\bar K_0}(\lambda)$ ($\lambda >0$), $\Lambda$, and $\rho$ be as in \eqref{eq:centered-eff-dim}, \eqref{eq:Lambda-def} and \eqref{eq:rho-def}, respectively. Suppose Assumptions~\ref{ass:integrability} and~\ref{ass:sub-gaussian} hold, $\covc \neq 0$, and $m\ge 3$. Then, for any $\delta\in(0,1)$, it holds with $\left( \P^n \otimes \Lambda^m \otimes \rho^n \right)$-probability of at least $1-\delta$ that
	\begin{align}
		B_n - \tilde B_n &\lesssim \frac{\sqrt{\log(m)}\log(8n/\delta)}{\sqrt{nm}}\sqrt{\mathcal N_{\bar K_0}\!\left(\frac{c_1}{m}\right)}, \label{eq:Nystrom-boostrap-error}
	\end{align}
	  given that $m \gtrsim \max\left\{\left(\tfrac{8}{\delta}\right)^{c_2\trace\big(\covc\big)},\log\left(\tfrac{8}{\delta}\right),\opnorm{\covc}^{-1}\log (m)\right\}$ with some absolute constants $c_1,c_2>0$.
\end{theorem}

\begin{remark}~ \label{remark:projection-perspective}
	\begin{enumerate}[label=(\alph*)]
		\item \tb{Consistency of N-KSD bootstrap.} Since $\mathcal N_{\bar K_0}(\lambda) \lesssim 1/\lambda$,
        Theorem~\ref{thm:consistency-nystroem-bootstrap} shows that, for any fixed $m$ large enough, $B_n-\tilde B_n = \O_P\!\left(n^{-1/2}\log(n)\right) = o_P(1)$ as $n\to \infty$. However, when constructing Nystr\"{o}m-accelerated GoF tests, one key question is the weak limit of $n\tilde B_n^2$ as $n,m\to \infty$. We present our corresponding result under additional assumptions on $\mathcal N_{\bar K_0}(\lambda)$ in Theorem~\ref{thm:limit-nystroem-ksd}. 
        \item \tb{Projection perspective.} \label{remark:item:projection-perspective} We note that the l.h.s.\ of \eqref{eq:Nystrom-boostrap-error} is nonnegative by using that (i) $B_n  =  B_n(X_1,\ldots,X_n,R_1,\ldots,R_n) = \rnorm{\kme \G}$  by \eqref{eq:ksd-bootstrap},
        (ii) $\tilde B_n = \tilde B_n(X_1,\ldots,X_n,I_1,\ldots,I_m,R_1,\ldots,R_n) = \rnorm{\projm\kme \G}$
        by Lemma~\ref{lemma:projection-perspective}, and (iii) the fact that a projection ($\projm$) is norm decreasing; therefore, we omit the usual~$|\cdot|$.
	\end{enumerate}
\end{remark}

Theorem~\ref{thm:consistency-nystroem-bootstrap} with Theorem~\ref{theorem:bootstrap-v-stat} allows us to obtain the following result, which is proved in Section~\ref{sec:proof-limit-nystroem-ksd}.

\begin{theorem}[Asymptotic distribution of N-KSD bootstrap]\label{thm:limit-nystroem-ksd}  Let Assumptions~\ref{ass:integrability}--\ref{ass:sub-gaussian} hold and let $\X$ be a (separable) metric space. Suppose $\P_0\in\mathcal M_1^+(\X)$ is such that $\covcn \neq 0$. Let $\left( R_i \right)_{i=1}^n\sim\rho^n$ with $\rho$ as in \eqref{eq:rho-def}, $\tilde B_n^2$ as in \eqref{eq:nystroem-bootstrap}, and $\mathcal N_{\bar K_0}(\lambda)$ ($\lambda >0$) as in \eqref{eq:centered-eff-dim}. Then, it holds  for $\P_0^\infty$-almost all sequences $\left( X_i \right)_{i=1}^\infty$ and for $\Lambda^\infty$-almost all sequences $(I_j)_{j=1}^\infty$ that
\begin{align}
n\tilde B_n^2 \dconv W, \label{eq:W-approximation}
\end{align}
with $W$ defined in \eqref{eq:def-limit-h0-w}, as $n,m \to \infty$,
		if
		\begin{enumerate}[label=(\roman*)]
			\item  $\mathcal N_{\bar K_0}(\lambda) \lesssim \lambda^{-\gamma}$ for some $\gamma \in (0,1]$ and $m =  \omega\!\left( \log^{\frac{3}{1-\gamma}} n \right)$, or
			\item $\mathcal N_{\bar K_0}(\lambda) \lesssim\log\!\left( 1+\tilde c/\lambda \right)$ for some $\tilde c>0$ and $m=\omega\!\left(\log^4 (n)\right)$.
		\end{enumerate}
\end{theorem}

\begin{remark}
	Recall that the Nystr\"{o}m-based bootstrap \eqref{eq:nystroem-bootstrap} can be computed in $\O\! \left(mn + m^3\right)$ time, that is, $m=o\!\left(n^{2/3}\right)$ guarantees an asymptotic speedup of \eqref{eq:nystroem-bootstrap} over \eqref{eq:ksd-bootstrap}. Hence,  Theorem~\ref{thm:limit-nystroem-ksd} implies that $m$ can be chosen such that the Nystr\"{o}m-based bootstrap has lower asymptotic runtime but the same limiting distribution.
\end{remark}

It remains to investigate if the Nystr\"{o}m-based acceleration
\begin{align}
	\tilde S_n \coloneq \tilde S_n(X_1,\ldots,X_n,I_1,\ldots,I_m) \coloneq \bm 1_{\left\{ n\tilde D_{\P_0}^2\!\left( \hat \P_n \right) > q_{W,1-\alpha} \right\}}, \label{eq:n-ksd-test}
\end{align}
of the quadratic time test \eqref{eq:test-definition}, preserves the statistical behavior of the latter. Moreover, the impact of replacing $q_{W,1-\alpha}$ by its Nystr\"{o}m-bootstrapped estimate 
\begin{align}
    \mqn \coloneq \mqn(x_1,\ldots,x_n,i_1,\ldots,i_m,R_1,\ldots,R_n) \label{eq:bootstrapped-quantile}
\end{align}
obtained with \eqref{eq:nystroem-bootstrap} is also open. We answer these questions in the following Theorem~\ref{thm:separation-n-ksd} (proved in Section~\ref{sec:proof-separation-n-ksd}), Remark~\ref{remark:tilde-qw-op1}, and Corollary~\ref{corr:n-ksd/n-bootstrap}.

\begin{theorem}[Local consistency of N-KSD test] \label{thm:separation-n-ksd}
	Let Assumptions~\ref{ass:integrability}--\ref{ass:sub-gaussian} hold, $\tilde S_n$ as in \eqref{eq:n-ksd-test}, assume $\theta > 0$,  $n^{\frac{2\theta}{2\theta+1}}\Delta_n \to \infty$ as $n\to \infty$,
	\begin{align}
		\sup_{\P\in \mathcal P_n}\norm{\integralop^{-\theta}u_\P}{L^2(\X,\P_0)} < \infty,
	\end{align}
    and $\inf_{\P\in\mathcal P_n}\opnorm{\covc} > 0$.
	Then $
		\beta(\tilde S_n,\mathcal P_n \times \{\Lambda^m\}) \to 0$ as $n\to \infty$, given that 
        \begin{enumerate}[label=(\roman*)]
            \item  $\sup_{\P\in \mathcal P_n}\mathcal N_{\bar K_0}(\lambda) \lesssim \lambda^{-\gamma}$ for some $\gamma\in(0,1]$ and $m = \omega\!\left(n^{\frac{1}{2-\gamma}}\log^\frac{3}{2-\gamma} n \right)$, or 
            \item $\sup_{\P\in \mathcal P_n}\mathcal N_{\bar K_0}(\lambda) \lesssim \log(1+\tilde c/\lambda)$ for some $\tilde c>0$ and $m= \omega\left(\sqrt n \log^{2}(n)\right)$.
        \end{enumerate}
\end{theorem}

\begin{remark} Notice that the number of Nyström points ($m$) in Theorem~\ref{thm:separation-n-ksd} is strictly larger than in Theorem~\ref{corr:decay-assumption} (discarding logarithmic factors).
\end{remark}

\begin{remark}\label{remark:tilde-qw-op1}
    Let $\mqn$ be as in \eqref{eq:bootstrapped-quantile}. As $n\tilde B_n^2$ and $nB_n^2$ have the same limit $W$ for almost all sequences $(X_i)_{i=1}^\infty$ and $(I_j)_{j=1}^\infty$ (as shown in Theorem~\ref{thm:limit-nystroem-ksd}), which is continuous, one has for any fixed $\alpha \in (0,1)$ that $\mqn - q_{W,1-\alpha} =  o_P(1)$ conditionally on $(X_i)_{i=1}^\infty$ and $(I_j)_{j=1}^\infty$.\footnote{This result is implied by \citet[Lemma~21.2]{vaart98asymptotic} and upon noting that convergence in distribution implies convergence in probability for constant limits.}  Hence, one may use the accelerated test
		\begin{align}\label{eq:accelerated-test}
			\tilde S_n' \coloneq \tilde S_n'(X_1,\ldots,X_n,I_1,\ldots,I_m,R_1,\ldots,R_n) \coloneq \bm 1_{\left\{n\tilde D_{\P_0}^2(\hat \P_n) > \mqn\right\} },
		\end{align}
		which has asymptotic level $\alpha$ and is consistent against all fixed alternatives.\footnote{We note that computing $\tilde S_n'$ requires the computation of $\tilde q_{W,1-\alpha,n}$ and hence depends on $\left(\rho^n\right)^{c_b}$, with $c_b$ being the number of bootstrap samples. As the threshold is typically computed independently, we hide the dependence.}     
\end{remark}    

        In fact, even local consistency of the accelerated test can be guaranteed, as is captured in the following result (proved in Section~\ref{sec:proof-corollary}).

\begin{corollary}[Local consistency of N-KSD test with Nystr\"{o}m bootstrap] \label{corr:n-ksd/n-bootstrap} Let $\X$ be a (separable) metric space and let Assumptions~\ref{ass:integrability}--\ref{ass:sub-gaussian} hold. Suppose that 
\begin{align}
    \inf_{\P\in\mathcal P_n} \opnorm{\covc} > 0. 
\end{align}
Let  $\mathcal N_{\bar K_0}(\lambda)$ ($\lambda >0$) be as in \eqref{eq:centered-eff-dim} and $\tilde S_n'$ as in \eqref{eq:accelerated-test}.  Assume that $\theta > 0$,  $n^{\frac{2\theta}{2\theta+1}}\Delta_n \to \infty$ as $n\to \infty$,
	\begin{align}
		\sup_{\P\in \mathcal P_n}\norm{\integralop^{-\theta}u_\P}{L^2(\X, \P_0)} < \infty,
	\end{align}
    and $\covcn \neq 0$.
	Then, it holds that $\beta(\tilde S_n',\mathcal P_n \times \{\Lambda^m\}\times \{\rho^n\}) \to 0$ for $\P_0^\infty$-almost all $(X_i)_{i=1}^\infty$ and $\Lambda^\infty$-almost all $(I_j)_{j=1}^\infty$ sequences  as $n,m \to \infty$, given that  $\mathcal N_{\bar K_0}(\lambda)$ and $m$ satisfy (i) or (ii) of Theorem~\ref{thm:separation-n-ksd}, respectively.
\end{corollary}

We make two final remarks before collecting our experiments.

\begin{remark}\label{remark:final-remark}~
	\begin{enumerate}[label=(\alph*)]
		\item \tb{Computational-statistical trade-off.} Recall that the computation of the ``classical'' KSD-based test \eqref{eq:test-definition} costs $\O\!\left( n^2 + c_bn^2 + c_b\log (c_b)\right)$ as the computation of the test statistic has quadratic cost, and one needs to obtain and sort $c_b$ bootstrap samples for estimating the quantile. We omit the runtime cost of sorting in the following, due to its small contribution. The proposed Nystr\"{o}m-based acceleration \eqref{eq:accelerated-test} changes the runtime cost to $\O\!\left( mn + m^3 + c_b(mn+m^3) \right)$, improving upon the quadratic cost for $m=o(n^{2/3})$. In the case of polynomial decay and $\gamma < 1/2$, or in the case of exponential decay, this choice of $m$ yields a test that is consistent against local alternatives, with separation rate established in Corollary~\ref{corr:n-ksd/n-bootstrap}---identical to that of the quadratic-time test, see below---but with asymptotic computational gains.
		\item \label{item:class-of-alternatives} \tb{Class of alternatives.} Notice that (i) and (ii) of Theorem~\ref{thm:separation-n-ksd} restrict the class of alternatives $\mathcal P_n$ such that the respective decay assumption w.r.t.\ the effective dimension $\mathcal N_{\bar K_0}$ holds uniformly. While the class $\mathcal P_n$ considered in Theorem~\ref{thm:separation-n-ksd} might be smaller than the class considered in  Theorem~\ref{thm:separation-ksd}, this additional requirement is what allows us to obtain a matching rate. Our experiments show tangible results in practical cases, indicating that the class of considered alternatives is large enough.
	\end{enumerate}
\end{remark}

\section{Simulation Studies} \label{sec:experiments}

This section collects our experiments, comparing GoF testing with the Nystr\"{o}m-based KSD \eqref{eq:nystroem-ksd-estimator} with our proposed Nystr\"{o}m-accelerated bootstrap \eqref{eq:nystroem-bootstrap} to the quadratic-time KSD \eqref{eq:v-stat-ksd} with the quadratic-time bootstrap \eqref{eq:ksd-bootstrap}.

The experiments in \citet{kalinke25nystromksd} showed the convincing performance of the Nystr\"{o}m-accelerated bootstrap with the Langevin-Stein-based KSD on Euclidean spaces. To complement their results, we apply the Nystr\"{o}m-based accelerations to GoF testing with the directional KSD defined on the unit sphere (Section~\ref{sec:experiments-directional}), and to testing with KSD on functional data (Section~\ref{sec:experiments-functional}). All experiments were performed on a PC with Ubuntu 20.04, 124GB RAM, and 32 cores with 2GHz each. The source code replicating the experiments is available at \url{https://github.com/FlopsKa/fast-ksd-testing}.

\subsection{Nystr\"{o}m KSD test on directional data} \label{sec:experiments-directional}

We consider data on the $d$-dimensional unit sphere $\mathcal S^{d-1}  \coloneq \{\b x \in \R^d : \|\b x\|_2 = 1\}$ and the von Mises-Fisher kernel $k(\b x,\tilde{\b x}) = \exp \big(\gamma \b x\T \tilde {\b x}\big)$ with $\b x,\tilde{\b x} \in \mathcal S^{d-1}$ and $\gamma >0$. The mapping from spherical coordinates $\bm \theta = (\theta_1,\ldots,\theta_{d-1})$ to Cartesian coordinates $(x_1,\ldots,x_d)$ takes the form, for $d\ge 2$,
\begin{align}
    x_k = \cos \theta_k \prod_{i=1}^{k-1}\sin \theta_i \text{ for all } k = 1,\ldots,d-1, &&\text{ and } &&
    x_d = \prod_{i=1}^{d-1}\sin \theta_i, \label{eq:sphere-mapping}
\end{align}
where $\theta_i \in [0,\pi]$ for $i = 1,\ldots,d-2$ and $\theta_{d-1} \in [0,2\pi)$. Let $J(\bm \theta) = \prod_{i=1}^{d-2}\sin^{d-i-1}(\theta_i)$. \citet{xu20directionalksd} have shown that a Stein kernel on $\mathcal S^{d-1}$---associated to a smooth target density $p_0$ on the sphere---is given by
\begin{align*}
	K_0(\bm \theta,\tilde{\bm \theta}) \!=\!& \sum_{i=1}^{d-1}\!\left[k(\b x,\tilde{\b x})\dfrac{\partial}{\partial \theta_i} \log(p_0(\bm \theta)J(\bm\theta))\dfrac{\partial}{\partial \tilde \theta_i}\log(p_0(\tilde{\bm \theta})J(\tilde{\bm \theta})) \right. \\
	&\!\left. + \dfrac{\partial}{\partial \theta_i} \log(p_0(\bm \theta)J(\bm\theta))\dfrac{\partial}{\partial \tilde \theta_i}k(\b x,\tilde{\b x}) \!+\! \dfrac{\partial}{\partial \tilde \theta_i} \log(p_0(\bm \tilde {\bm\theta})J(\tilde{\bm\theta}))\dfrac{\partial}{\partial \theta_i}k(\b x,\tilde{\b x}) \!+\! \dfrac{\partial^2}{\partial\theta_i\partial\tilde\theta_i}k(\b x,\tilde{\b x})\right]\!,
\end{align*}
where $\b x$ and $\tilde{\b x}$ are identified through mapping $\bm \theta$ and $\tilde{\bm \theta} = (\tilde \theta_1,\ldots,\tilde \theta_{d-1})$ with \eqref{eq:sphere-mapping}, respectively.

We set the uniform density on the sphere as the target, that is, $p_0(\bm \theta)\propto 1$, and consider $d=2,3$, replicating the experiments in \citet{xu20directionalksd} to simplify comparison. We test at the level of $\alpha = 0.01$ and use $c_b = 1000$ bootstrap samples to approximate the null distribution. For estimating the nominal level ($H_0$ holds) and the power ($H_1$ holds) of both tests, each experiment is repeated $600$ times, on different draws of the data.

To approximate the power, we sample from the von Mises-Fisher distribution (resp.\ its specific case if $d=2$, the von Mises distribution), which has density 
\begin{equation}
    p(\bm \theta) = \frac{e^{\kappa \bm\mu\T \b x}}{N_d(\kappa)},
\end{equation}
with direction vector $\bm \mu\in\mathcal S^{d-1}$, concentration parameter $\kappa >0$, and normalization constant
\begin{equation}
	N_d(\kappa) = \frac{\kappa^{d/2-1}}{(2\pi)^{d/2}B_{d/2-1}(\kappa)},
\end{equation}
where $B_v$ is the modified Bessel function of the first kind and order $v$. This distribution is unimodal and peaks at $\bm \mu$. Increasing $\kappa$ renders detecting the alternative easier. For $d=2$, we set $\bm \mu = (1,0)$ and $\kappa=0.5$; for $d=3$, $\bm \mu = (1,0,0)$ with varying $\kappa \in (0,6]$. The number of Nystr\"{o}m samples is $m = \sqrt n$, with $n$ the number of samples obtained from the alternative. We optimize the parameter $\gamma > 0$ of the base kernel $k$ on separate draws from the alternative. For $d=2$, we obtain $\gamma =0.12$; for $d=3$, we arbitrarily fix $\kappa = 2$ and obtain $\gamma= 0.28$.

\begin{figure}
	\centering
	\includegraphics[width=\linewidth]{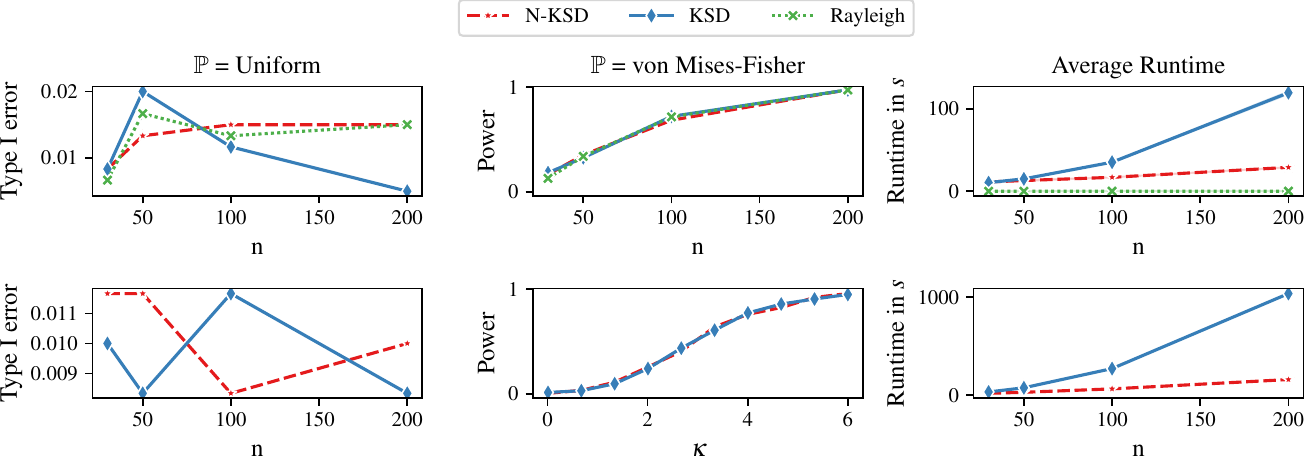}
	\caption{Results of approximating the nominal level, power, and the total average runtimes (including the bootstrap computation) for $d=2$ (top), $d=3$ (bottom), and different choices of the sample size $n$ and the concentration parameter $\kappa$.} \label{fig:directional-ksd-results-d2}
\end{figure}

Figure~\ref{fig:directional-ksd-results-d2} shows our results for $d=2,3$, comparing the Nystr\"{o}m-accelerated KSD test to the existing quadratic-time test. The latter has shown the best performance in similar experiments of \citet{xu20directionalksd}. As in their work, we additionally include the Rayleigh test as a baseline for $d=2$. 
For the power experiments with $d=3$, we fix the sample size at $n=200$, and increase the concentration parameter $\kappa$ of the von Mises-Fisher distribution from $0.01$ to $6$, simplifying detection.

Regarding the level, our results indicate that all tests operate on the nominal level of $\alpha = 0.01$, up to statistical fluctuations. Regarding the power, the figure shows that the proposed method is on par with the quadratic-time approach, although with a runtime that is orders of magnitude lower. We emphasize that the runtime gains possible through the Nystr\"{o}m approximation are further amplified by the requirement of obtaining a sufficient number of bootstrap samples; each repetition profits from the reduced runtime complexity.

\subsection{Nystr\"{o}m KSD Test on Functional Data} \label{sec:experiments-functional}

In this section, we apply the Nystr\"{o}m acceleration to GoF testing to functional data. In particular, we employ the setup of \citet{wynne25fourierksd} (also employed in \citet{hagrass24stein}), which we recall in the following.

Let $C$ denote the covariance operator on $\X = L^2([0,1])$ with eigenvalues $\lambda_i = (i-1/2)^{-2}\pi^{-2}$ and corresponding eigenvectors $e_i(t) = \sqrt 2 \sin ((i-1/2)\pi t)$ for $i \in \mathbb N_{>0}$ and $t\in[0,1]$. The target distribution is Brownian motion over $[0,1]$, corresponding to the centered Gaussian measure on $\X$ with covariance $C$, that is, the unique measure on $\X$ whose pushforward under $x\mapsto \ip{y,x}{\X}$ is Gaussian with mean zero and variance $\ip{Cy,y}{\X}$ for all $y\in\X$.
The kernel function is $k(x,y) = e^{-\frac{1}{2\gamma^2}\norm{Tx-Ty}{\X}^2}$ with $\gamma > 0$ chosen by the median heuristic and $Tx = \sum_{i=1}^\infty\eta_i\ip{x,e_i}{\X}e_i$, where
\begin{align*}
	\eta_i = \begin{cases}
		\frac{1}{\lambda_i} \text{ for } 1 \le i \le 50, \\
		1 \text{ for } i > 50.
	\end{cases}
\end{align*}
This cutoff emphasizes higher frequency activity w.r.t.\ the Brownian motion basis and has shown good results in the experiments that we replicate. We refer to \citet[Section 5.2]{wynne25fourierksd} for additional information and for the definition of the Stein operator applied to $k$ to obtain the Stein kernel. The observed data is discretized to $100$ points on a uniform grid on $[0,1]$. We repeat each experiment $500$ times to approximate the rejection rates, use $c_b=1000$ bootstrap samples for each run, fix the level $\alpha=0.05$, and set $m = 4 \sqrt{n}$ for the Nystr\"{o}m approximations.

\begin{figure}
	\centering
	\includegraphics[width=\linewidth]{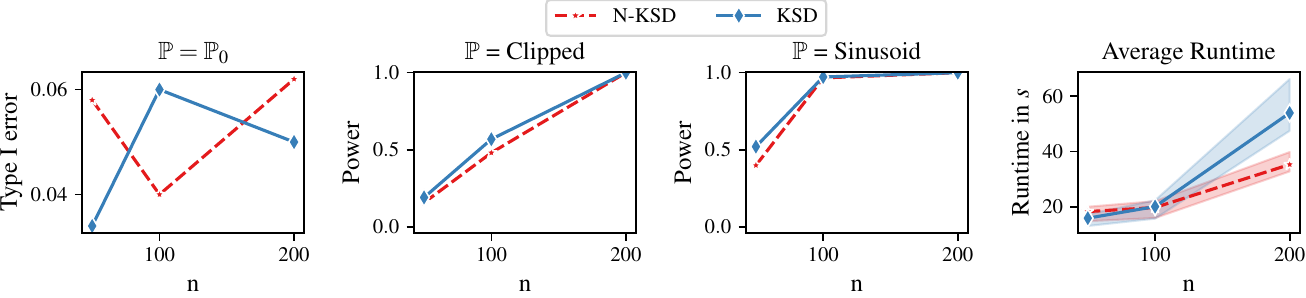}
	\caption{GoF testing on functional data with $n$ samples of $\P$. The target distribution is Brownian motion, the respective sampling distributions are indicated on the top of the figures, with the average runtime shown on the r.h.s.} \label{fig:results-wynne}
\end{figure}

Figure~\ref{fig:results-wynne} collects the rejection rates for $n\in \{50,100,200\}$ for different sampling distributions together with the average runtime. The ``clipped'' sampling distribution corresponds to a law of Brownian motion clipped to $\sum_{i=1}^{8}\sqrt{\lambda}\xi_i e_i$ with $\xi_i \overset{\text{i.i.d.}}{\sim} \mathcal N(0,1)$, while the ``sinusoid'' corresponds to sampling from a random variable $X_t = (1+\sin(2\pi t))B_t$, with $B_t$ standard Brownian motion. These setups were also considered in \citet{bongiorno19functionaltest}.

Like the results of our experiments on directional data in Section~\ref{sec:experiments-directional}, our results indicate that the Nystr\"{o}m accelerated test statistic and bootstrap perform similarly to their quadratic-time counterparts w.r.t.\ level and power while reducing the runtime. Indeed, for $n=200$, the Nystr\"{o}m accelerated test requires almost only half the time of the quadratic-time test, still matching the slower test in terms of power.

\section{Proofs} \label{sec:proofs}
The proofs of the results in the main text are presented in this section.
\subsection{Proof of Proposition~\ref{lemma:stein-op-properties}}\label{sec:proof-lemma-stein-op-properties}
We prove each bullet separately.
\begin{enumerate}[label=(\roman*)]
	\item As $\range(\mathcal T_\Q) = \{\mathcal T_\Q f : f\in \H\}$, we must show that $\mathcal T_\Q f \in \mathcal L^r(\X, \Q')$ for any $f \in \H$. Indeed, let $f \in \H$ be fixed. Then
	\begin{equation}
		\int_\X |(\mathcal T_\Q f)(x)|^r\d \Q'(x) \overset{\eqref{eq:mean-zero-property}}{=} \int_\X \big|\ip{\Psi_\Q(x),f}{\H}\big|^r\d \Q'(x) \overset{\text{CBS}}{\le} \norm{f}{\H}^r \int_\X \norm{\Psi_\Q(x)}{\H}^r \d \Q'(x),
	\end{equation}
	which is finite as $x\mapsto \norm{\Psi_\Q(x)}{\H} \in \mathcal L^r(\X, \Q')$ was assumed.
	\item Observe that, for any $f\in \H$,
	\begin{equation}
		\norm{\mathcal T_\Q f}{\infty} = \sup_{x\in \X}|(\mathcal T_\Q f)(x)| \overset{\eqref{eq:mean-zero-property}}{=} \sup_{x\in \X} \big|\ip{\Psi_\Q(x),f}{\H}\big| \overset{\text{CBS}}{\le} \norm{f}{\H} \sup_{x\in\X}\norm{\Psi_\Q(x)}{\H},
	\end{equation}
	which is finite as $\norm{x\mapsto \norm{\Psi_\Q(x)}{\H}}{\infty} < \infty$ was assumed.
	\item For arbitrary $f\in \H$ and $x, y \in \X$, we have that
	\begin{equation}
			\left|(\mathcal T_\Q f) (x) - (\mathcal T_\Q f)(y)\right| \overset{(a)}{=} |\ip{\Psi_\Q(x)- \Psi_\Q(y),f}{\H}|  \overset{\text{CBS}}{\le} \norm{\Psi_\Q(x)-\Psi_\Q(y)}{\H}\norm{f}{\H},
		\end{equation}
		where (a) is by \eqref{eq:mean-zero-property} and linearity of the inner product. As $\Psi_\Q$ is Hölder continuous, this proves the claim.
\end{enumerate}
The combinations stated in the lemma are immediate consequences of (i)--(iii).

\subsection{Proof of Theorem~\ref{thm:v-stat-consistency2}} \label{sec:proof-consistency-v-stat}
The proof mirrors that of \citet[Theorem~3]{kalinke25nystromksd}---omitting the relaxation of the $\psi_2$ to the $\psi_1$-norm---and is included for completeness.

By the reverse triangle inequality, one has that
\begin{align}
	\left|\D - D_{\P_0}\!\left( \hat \P_n \right)\right| &\le \rnorm{\kme{\P} -\kme{\hat \P_n}} \\
	&= \Bigg\|\frac1n\sum_{i=1}^{n}\underbrace{\left[\fm{X_i} - \E_{\P} \fm X\right]}_{\eqcolon\eta_i}\Bigg\|_{\H_{K_0}}, \label{eq:relaxation-sq-ksd-bound}
\end{align}
which measures the concentration of i.i.d.\ centered random variables. To obtain the bound, we will use Bernstein's inequality (recalled in Theorem~\ref{thm:bernstein-unbounded}) by gaining control on the moments of $\rnorm{\eta_i}$ with Lemma~\ref{lemma:sub-exp-bernstein}.

First, note that the $\rnorm{\eta_i}$-s ($i\in[n]$) are sub-exponential as
\begin{align}
	\psione{\rnorm{\eta_i}}&\stackrel{(a)}{=}\psione{\rnorm{\fm{X_i} - \E_{\P} \fm X}}\\
	&\stackrel{(b)}{\le} \psione{\rnorm{\fm{X_i}}  +\rnorm{\E_{\P} \fm X}} \\
	&\stackrel{(c)}{\le}  \psione{\rnorm{\fm{X_i}}  +\E_{\P}\rnorm{ \fm X}} \stackrel{(d)}{\lesssim} \psione{\rnorm{\fm{X_i}}} < \infty.\qquad \label{eq:etai-sub-exp}
\end{align}
(a) is by the definition of $\eta_i$. (b) is implied by the triangle inequality and the monotonicity of norms. (c) is by Jensen's inequality holding for Bochner integrals and the monotonicity of norms, and (d) comes from the triangle inequality  followed by Lemma~\ref{lemma:orlicz-properties}(2.). Finiteness is due to the imposed assumption.

Hence, by Lemma~\ref{lemma:sub-exp-bernstein}, it holds for any $p\ge 2$ that
\begin{align}
	\E_{X\sim\P} \rnorm{\eta_1}^p \le \frac12p!\sigma^2B^{p-2},
\end{align}
with $\sigma, B \lesssim \psione{\rnorm{\eta_1}} \eqcolon K$. Now,  applying  Theorem~\ref{thm:bernstein-unbounded} yields that, for any $\delta \in (0,1)$, it holds with probability at least $1-\delta$ that
\begin{align}
	\rnorm{\frac 1n \sum_{i=1}^{n}\eta_i} \lesssim \frac{2K\log(2/\delta)}{n} + \sqrt{\frac{2K^2\log(2/\delta)}{n}},
\end{align}
which implies the stated claims by the relaxation in \eqref{eq:relaxation-sq-ksd-bound}.

\subsection{Proof of Theorem~\ref{thm:serfling-v-stat}} \label{sec:proof-serfling-v-stat}

We check the conditions $\mathcal A_m$ ($m\in[2]$) from \citet[Section~6.4.1]{serfling80approximation}, which we recall in Appendix~\ref{sec:external-statements}. Throughout the proof, we will be using the following well-known properties of tensor product Hilbert spaces:
\begin{align}
\langle a \otimes b, c \otimes d \rangle_{\H_{K_0} \otimes \H_{K_0}} &= \langle a,c\rangle_{\H_{K_0}} \langle b,d \rangle_{\H_{K_0}} \quad \text{ for } a,b,c,d\in \H_{K_0},\\
\langle f,Lg \rangle_{\H_{K_0}} &= \langle L, f\otimes g\rangle_{\mathrm{HS}\left(\H_{K_0}\right)} \quad \text{ for } f,g\in \H_{K_0}, L \in \mathrm{HS}\!\left(\H_{K_0}\right),
\end{align}
where $\mathrm{HS}(\H_{K_0})$ denotes the space of $\H_{K_0} \to \H_{K_0}$ Hilbert-Schmidt operators (known to be isomorphic to the tensor product Hilbert space $\H_{K_0}\otimes \H_{K_0}$), defined for separable Hilbert spaces. In our case separability holds by Remark~\ref{remark:integrability}\ref{item:HK0-sep}.

To start, we compute the respective quantities using the von Mises calculus outlined therein, setting $h(x,y) = K_0(x,y)$;\footnote{See also the summary in \citet[Section 6.5]{serfling80approximation}.} in this case $T(\P) = D_{\P_0}^2(\P)$ and $T\big(\hat \P_n\big) = D_{\P_0}^2\big(\hat \P_n\big)$. We will show that:
\begin{align}
	d_1T(\P; \Q-\P) &= 2\ip{\kme{\P},\kme{\Q}}{\H_{K_0}} - 2\rnorm{\kme{\P}}^2, \label{eq:d1T}\\
	d_2T(\P; \Q-\P) &= 2\rnorm{\kme{\Q}-\kme{\P}}^2, \label{eq:d2T} \\
	h(\P;x) &= 2\ip{\kme{\P},\fm{x}}{\H_{K_0}} - 2 \rnorm{\kme{\P}}^2, \label{eq:h(P;x)}\\
	h(\P;x,y) &= K_0(x,y) - \rnorm{\kme{\P}}^2,  \label{eq:h(P;x,y)}\\
    \Var_\P(h(\P;X)) &= 4\rnorm{\covc^{1/2}\kme{\P}}^2,	 \label{eq:var-serfling-1}\\
	\Var_{\P^2}(h(\P;X,Y) &= \norm{\cov}{\H_{K_0}\otimes \H_{K_0}}^2 \hspace{-0.1cm}-  \rnorm{\kme{\P}}^4.\hspace{.2em} \label{eq:var-serfling-2}
\end{align}

Indeed, the derivations of these identities are as follows.
\begin{itemize}
    \item $d_1T(\P; \Q-\P)$ and $d_2T(\P; \Q-\P)$ [\eqref{eq:d1T} and \eqref{eq:d2T}]: By definition
\begin{align}
d_1T(\P; \Q-\P) & = \frac{\d}{\d \lambda} f(\lambda)\Big |_{\lambda\downarrow 0}, & d_2T(\P; \Q-\P) &= \frac{\d^2}{\d \lambda^2} f(\lambda)\Big |_{\lambda\downarrow 0},
\end{align}
with
\begin{align}
f(\lambda) &= D_{\P_0}^2(\P + \lambda(\Q - \P)) \stackrel{\eqref{eq:population-ksd}}{=} \rnorm{\kme{\P + \lambda(\Q - \P)}}^2\\
& \stackrel{(a)}{=} \rnorm{\kme{\P}+\lambda \kme{\Q-\P}}^2\\
& \stackrel{(b)}{=} \rnorm{\kme{\P}}^2 + 2 \lambda \rip{\kme{\P},\kme{\Q-\P}} + \lambda^2\rnorm{\kme{\Q-\P}}^2,
\end{align}
where in (a) we use that the kernel mean embedding can be defined on finite signed measures \citep{sejdinovic13kernel}; the fact that in a Hilbert space the norm is induced by its inner product and that the inner product is linear were used in (b). We then get
\begin{align}
\frac{\d}{\d \lambda} f(\lambda)\Big |_{\lambda\downarrow 0} & = 2 \rip{\kme{\P},\kme{\Q-\P}} + 2\lambda\rnorm{\kme{\Q-\P}}^2 \Big |_{\lambda\downarrow 0} \\
&=  2 \rip{\kme{\P},\kme{\Q-\P}} 
\stackrel{(a)}{=} 2 \rip{\kme{\P},\kme{\Q}-\kme{\P}}\\
& \stackrel{(b)}{=} 2 \rip{\kme{\P},\kme{\Q}}-2\rnorm{\kme{\P}}^2,\\
\frac{\d^2}{\d \lambda^2} f(\lambda)\Big |_{\lambda\downarrow 0} &= 2\rnorm{\kme{\Q-\P}}^2 \Big |_{\lambda\downarrow 0} 
=2\rnorm{\kme{\Q-\P}}^2 \\
&\stackrel{(a)}{=} 2\rnorm{\kme{\Q}-\kme{\P}}^2.
\end{align} 
(a) follows from the linearity of the mean embedding, in (b) the linearity of the inner product and the fact that the norm in a Hilbert space is induced by its inner product were used. This shows the claimed \eqref{eq:d1T} and \eqref{eq:d2T}.
\item $h(\P;x)$ [\eqref{eq:h(P;x)}]: For  $x\in\X$, specializing the expression obtained for $d_1T(\P; \Q-\P)$ in \eqref{eq:d1T} by choosing $\Q=\delta_x$, we get that 
\begin{align}
h(\P;x) &= d_1T(\P;\delta_x-\P) = 2\ip{\kme{\P},\fm{x}}{\H_{K_0}} - 2 \rnorm{\kme{\P}}^2,
\end{align}
proving \eqref{eq:h(P;x)}.
\item $h(\P;x,y)$ [\eqref{eq:h(P;x,y)}]: Considering $h(\P;x,y)$ for $x,y\in \X$ [see \eqref{eq:h-def}], we get that
\begin{align}
\MoveEqLeft d_1T\!\left(\P;\hat \P_n - \P\right) + \frac{1}{2}d_2T\!\left(\P;\hat \P_n - \P\right) \\
& \stackrel{\mathclap{\eqref{eq:d1T}, \eqref{eq:d2T}}}{=} \hspace*{0.4cm}2\ip{\kme{\P},\kme{\hat \P_n}}{\H_{K_0}} - 2\rnorm{\kme{\P}}^2 + \frac{1}{2} 2\rnorm{\kme{\hat \P_n}-\kme{\P}}^2 \notag \\
& \stackrel{\mathclap{(a)}}{=} \rnorm{\kme{\hat\P_n}}^2 - \rnorm{\kme{\P}}^2 \\
&\stackrel{\mathclap{(b)}}{=} \frac{1}{n^2} \sum_{i,j\in [n]}\rip{ \fm{X_i}, \fm{X_j}} - \rnorm{\kme{\P}}^2\\
&\stackrel{\mathclap{(c)}}{=} \frac{1}{n^2} \sum_{i,j\in [n]} \left[K_0(X_i,X_j) - \rnorm{\kme{\P}}\right]\\
&\implies h(\P;x,y) = K_0(x,y) - \rnorm{\kme{\P}}^2.
\end{align}
In (a), we expand the squared norm and cancel like terms. (b) follows from the definitions of the mean embedding and the empirical measure, and by the linearity of the inner product. (c) is implied by the reproducing property and by rewriting $a=\frac{1}{n^2} \sum_{i,j\in [n]}a$ ($a\in \R$). This proves \eqref{eq:h(P;x,y)}.
\item $\Var_\P(h(\P;X))$ [\eqref{eq:var-serfling-1}]: Consider the decomposition
\begin{align}
\MoveEqLeft \Var_\P(h(\P;X)) \stackrel{\eqref{eq:h(P;x)}}{=} \Var_\P\left(2\ip{\kme{\P},\fm{X}}{\H_{K_0}} - 2 \rnorm{\kme{\P}}^2\right)\\
&\stackrel{(a)}{=} \Var_\P\left(2\ip{\kme{\P},\fm{X}}{\H_{K_0}}\right) \stackrel{(b)}{=} 4 \Var_\P\left(\ip{\kme{\P},\fm{X}}{\H_{K_0}}\right)\label{eq:var-term}\\
&\overset{(c)}{=} 4\E_{\P}\Big[\rip{\kme{\P},\fm{X}}-\E_{\P}\rip{\kme{\P},\fm{X}}\Big]^2 \\
    &\overset{(d)}{=} 4\E_{\P}\Big[\rip{\kme{\P},\fm{X}}-\rip{\kme{\P},\E_{\P}\fm X}\Big]^2 \\
    &\overset{(e)}{=} 4\E_{\P}\Big[\rip{\kme{\P},\fm{X} - \E_{\P}\fm X}\Big]^2 \\
    &\hspace*{-0.075cm}\overset{\eqref{eq:centered-stein-fm}}{=} 4\E_{\P}\Big[\rip{\kme{\P},\fmc{X}}\Big]^2 \\
    & \overset{(f)}{=} 4\E_{\P}\ip{\kme{\P}\otimes \kme\P,\fmc{X}\otimes \fmc X}{\H_{K_0}\otimes \H_{K_0}} \\
    & \overset{(d)}{=} 4\ip{\kme{\P}\otimes \kme\P,\E_{\P}\left[\fmc{X}\otimes \fmc X\right]}{\H_{K_0}\otimes \H_{K_0}} \\
    & \overset{(g)}{=} 4\ip{\kme{\P}\otimes \kme\P,\covc}{\H_{K_0}\otimes \H_{K_0}} 
     \overset{(h)}{=} 4\rip{\kme{\P},\covc\kme{\P}} \\
     & \overset{(i)}{=} 4\rip{\covc^{1/2}\kme{\P},\covc^{1/2}\kme{\P}} 
     \overset{(j)}{=} 4\rnorm{\covc^{1/2}\kme\P}.
\end{align}
In (a) we use that $\Var_\P(X+c) =  \Var_\P(X)$ for any $c\in \R$, in (b) the factor of 2 multiplier was pulled out. (c) is by the definition of variance. In (d), we exchange the expectation and the inner product \citep[(A.32)]{steinwart08support}. The linearity of the inner product gives (e). (f) is by the definition of the inner product on tensor product Hilbert spaces. In (g), the definition of $\covc$ is applied, and (h) is by the properties of tensor product Hilbert spaces. (i) follows from the definition of the adjoint operator and by using the self-adjointness of $\covc^{1/2}$ following from that of $\covc$. (j) uses that the norm on a Hilbert space is induced by its inner product.
This shows the stated \eqref{eq:var-serfling-1}.

\item $\Var_{\P^2}(h(\P;X,Y))$ [\eqref{eq:var-serfling-2}]: Consider
\begin{align}
\Var_{\P^2}(h(\P;X,Y)) &\stackrel{\eqref{eq:h(P;x,y)}}{=} \Var_{\P^2}\left(K_0(X,Y) - \rnorm{\kme{\P}}^2\right) \stackrel{(a)}{=} \Var_{\P^2}\left(K_0(X,Y)\right)\\
&\stackrel{(b)}{=} \E_{\P^2}\left( K_0^2(X,Y)\right) - \E_{\P^2}^2\left( K_0(X,Y)\right), \text{ where}\\
\E_{\P^2}\left( K_0^2(X,Y)\right) & \stackrel{(c)}{=} \E_{\P^2}\left[\rip{\fm{X},\fm{Y}}\rip{\fm{X},\fm{Y}}\right]\\
&\stackrel{(d)}{=} \E_{\P^2} \left[ \ip{\fm{X} \otimes \fm{X}, \fm{Y} \otimes \fm{Y}}{\H_{K_0}\otimes \H_{K_0}}\right],\\
&\stackrel{(e)}{=} \ip{\E_\P\left[\fm{X} \otimes \fm{X}\right], \E_\P\left[\fm{Y} \otimes \fm{Y}\right]}{\H_{K_0}\otimes \H_{K_0}}\\
&\stackrel{(f)}{=} \ip{\cov,\cov}{\H_{K_0}\otimes \H_{K_0}} \stackrel{(g)}{=} \left\| \cov \right\|_{\H_{K_0}\otimes\H_{K_0}}^2,\\
\E_{\P^2}\left( K_0(X,Y)\right) & \stackrel{(c)}{=} \E_{\P^2}\left[ \rip{\fm{X},\fm{Y}} \right] \stackrel{(e)}{=}  \rip{\E_\P[\fm{X}],\E_\P[\fm{Y}]} \notag \\
&\stackrel{(h)}{=} \rip{\mu_{K_0}(\P),\mu_{K_0}(\P)}
\stackrel{(g)}{=} \rnorm{\mu_{K_0}(\P)}^2.
\end{align}
In (a) we use that $\Var_\P(X+c) =  \Var_\P(X)$ for any $c\in \R$, (b) follows from the definition of the variance, the reproducing property of kernels gives (c), (d) follows from the definition of the inner product in $\H_{K_0} \otimes \H_{K_0}$, the expectation and the inner product were flipped in (e), (f) follows from the definition of $C_{\P,K_0}$, (g) holds as in a Hilbert space the norm is induced by its inner product, (h) is by the definition of mean embeddings. This gives the claimed \eqref{eq:var-serfling-2}.
\end{itemize}

Having established \eqref{eq:d1T}--\eqref{eq:var-serfling-2}, we now tackle parts (i) and (ii) of the statement separately.

\begin{enumerate}[label=(\roman*)]
	\item In this case \hl{$\P_0 \ne \P$}. Our goal is to apply Theorem~\ref{thm:serfling-A1} for which it suffices to check that condition $\mathcal{A}_1$ holds.  
	\begin{description}
		\item[$\mathcal{A}_1$(i).]
        \textbf{Condition $0<\Var_\P(h(\P;X))$:} Together with the reproducing property, \eqref{eq:var-term} shows that
		\begin{align}
			\Var_\P(h(\P;X)) &= 4\Var_\P(\kme{\P}(X)). \label {eq:variance-a1-proof}
		\end{align}
		To show that \eqref{eq:variance-a1-proof} is positive, we argue by contradiction, that is, we assume that $\Var_\P(\kme{\P}(X)) = 0$.
        Then, $\kme{\P}(X)$ is constant $\P$-a.s.\ by the definition of the variance. In other words, this means that \hl{$\kme{\P}(X)$ is degenerate} (see footnote \ref{fn:degenerate}), contradicting our assumption.
        
           \textbf{Condition $\Var_\P(h(\P;X))<\infty$:} By the definition of the Bochner integral, it holds that $\Var_\P(h(\P;X))\stackrel{\eqref{eq:var-serfling-1}}{=}4\rnorm{\covc^{1/2}\kme{\P}}^2 < \infty$ if \hl{$\E_{\P}K_0(X,X)<\infty$}; the latter was assumed in the theorem.

		\item[$\mathcal{A}_1$(ii).] \textbf{Condition $\sqrt n\big(T\big(\hat\P_n\big) - T(\P) - d_1T\big(\P;\hat \P_n-\P\big)\big) = o_P(1)$:} We have that
		\begin{align}
			 \MoveEqLeft T\big(\hat\P_n\big) - T(\P) - d_1T\big(\P;\hat \P_n-\P\big) \\ &
             \stackrel{(a)}{=} D_{\P_0}^2\big(\hat \P_n\big) - D_{\P_0}^2(\P) - 2\ip{\kme{\P},\kme{\hat \P_n}}{\H_{K_0}} + 2\rnorm{\kme{\P}}^2\\
             & \stackrel{\eqref{eq:population-ksd}}{=} \norm{\mu_{K_0}\left(\hat \P_n\right)}{\H_{K_0}}^2 \hspace{-0.35cm} - \norm{\mu_{K_0}(\P)}{\H_{K_0}}^2 \hspace{-0.25cm}- 2\ip{\kme{\P},\kme{\hat \P_n}}{\H_{K_0}} \hspace{-0.35cm}+ 2\rnorm{\kme{\P}}^2 \notag\\
             & \stackrel{(b)}{=}\norm{\mu_{K_0}\left(\hat \P_n\right)}{\H_{K_0}}^2 - 2\ip{\kme{\P},\kme{\hat \P_n}}{\H_{K_0}} + \rnorm{\kme{\P}}^2\\
            &\stackrel{(c)}{=} \rnorm{\kme{\P}- \kme{\hat \P_n}}^2 \stackrel{(d)}{=} O_P\!\left(n^{-1}\right).
		\end{align}
        (a) holds by the definition of $T$ and \eqref{eq:d1T}, the expression is simplified in (b), the fact that in a Hilbert space the norm is induced by its inner product and the symmetry of inner product give (c), and (d) follows from \eqref{eq:sqrt-n-sub-exp-rate} by using the imposed assumption \hl{$\norm{\rnorm{\fm{X}}}{\psi_1} < \infty$}.
		Hence, we have that 
		\begin{equation}
			\sqrt{n}\left(T\Big(\hat\P_n\Big) - T(\P) - d_1T\Big(\P;\hat \P_n-\P\Big)\right) = O_P\!\left(n^{-1/2}\right) = o_P(1).
		\end{equation}
	\end{description}
	With both conditions of $\mathcal{A}_1$ satisfied, Theorem~\ref{thm:serfling-A1} yields the stated result upon noting that 
    \begin{align}
    \mu(T,\P) &= \E_\P(h(\P;X)) 
    \stackrel{\eqref{eq:h(P;x)}}{=} \E_\P\left[2\ip{\kme{\P},\fm{X}}{\H_{K_0}} - 2 \rnorm{\kme{\P}}^2\right]\\
    &\stackrel{(a)}{=} 2\ip{\kme{\P},\E_\P\left[\fm{X}\right]}{\H_{K_0}} - 2 \rnorm{\kme{\P}}^2\\
    &\stackrel{(b)}{=} 2\ip{\kme{\P},\kme{\P}}{\H_{K_0}} - 2 \rnorm{\kme{\P}}^2
    \stackrel{(c)}{=} 0. \label{eq:mu-T-P-zero}
    \end{align}
    (a) comes from the linearity of the expectation and by changing the inner product with the expectation, (b) is by the definition of the mean embedding, and the fact that in a Hilbert space the norm is induced by its inner product implies (c).
	
	\item In this case \hl{$\P_0 = \P$}. Our goal is  to apply Theorem~\ref{thm:serfling-A2}, hence we verify its conditions.
	\begin{description}
		\item[$\mathcal{A}_2$(i).] \textbf{Condition $\Var_{\P_0}(h(\P_0;X)) = 0$:} By evaluating the obtained variance expression \eqref{eq:var-serfling-1} at $\P = \P_0$, it holds that
        \begin{align}
        \Var_\P(h(\P;X))\big|_{\P=\P_0} &= 4\rnorm{\covc^{1/2}\kme{\P}}^2\Bigg|_{\P=\P_0}
        =4\rnorm{\covcn^{1/2}\kme{\P_0}}^2.
        \end{align}    
        As $\kme{\P_0} = 0$ holds by Remark~\ref{remark:integrability}\ref{remark:item:e-p0-equals-0}, we get that $\Var_{\P_0}(h(\P_0;X))=0$.
        
        \item[$\mathcal{A}_2$(ii).] \textbf{Condition $\Var_{\P_0^2}(h(\P_0;X,Y))>0$:} Evaluating the obtained variance expression \eqref{eq:var-serfling-2} at $\P = \P_0$, we get that
        \begin{align}
        \MoveEqLeft\Var_{\P^2}(h(\P;X,Y))\big|_{\P=\P_0} = \norm{\cov}{\H_{K_0}\otimes \H_{K_0}}^2 \hspace{-0.1cm}-  \rnorm{\kme{\P}}^4\Big|_{\P=\P_0}\\
        & = \norm{\covn}{\H_{K_0}\otimes \H_{K_0}}^2 \hspace{-0.1cm}-  \rnorm{\kme{\P_0}}^4
         \stackrel{(a)}{=} \norm{\covn}{\H_{K_0}\otimes \H_{K_0}}^2 \stackrel{(b)}{>}0.
      \end{align}
        (a) follows from the fact that $\kme{\P_0} = 0$ by Remark~\ref{remark:integrability}\ref{remark:item:e-p0-equals-0}, (b) is implied by the imposed assumption of \hl{$\covn\ne 0$}.

		\item[$\mathcal{A}_2$(iii).] \textbf{Condition $n\big(T\big(\hat\P_n\big) - T(\P_0) - n^{-2}\sum_{i,j=1}^n h(\P_0;X_i,X_j)\big) = o_P(1)$:} The remainder term satisfies
		\begin{align}
			\MoveEqLeft T\big(\hat\P_n\big) - T(\P_0) - n^{-2}\sum_{i,j=1}^n h(\P_0;X_i,X_j)\\ 
            &\stackrel{(a)}{=} D_{\P_0}^2\big(\hat \P_n\big) - D_{\P_0}^2(\P_0) - n^{-2}\sum_{i,j=1}^n h(\P_0;X_i,X_j)\\
            &\stackrel{(b)}{=} \frac{1}{n^2}\sum_{i,j=1}^{n}K_0(X_i,X_j) - 0 - n^{-2}\sum_{i,j=1}^n\left[K_0(X_i,X_j)\right] = 0.
		\end{align}
        (a) holds by the definition of $T$, (b) comes from combining \eqref{eq:v-stat-ksd},  $D_{\P_0}^2(\P_0) = 0$ holding by \eqref{eq:population-ksd} and 
        \begin{align}
        h(\P_0;x,y) &= K_0(x,y) - \rnorm{\kme{\P_0}}^2 =  K_0(x,y), \label{eq:h-P0-x-y}
        \end{align}
        following from \eqref{eq:h(P;x,y)} and $\kme{\P_0} \stackrel{\eqref{eq:population-ksd}}{=}0$.
		Hence, the condition is satisfied.
        
		\item[Condition $h(\P_0;x,y) = h(\P_0;y,x)$:] 
        Using that 
        $
        h(\P_0;x,y) \stackrel{\eqref{eq:h-P0-x-y}}{=} K_0(x,y)$,
        this condition is implied by the symmetry of $K_0$.
		\item[Conditions $\E_{\P^2_0} h^2(\P_0;X,Y) < \infty$, $\E_{\P_0}|h(\P_0;X,X)| < \infty$:] We have that 
		\begin{align}
			\E_{\P^2_0} h^2(\P_0;X,Y) &\stackrel{\eqref{eq:h-P0-x-y}}{=} \E_{\P^2_0} K_0^2(X,Y) \stackrel{(a)}{<}\infty,\\
            \E_{\P_0}|h(\P_{0};X,X)| &\stackrel{\eqref{eq:h-P0-x-y}}{=} \E_{\P_0}|K_0(X,X)| \stackrel{(b)}{<}\infty,
		\end{align}
        where the assumed \hl{$\E_{\P_{0}} K_0(X,X)<  \infty$} implies (a) by Section~\ref{sec:proof-limit-v-stat}(ii) and (b) by  Section~\ref{sec:proof-limit-v-stat}(i).
		\item[Condition $\E_{\P_0} h(\P_0;\cdot,X)\equiv c$:] We have 
		\begin{equation}
			\E_{\P_0} h(\P_0;x,X) \stackrel{\eqref{eq:h-P0-x-y}}{=} \E_{\P_0}K_0(x,X) \stackrel{(a)}{=} \mu_{K_0}(\P_0)(x) \stackrel{(b)}{=} 0 \quad \text{for all } x\in\X, \label{eq:degeneracy-K0}
		\end{equation}
        where (a) holds by the definition of the mean embedding, and (b) follows from $\mu_{K_0}(\P_0)=0$ implied by Remark~\ref{remark:integrability}(d). This means that the required condition is satisfied.
	\end{description}
	Now note that 
    \begin{align}
        \mu(T,\P_0) &= \E_{\P_0^{2}} h(\P;X,Y) \overset{\eqref{eq:h(P;x,y)}}{=} \E_{\P_0^2}K_0(X,Y) - \rnorm{\kme{\P_0}}^2 \\
        &\overset{(a)}{=} \E_{\P_0^2}K_0(X,Y) \overset{(b)}{=} \rip{\kme{\P_0},\kme{\P_0}} \overset{(a)}{=} 0,
    \end{align}
    where (a) follows from $\mu_{K_0}(\P_0)=0$ implied by Remark~\ref{remark:integrability}(d), and (b) uses \eqref{eq:repr-prop-of-K_0}, flips the expectations and the inner product, and applies the definition of the kernel mean embedding.
    We invoke Theorem~\ref{thm:serfling-A2} to obtain the stated claim.
\end{enumerate}

\subsection{Proof of Theorem~\ref{theorem:bootstrap-v-stat}}\label{sec:proof-limit-v-stat}
Recall that under \hl{$H_0$, $\P=\P_0$}. It suffices to choose  $h\coloneq K_0$ and to check the two conditions in \citet[Theorem 3.1]{dehling94bootstrap} (recalled in Theorem~\ref{thm:dehling-v-stat-bootstrap}).
\begin{enumerate}[label=(\roman*)]
	\item \tb{Assumption $\E_{\P_0} |K_0(X,X)| < \infty$.} Observe that $K_0(x,x) = \rnorm{\fm x}^2 \ge 0$ (with the equality following from \eqref{eq:repr-prop-of-K_0} and the fact that in a Hilbert space the norm is induced by its inner product) implies that $|K_0(x,x)| = K_0(x,x)$ ($x\in\X$); hence, by \hl{Assumption~\ref{ass:integrability}}, $\E_{\P_0}|K_0(X,X)| = \E_{\P_0} K_0(X,X) < \infty$. We note that Assumption~\ref{ass:integrability} ensures the well-definedness of $\integralop$ as discussed in Remark~\ref{remark:integrability}\ref{item:compact-op}.
	\item \tb{Assumption $\E_{\P_0^{2}} K_0^2(X_1,X_2) < \infty$.} It is known that for $x_1,x_2 \in \X$ one has that $|K_0(x_1,x_2)|^2 \le K_0(x_1,x_1)K_0(x_2,x_2)$ \citep[(4.14)]{steinwart08support}. As $X_1$ and $X_2$ are independent, (i) now implies the assumption. 
\end{enumerate}
The degeneracy of $K_0$ follows from
\begin{align}
    \Var_{X_1\sim\P_0}\big[\E_{X_2\sim\P_0}K_0(X_1,X_2)\big] \overset{\eqref{eq:degeneracy-K0}}{=} \Var_{X_1\sim\P_0}[0] = 0.
\end{align}
$h$ is symmetric as $K_0$ is so.

These observations prove the claim.

\subsection{Proof of Theorem~\ref{thm:separation-ksd}}\label{sec:proof-separation-ksd}
The following argument extends the proof of \citet[Theorem~1]{balasubramanian21minimaxgof} (stated for goodness-of-fit testing with MMD; see \citet{hagrass24stein} for their relationship) to a broader setting. Particularly, the case considered in \citet[Theorem~1]{balasubramanian21minimaxgof} has two drawbacks:
\begin{enumerate}[label=(\roman*)]
	\item only the case $\theta=1/2$ is considered, and 
	\item the uniform boundedness of the eigenfunctions $(\phi_i)_{i\ge1}$, that is,  $\sup_{i\ge 1}\norm{\phi_i}{\infty} < \infty$ is assumed.
\end{enumerate}
(ii) typically does not hold for KSD (see \citealt[Example 1]{kalinke25nystromksd} or \citealt[Remark 4(i)]{hagrass24stein}). To resolve this issue, we observe that the structure of $\mathcal P_n$ and our concentration result (Theorem~\ref{thm:v-stat-consistency2}) allow us to lift this condition. Further, we consider the case $\theta>0$, broadening (i), by adapting a recent result \citep[Lemma~A.19]{hagrass24stein} to KSD in Lemma~\ref{lemma:ksd-inequality}.

The proof proceeds as follows. We first establish the divergence of $\inf_{\P\in \mathcal P_n}nD^2_{\P_0}(\P)$ and the finiteness of $\sup_{\P\in\mathcal P_n} \psione{\rnorm{\bar K_0(\cdot, X)}}$. Then, by using the definition of $\beta$, our concentration result (Theorem~\ref{thm:v-stat-consistency2}), and the assumed structure of $\mathcal P_n$, we obtain the claim. The details are as follows.

\begin{itemize}
	\item \tb{Proof of $\inf_{\P\in\mathcal P_n}\left(\sqrt{n}D_{\P_0}(\P) -\sqrt{q_{W,1-\alpha}}\right)^2 \to \infty$.} As the quantile value $q_{W,1-\alpha}$ [defined before \eqref{eq:def-limit-h0-w}] is constant, it suffices to show that $\inf_{\P\in \mathcal P_n}nD^2_{\P_0}(\P)$ diverges, which holds as
\begin{align}
	\inf_{\P\in \mathcal P_n}nD^2_{\P_0}(\P) & \stackrel{(a)}{\ge}
	\inf_{\P\in \mathcal P_n}\frac{n\norm{u_\P}{L^2(\X,\P_0)}^{\frac{2\theta+1}{\theta}}}{\norm{T_{\P_0}^{-\theta}u_\P}{L^2(\X,\P_0)}^{\frac{1}{\theta}}}
	\stackrel{(b)}{\ge} \frac{\inf_{\P\in \mathcal P_n} n\norm{u_\P}{L^2(\X,\P_0)}^{\frac{2\theta+1}{\theta}}}{\sup_{\P\in \mathcal P_n}\norm{T_{\P_0}^{-\theta}u_\P}{L^2(\X,\P_0)}^{\frac{1}{\theta}}} \\
	&\stackrel{(c)}{\ge} \frac{n\Delta_n^{\frac{2\theta +1}{2\theta}}}{\sup_{\P\in \mathcal P_n}\norm{T_{\P_0}^{-\theta}u_\P}{L^2(\X,\P_0)}^{\frac{1}{\theta}}} \stackrel{(d)}{\to} \infty. \hspace{0.5cm}\label{eq:uniform-divergence}
\end{align}
We use Lemma~\ref{lemma:ksd-inequality} in (a) and distribute the infimum in (b). $u_\P=\tfrac{\d \P}{\d \P_0}-1$, \eqref{eq:chi-square-definition}, and \eqref{eq:alternatives} yield (c). For (d), we observe that \hl{$\sup_{\P\in \mathcal P_n}\norm{\integralop^{-\theta}u_\P}{L^2(\X,\P_0)}<\infty$} and \hl{$n^{\frac{2\theta}{2\theta+1}}\Delta_n \to \infty$ as $n\to\infty$} were assumed.

\item \tb{Proof of $\sup_{\P\in\mathcal P_n} \psione{\rnorm{\bar K_0(\cdot, X)}} < \infty$. }
First, notice that 
\begin{equation*}
	\sup_{\P\in\mathcal P_n}\psitwo{\rnorm{\fm X}}^2 \overset{(a)}{=} \sup_{\P\in\mathcal P_n}\psione{\rnorm{\fm X}^2} \overset{(b)}{=} \sup_{\P\in\mathcal P_n}\psione{K_0(X,X)} \overset{(c)}{<} \infty;
\end{equation*}
where (a) uses Lemma~\ref{lemma:orlicz-properties}(4.), (b) is by the reproducing property, and (c) follows by the definition of $\mathcal P_n$ in \eqref{eq:alternatives} and Lemma~\ref{lemma:bernstein-implies-psi1}.

Having shown that $\sup_{\P\in\mathcal P_n}\psitwo{\rnorm{\fm X}}^2 < \infty$, we thus also have
\begin{equation}
	\sup_{\P\in\mathcal P_n}\psitwo{\rnorm{\fm X}}< \infty. \label{eq:sup-bound}
\end{equation}

We now observe that
\begin{equation}
	\sup_{\P\in\mathcal P_n} \psione{\rnorm{\fmc X}} \overset{(a)}{\lesssim}  \sup_{\P\in\mathcal P_n} \psione{\rnorm{\fm X}} \overset{(b)}{\lesssim} \sup_{\P\in\mathcal P_n} \psitwo{\rnorm{\fm X}} \overset{\eqref{eq:sup-bound}}{<} \infty,
\end{equation}
with (a) as in \eqref{eq:etai-sub-exp} and (b) by Lemma~\ref{lemma:orlicz-properties}(3.). This proves the claim.

\item \tb{Proof of the statement.}
Starting from the definition of $\beta(S,\mathcal P_n)$ with $\mathcal P_n \coloneq \mathcal P_n(\Delta_n,\theta)$, we have that
\begin{align}
	\MoveEqLeft\beta\!\left(S,\mathcal P_n\right) = \sup_{\P\in\mathcal P_n}\E_\P\left[1-S_n(X_1,\ldots,X_n)\right] 
	\stackrel{(a)}{=} \sup_{\P\in\mathcal P_n}\P\!\left(nD^2_{\P_0}\big(\hat \P_n\big) \le q_{W,1-\alpha}\right)             \\
	 & \stackrel{(b)}{=} \sup_{\P\in\mathcal P_n}\P\!\left(\sqrt{n}D_{\P_0}\big(\hat \P_n\big) \le \sqrt{q_{W,1-\alpha}}\right)                                                                                                                                                                                         \\
	 & \stackrel{(c)}{\le} \sup_{\P\in\mathcal P_n}\P\!\left(\sqrt{n}D_{\P_0}(\P)-\sqrt{n}D_{\P_0}\!\left(\P-\hat \P_n\right) \le \sqrt{q_{W,1-\alpha}}\right)                                                                                                               \\
	 &\stackrel{(d)} {=} \sup_{\P\in\mathcal P_n}\P\!\left(\sqrt{n}D_{\P_0}\!\left(\P-\hat \P_n\right) \ge \sqrt{n}D_{\P_0}(\P) -\sqrt{q_{W,1-\alpha}}\right)       \\
	 &\overset{(e)}{=} \sup_{\P\in\mathcal P_n}\P\!\left(D_{\P_0}\!\left(\P-\hat \P_n\right) \ge r_n(\P)\right), \label{eq:consistency-eq1}
\end{align}
with $r_n(\P) \coloneq D_{\P_0}(\P) -\sqrt{q_{W,1-\alpha}}/\sqrt n$. The details are as follows. (a) is implied by \eqref{eq:test-definition}. In (b), we noted that $D_{\P_0}\left(\hat \P_n\right)\ge 0$ by \eqref{eq:population-ksd}, we observed from Theorem~\ref{thm:serfling-v-stat}\ref{thm:conv:null} that $W \ge 0 $ ensures that $q_{W,1-\alpha} \ge 0$ and we took the square root. (c) follows from Lemma~\ref{lemma:triangle-ineq} guaranteeing that
\begin{align}
	D_{\P_0}(\P) \le D_{\P_0}\left(\hat \P_n\right) + D_{\P_0}\!\left(\P-\hat \P_n\right).
\end{align}
In (d), the terms were rearranged. We divide by $\sqrt n$ and introduce $r_n(\P)$ in (e).

To apply Theorem~\ref{thm:v-stat-consistency2}, we next show that for $n$ large enough it holds that $\inf_{\P\in\mathcal P_n}r_n(\P) > 0$. Indeed, notice that for any $n\in\mathbb N_{>0}$ we have
\begin{align}
    \MoveEqLeft\inf_{\P\in\mathcal P_n}r_n(\P) = \inf_{\P\in\mathcal P_n}\left(D_{\P_0}(\P)-\frac{\sqrt{q_{W,1-\alpha}}}{\sqrt n}\right) > 0 \\
    &\overset{(a)}{\iff} \inf_{\P\in\mathcal P_n}\left(\sqrt n D_{\P_0}(\P)-\sqrt{q_{W,1-\alpha}}\right) > 0 \\
    &\overset{(a)}{\iff} \inf_{\P\in\mathcal P_n}\sqrt n D_{\P_0}(\P) > \sqrt{q_{W,1-\alpha}}  \\
     &\overset{(b)}{\iff} \inf_{\P\in\mathcal P_n} n D_{\P_0}^2(\P) > q_{W,1-\alpha}, \label{eq:rn-positive}
\end{align}
where we rearrange in (a) and square in (b); the latter is valid by the properties of the infimum and as $D_{\P_0}(\P)\ge 0$ by \eqref{eq:population-ksd} and as $W \ge 0 $ [by Theorem~\ref{thm:serfling-v-stat}\ref{thm:conv:null}]  ensures that also $q_{W,1-\alpha} \ge 0$. The l.h.s.\ of \eqref{eq:rn-positive} diverges for $n\to\infty$ by the first bullet of this proof, implying by the chain of equivalences that $\inf_{\P\in\mathcal P_n}r_n(\P) >0$ holds for $n$ large enough.

Applying Theorem~\ref{thm:v-stat-consistency2}---its assumption is satisfied as the $\psi_1$-norm is not larger than the $\psi_2$-norm (Lemma~\ref{lemma:orlicz-properties}(3.)), which we have shown to be finite in \eqref{eq:sup-bound}---, we obtain that for $n$ large enough 
\begin{align*}
	\eqref{eq:consistency-eq1} &\lesssim  \sup_{\P\in\mathcal P_n} \exp\left(-\frac{nr_n^2(\P)}{2K_{\P}^2}\right) 
	\overset{(a)}{=}  \sup_{\P\in\mathcal P_n} \exp\left(-\frac{\left(\sqrt{n}D_{\P_0}(\P) -\sqrt{q_{W,1-\alpha}}\right)^2}{2K_{\P}^2}\right) \\
	&\overset{(b)}{\le} \exp\left(-\frac{\inf_{\P\in\mathcal P_n}\left(\sqrt{n}D_{\P_0}(\P) -\sqrt{q_{W,1-\alpha}}\right)^2}{2\sup_{\P\in\mathcal P_n}K_{\P}^2}\right) \overset{(c)}{\to} 0
\end{align*}
as $n\to \infty$, with $K_{\P}\coloneq K(\P,K_0)\coloneq \psione{\rnorm{\bar K_0(\cdot, X)}}$ as in the applied theorem. (a) is by the definition of $r_n(\P)$ and rearranging. In (b), we use the monotonicity of the exponential function, the properties of suprema/infima, and split the infimum. For (c), we notice that $\sup_{\P\in\mathcal P_n}K_{\P}^2 < \infty$ holds by the second bullet and the divergence of $\inf_{\P\in\mathcal P_n}\left(\sqrt{n}D_{\P_0}(\P) -\sqrt{q_{W,1-\alpha}}\right)^2$ was shown in the first bullet.
\end{itemize}

\subsection{Proof of Lemma~\ref{lemma:range-space-equivalences}} \label{sec:proof-lemma-range-space-equivalences}
We first establish that (i) $\implies$ (ii) $\implies$ (iii) $\implies$ (i) and afterwards show that $\norm{\integralop^{-\theta}u_\P}{L^2(\X,\P_0)}^2 = \sum_{j\in J} \lambda_j^{-2\theta}\big\langle u_\P, \tilde \phi_j\big\rangle^2_{L^2(\X,\P_0)} $. Let $J'$ be such that $\big(\tilde\phi_j'\big)_{j\in J'}$ extends the ONS $\big(\tilde\phi_j\big)_{j\in J}$ to an ONB of $L^2(\X,\P_0)$.

\begin{itemize}
	\item (i) $\implies$ (ii): 
	By the range space assumption, $u_\P = \integralop^\theta f = \sum_{j\in J}\lambda_j^\theta \big\langle f,\tilde \phi_j\big\rangle_{L^2(\X,\P_0)}\tilde \phi_j$ for some $f \in L^2(\X,\P_0)$.  
	Observe that 
	$f= \sum_{j\in J}\big\langle f,\tilde \phi_j\big\rangle_{L^2(\X,\P_0)}\tilde \phi_j + \sum_{j\in J'}\big\langle f,\tilde \phi_j'\big\rangle_{L^2(\X,\P_0)}\tilde \phi'_j$ with $\sum_{j\in J}\big\langle f,\tilde \phi_j\big\rangle_{L^2(\X,\P_0)}^2 \overset{(\dagger)}{<} \infty$.
    This means that
    \begin{align}
    \sum_{j\in J} \lambda_j^{-2\theta}\big \langle u_\P, \tilde \phi_j\big\rangle^2_{L^2(\X,\P_0)}  & \stackrel{(a)}{=} \sum_{j\in J} \lambda_j^{-2\theta}\big\langle \integralop^\theta f, \tilde \phi_j\big\rangle^2_{L^2(\X,\P_0)} 
    \stackrel{(b)}{=} \sum_{j\in J} \lambda_j^{-2\theta} \big\langle  f, \integralop^\theta\tilde \phi_j\big\rangle^2_{L^2(\X,\P_0)}  \\
    &\stackrel{(c)}{=} \sum_{j\in J} \lambda_j^{-2\theta} \big\langle  f, \lambda_j^\theta\tilde \phi_j\big\rangle^2_{L^2(\X,\P_0)} 
    \stackrel{(d)}{=}\sum_{j\in J} \big\langle  f, \tilde \phi_j\big\rangle^2_{L^2(\X,\P_0)} \overset{(\dagger)}{<} \infty,
    \end{align}
    where we used that $u_\P = \integralop^\theta f$ in (a), (b) holds by the definition of the adjoint operator and the self-adjointness of $\integralop^\theta$ following from that of $\integralop$, the spectral decomposition $\integralop^\theta=\sum_{j\in J} \lambda_j^\theta \tilde \phi_j\otimes_{L^2(\X,\P_0)} \tilde \phi_j$ implied by that of $\integralop=\sum_{j\in J} \lambda_j \tilde \phi_j\otimes_{L^2(\X,\P_0)} \tilde \phi_j$, with  the orthogonality of the $\tilde \phi_i$-s ($i\in J$) yield (c), (d) holds by the linearity of the inner product.
	\item (ii) $\implies$ (iii): It suffices to choose $a_j = \big\langle u_\P, \tilde \phi_j\big\rangle_{L^2(\X,\P_0)}$ ($j\in J$) in \eqref{eq:S-def}.
	\item  (iii) $\implies$ (i): We show that 
    \begin{align}
        \range\!\left( T_{\P_0}^\theta \right) = S_{T_{\P_0}^{-\theta}} \overset{\eqref{eq:S-def}}{=}  \left\{ \sum_{j\in J} a_j\tilde \phi_j :  \sum_{j\in J} a_j^2\lambda_j^{-2\theta} < \infty \right\},
    \end{align}
    which implies the claim.

    $\range\!\left( T_{\P_0}^\theta \right) \subseteq S_{T_{\P_0}^{-\theta}}$: Let $f\in \range\!\left( T_{\P_0}^\theta \right)$. Then, using the fact that $\big(\tilde\phi_j\big)_{j\in J} \cup \big(\tilde\phi_j'\big)_{j\in J'}$ forms an ONB of $L^2(\X,\P_0)$, there exists $g=\sum_{j\in J}b_j\tilde \phi_j + \sum_{j\in J'}b_j'\tilde \phi_j' \in L^2(\X,\P_0)$ with $\sum_{j\in J} b_j^2 \le \sum_{j\in J} b_j^2 +\sum_{j\in J'} \big(b_j^{'}\big)^2< \infty$ such that $f= \integralop^\theta g \stackrel{(a)}{=} \integralop^\theta \left(\sum_{j\in J}b_j\tilde \phi_j\right) =\sum_{j\in J}\lambda^\theta_j b_j\tilde \phi_j$, where (a) holds as $\integralop^\theta=\sum_{j\in J} \lambda_j^\theta \tilde \phi_j \otimes_{L^2(\X,\P_0)} \tilde \phi_j$ and the fact that the elements of $\big(\tilde\phi_j'\big)_{j\in J'}$ are orthogonal to those of $\big(\tilde\phi_j\big)_{j\in J}$. Let $a_j = \lambda_j^\theta b_j$ ($j\in J)$. Then $f=\sum_{j\in J}a_j\tilde \phi_j$ and $\sum_{j\in J}a_j^2\lambda_{j}^{-2\theta} = \sum_{j\in J} b_j^2 < \infty$, that is, $f\in S_{T_{\P_0}^{-\theta}}$.

    $\range\!\left( T_{\P_0}^\theta \right) \supseteq S_{T_{\P_0}^{-\theta}}$: Let $f\in S_{\integralop^{-\theta}}$. Then $f=\sum_{j\in J}\alpha_j\tilde \phi_j$ with $\sum_{j\in J}\alpha_j^2\lambda_j^{-2\theta} < \infty$. Let $g= \sum_{j\in J}\lambda^{-\theta}_j\alpha_j\tilde \phi_j$. As $\sum_{j\in J}\lambda^{-2\theta}_j\alpha_j^2 < \infty$, $g\in L^2(\X,\P_0)$. In particular, $\integralop^{\theta} g = \sum_{j\in J} \alpha_j\tilde \phi_j = f \in \range\!\left( T_{\P_0}^\theta \right)$.
\end{itemize}
This concludes the first part of the proof.

To show that $\norm{\integralop^{-\theta}u_\P}{L^2(\X,\P_0)}^2 = \sum_{j\in J} \lambda_j^{-2\theta}\big\langle u_\P, \tilde \phi_j\big\rangle^2_{L^2(\X,\P_0)} $, we show that $\integralop^{-\theta}u_\P\in\overline{\Span}\big(\tilde \phi_j : j\in J\big)$ and Parseval's identity will then yield the result. Indeed, as $u_\P \in \range(\integralop^\theta)$, there exists $f\in L^2(\X,\P_0)$ such that $u_\P = \integralop^\theta f$. By the definition of $\integralop^{-\theta}$, we have that
\begin{equation}
	\integralop^{-\theta}u_\P = \integralop^{-\theta}\integralop^\theta f = \sum_{j\in J}\big\langle f, \tilde \phi_j \big\rangle_{L^2(\X,\P_0)}\tilde \phi_j \in \overline{\Span}\big(\tilde \phi_j : j\in J\big). 
\end{equation}

Parseval's identity yields
\begin{align}
	\norm{\integralop^{-\theta}u_\P}{L^2(\X,\P_0)}^2 = \sum_{j\in J}\ip{\integralop^{-\theta}u_\P,\tilde\phi_j}{L^2(\X,\P_0)}^2,
\end{align}
which we rewrite as
\begin{align}
\MoveEqLeft
\sum_{j\in J}\ip{\integralop^{-\theta}u_\P,\tilde\phi_j}{L^2(\X,\P_0)}^2 = \sum_{j\in J}\lambda_j^{-2\theta}\ip{u_\P,\tilde\phi_j}{L^2(\X,\P_0)}^2
\end{align}
by using the self-adjointness of $\integralop^{-\theta}$ and that $\integralop^{-\theta}\tilde \phi_j = \lambda_j^{-\theta}\tilde \phi_j$ for $j\in J$.

\subsection{Proof of Lemma~\ref{lemma:projection-perspective}}\label{sec:proof-projection-perspective}
Define the sampling operators $S_n : \H_{K_0} \to \R^n$, $f\mapsto \left(f(X_i)\right)_{i=1}^n$ and $\tilde S_m : \H_{K_0} \to \R^m$, $f\mapsto \left(f(X_{I_j})\right)_{j=1}^m$ \citep{smale07learningtheory}. They have adjoints $S_n^* : \R^n\to \H_{K_0}$, $\bm\alpha=\left(\alpha_i\right)_{i=1}^n \mapsto \sum_{i=1}^n\alpha_i\fm{X_i}$ and $\tilde S_m^* : \R^m\to \H_{K_0}$, $\bm\alpha=\left(\alpha_j\right)_{j=1}^m \mapsto \sum_{j=1}^m\alpha_i\fm{X_{I_j}}$, respectively. Further, $S_nS_n^* = \Knn$, $\tilde S_m \tilde S_m^* = \Kmm$ and $S_n\tilde S_m^* = \Knm$.

With this notation, we observe that $\projm\kme\G = \tilde S_m^*\bm \alpha$, where $\bm\alpha$ is the solution to the optimization problem
\begin{align}
    \min_{\bm \alpha = \left(\alpha_j\right)_{j=1}^m \in \R^m} \rnorm {\kme\G - \sum_{j=1}^m\alpha_j\fmij}^2, \label{eq:bootstrap-optimization}
\end{align}
that is, the orthogonal projection of $\kme \G = \frac{1}{n}S_n^*R$ onto $\H_{K_0,m}$, with $R=\left( R_i \right)_{i=1}^n\in\R^n$. To solve \eqref{eq:bootstrap-optimization}, we first rewrite the norm as
\begin{align}
    \MoveEqLeft\rnorm{\kme \G - \sum_{j=1}^m\alpha_j\fmij}^2
	= \rnorm{\frac1nS_n^*R - \tilde S_m^* \bm \alpha}^2 \\
	&= \rnorm{\frac1nS_n^*R}^2 + \rnorm{\tilde S_m^*\bm \alpha}^2 - 2\ip{\frac1nS_n^*R,\tilde S_m^*\bm \alpha}{\H_{K_0}} \\
    &= \frac{1}{n^2}\ip{R, S_nS_n^*R}{\R^n} + \ip{\bm\alpha, \tilde S_m\tilde S_m^*\bm \alpha}{\R^m} - \frac2n\ip{R, S_n\tilde S_m^*\bm \alpha}{\R^n} \\
	&= \frac{1}{n^2}R\T\Knn R + \bm \alpha\T \Kmm \bm\alpha - \frac2nR\T\Knm\bm \alpha.
\end{align}
Considering the zeros of the derivative and solving for $\bm\alpha$ yields that $\bm \alpha = \frac1n\Kmm^-\Kmn R$ is the minimum norm solution \citep[Theorem~6.3, Remark~6.5]{laub04matrix} to \eqref{eq:bootstrap-optimization}. With this choice of $\bm\alpha$, we have 
\begin{align}
    \MoveEqLeft\frac{1}{n^2}R\T\Knm\Kmm^- \Kmn R = \frac{1}{n^2}R\T \Knm \underbrace{\Kmm^- \Kmm \Kmm^-  }_{=\Kmm^-  }\Kmn R \\
	&= \bm \alpha\T \Kmm \bm \alpha
	=\norm{\tilde S_m^*\bm\alpha}{\H_{K_0}}^2 = \rnorm{\projm \kme\G}^2.
\end{align}

\subsection{Proof of Theorem~\ref{thm:consistency-nystroem-bootstrap}}\label{sec:proof-consistency-nystroem-bootstrap}

To obtain the result, we transform the problem into a statement about the projection of the mean embedding onto a subspace spanned by the feature map of the Nystr\"{o}m samples (see Lemma~\ref{lemma:projection-perspective}). We then use a decomposition similar to \citet[Theorem~2]{kalinke25nystromksd}, allowing us to take Assumption~\ref{ass:sub-gaussian} into account, with the difference that we also consider the randomness in the bootstrap. We bound one of the terms of the decomposition by \citet[Lemma~B.1]{kalinke25nystromksd} (recalled in Theorem~\ref{lemma:bound-projection}) and the other one with the Bernstein inequality for separable Hilbert spaces (recalled in Theorem~\ref{thm:bernstein-bounded}). The combination of the bound then yields the stated claim.

\begin{itemize}
	\item \tb{Projection perspective.}
As in Lemma~\ref{lemma:projection-perspective}, let $\G = \frac1n \sum_{i=1}^nR_i \delta_{X_i}$ and observe that $\kme \G = \frac1n\sum_{i=1}^n R_i \fmi$, which together with Lemma~\ref{lemma:projection-perspective} implies that \eqref{eq:ksd-bootstrap} and \eqref{eq:nystroem-bootstrap} can be expressed as 
\begin{align}
    \begin{split}B_n^2 &= B_n^2 (X_1,\ldots,X_n, R_1,\ldots,R_n) = \rnorm{\kme \G}^2  \text{ and} \\  \tilde B_n^2 &= \tilde B_n^2(X_1,\ldots,X_n,I_1,\ldots,I_m,R_1,\ldots,R_n) =  \rnorm{\projm \kme \G}^2, \end{split}\label{eq:explicit-proj-perspective}
\end{align}
respectively. 

\item \tb{Decomposition.} Let $\lambda > 0$.
We start by introducing the regularized centered covariance operator to obtain the decomposition
\begin{align*}
    \MoveEqLeft B_n - \tilde B_n \stackrel{(a)}{=} \left| B_n - \tilde B_n \right| \stackrel{(b)}{\le} \rnorm{\kme \G-\projm \kme\G} 
    \end{align*}
    \begin{align}
	& \stackrel{(c)}{=} \norm{\left( I-\projm \right)\left[\frac1n\sum_{i=1}^nR_i\!\left(\fmi -\frac1m\sum_{j=1}^m\fmij\right)\right]}{\H_{K_0}} \notag\\
    &\stackrel{(d)}{\le}  \underbrace{\opnorm{\left( I-\projm \right)\covcr^{1/2}}}_{\eqcolon t_1}\underbrace{\rnorm{\covcr^{-1/2}\left[\frac1n\sum_{i=1}^nR_i\!\left(\fmi -\frac1m\sum_{j=1}^m\fmij\right)\right]}}_{\eqcolon t_2}. \label{eq:decomposition}
\end{align}
In (a), we use the fact that a projection is norm-decreasing; hence, the difference is non-negative. The reverse triangle inequality and the established projection perspective \eqref{eq:explicit-proj-perspective} give (b). (c) follows by distributivity and by using that $\frac1m\sum_{j=1}^{m}\fmij \in \H_{K_0,m}$. For obtaining (d), observe that $I=\covcr^{1/2}\covcr^{-1/2}$ and use the definition of the operator norm.
We will obtain probabilistic bounds on $t_1$ and $t_2$ in the following.

\item \tb{Bound on term $t_1$.} Assume that $0<\lambda\le\opnorm{\covc}$. Then \hl{Lemma~\ref{lemma:bound-projection}} (leveraging also Assumptions~\ref{ass:integrability} and \ref{ass:sub-gaussian}) yields that for any $\delta\in(0,1)$ it holds that
\begin{align}
    \left( \P^n\otimes \Lambda^m  \right)\left(\opnorm{\left( I-\projm \right)\covcr^{1/2}}\lesssim \sqrt\lambda\right) \ge 1-\delta/2 \label{eq:bound-term-t1}
\end{align}
provided that $m\gtrsim \max\left\{\trace\left( C_{\P,\bar K_0} \right)\lambda^{-1},1\right\}\log(8/\delta)$. We note that \eqref{eq:bound-term-t1} is independent of $\rho^n$; hence, the same holds with probability $\P^n\otimes \Lambda^m \otimes \rho^n$.

\item \tb{Bound on term $t_2$.} 
Fix $\left( I_j \right)_{j=1}^m$ and $\left( X_i \right)_{i\in [n]}$; we will write $\left( i_j \right)_{j=1}^m$ and $\left(x_i\right)_{i=1}^n$ to refer to the fixed quantities and note that the only randomness is in $\left( R_i \right)_{i=1}^n$. Let 
\begin{align}
Y_i = R_i\covcr^{-1/2}\left( \fmxi -\frac1m\sum_{j=1}^m\fm{x_{i_j}}\right), \quad i\in[n].
\end{align}
Then $Y_1,\ldots,Y_n$ are mutually independent, one has that $\E_{\rho}Y_i = 0$, and that $t_2 = \rnorm{\frac1n\sum_{i=1}^{n}Y_i}$ measures the concentration of centered independent random variables. We will show that the $\rnorm{Y_i}$-s are bounded, which will imply their concentration by Bernstein's inequality holding for separable Hilbert spaces (recalled in Theorem~\ref{thm:bernstein-bounded}). Clearly,
\begin{align}
	\max_{i\in[n]}\rnorm{Y_i} &\overset{(a)}{=} \max_{i\in[n]}\rnorm{\covcr^{-1/2}\left( \fmxi -\frac1m\sum_{j=1}^m\fm{x_{i_j}}\right)} \\
	&\overset{(b)}{\le} \max_{i\in[n]}\rnorm{\covcr^{-1/2}\!\fmc{x_i}}+\rnorm{\covcr^{-1/2}\!\left(\frac1m\sum_{j=1}^m\fmc{x_{i_j}}\right)} \\
	&\overset{(c)}{\le} 2  \max_{i\in[n]}\rnorm{\covcr^{-1/2}\!\fmc{x_i}}  \eqcolon b\left(\lambda,(x_i)_{i=1}^n\right) \eqcolon b,
\end{align}
where $R_i\in\{-1,1\}$ and the positive homogeneity of the norm implies (a). For (b), we add $\pm \E_\P \fm{X}$ and use the triangle inequality. For obtaining (c), we bound the second term \begin{align}\MoveEqLeft\rnorm{\covcr^{-1/2}\!\left(\frac1m\sum_{j=1}^m\fmc{x_{i_j}}\right)} \le \frac1m\sum_{j=1}^m \rnorm{\covcr^{-1/2}\fmc{x_{i_j}}}\\
	&\le \max_{j\in[m]}\rnorm{\covcr^{-1/2}\fmc{x_{i_j}}} \le \max_{i\in[n]}\rnorm{\covcr^{-1/2}\fmc{x_{i}}},
\end{align}
by the triangle inequality and by using that the $x_{i_j}$-s are a subset of the $x_i$-s. The application of Theorem~\ref{thm:bernstein-bounded} for the separable Hilbert space $\H_{K_0}$ [holding by Remark~\ref{remark:integrability}\ref{item:HK0-sep}] yields that, conditioned on $\left( I_j \right)_{j=1}^m$, $\left( X_i \right)_{i=1}^n$, with $\rho^n$-probability at least $1-\delta/4$, one has that
\begin{align}
	\rnorm{\frac1n\sum_ {i=1}^nY_i} \le b\frac{\sqrt{2\log (8/\delta})}{\sqrt n}. \label{eq:t2-bounded-bernstein}
\end{align}
The next step is to lift the conditioning, for which we note that \hl{$b$ is the maximum of sub-Gaussian random variables w.r.t.\ $\P$}. Indeed, \hl{Lemma~\ref{lemma:sub-gauss-norm}} immediately yields that 
\begin{align}
	\psitwo{\rnorm{\covcr^{-1/2}\fmc{X_{i}}}}^2 \lesssim \trace\!\left( \covcr^{-1}\covc \right) = \mathcal N_{\bar K_0}(\lambda),
\end{align}
and, using \hl{Lemma~\ref{lemma:max-of-sub-gauss}}, we have that with $\P^n$-probability of at least $1-\delta/4$,
\begin{align}
	b\left( \lambda, \left( X_i \right)_{i=1}^n \right) \lesssim \sqrt{\mathcal N_{\bar K_0}(\lambda)\log\left( 8n/\delta \right)}. \label{eq:t2-max-bound}
\end{align}
It remains to combine \eqref{eq:t2-bounded-bernstein} and \eqref{eq:t2-max-bound}, taking all sources of randomness into account.

Let us define (the subscript indicates the conditioning)
\begin{align}
	A_{(i_j)_{j=1}^m} &= \left\{ \big( \left( R_i \right)_{i=1}^n , \left( X_i  \right)_{i=1}^n \big) :  t_2 \lesssim  \frac{\sqrt{\mathcal N_{\bar K_0}(\lambda)\log\left( 8n/\delta \right)2\log (8/\delta)}}{\sqrt n}\right\}, \\
	B_{(i_j)_{j=1}^m} &= \left\{ \left( X_i \right)_{i=1}^n : b\left( \lambda, \left( X_i \right)_{i=1}^n \right) \lesssim \sqrt{\mathcal N_{\bar K_0}(\lambda)\log\left( 8n/\delta \right)}\right\}, \text{ and} \\
	C_{(i_j)_{j=1}^m} &= \left\{ \big( \left( R_i \right)_{i=1}^n , \left( X_i  \right)_{i=1}^n \big) : t_2 \le b\left( \lambda, \left( X_i \right)_{i=1}^n \right)\frac{\sqrt{2\log (8/\delta)}}{\sqrt n}, \left( X_i \right)_{i=1}^n\in B \right\}.
\end{align}
Notice that $C_{(i_j)_{j=1}^m} \subseteq  A_{(i_j)_{j=1}^m}$. We compute the conditional probability
\begin{align}
    \left(\P^n\otimes \rho^n\right)\left(A_{\left( i_j \right)_{j=1}^m}\right) &= \E_{\P^n}\!\left[ \rho^n\!\left(A_{\left( i_j \right)_{j=1}^m}\;| \left(X_i\right)_{i=1}^n\right)\right] \\
	&= \int_{\X^n}\rho^n\!\left( A_{\left( i_j \right)_{j=1}^m}\;| \left(x_i\right)_{i=1}^n\right)\d\P^n\!\left(x_1,\ldots,x_n\right) \\
    &\ge \int_{B_{\left( i_j \right)_{j=1}^m}} \rho^n\!\left(A_{\left( i_j \right)_{j=1}^m}\;| \left(x_i\right)_{i=1}^n\right)\d\P^n\!\left(x_1,\ldots,x_n\right) \\
	&\ge \int_{B_{\left( i_j \right)_{j=1}^m}} \rho^n\!\left(C_{\left( i_j \right)_{j=1}^m}\;| \left(x_i\right)_{i=1}^n\right)\d\P^n\!\left(x_1,\ldots,x_n\right) \\
    &\stackrel{\eqref{eq:t2-bounded-bernstein}}{\ge} \left(1-\delta/4\right)\P^n\!\left(B_{\left( i_j \right)_{j=1}^m}\right) \stackrel{\eqref{eq:t2-max-bound}}{\ge} \left(1-\delta/4\right)^2 \ge 1-\delta/2. 
\end{align}
We lift the conditioning by integrating w.r.t.\ $\Lambda^m$ and obtain
\begin{align}
	\left(\P^n\otimes \Lambda^m\otimes \rho^n\right)\!\left(t_2 \lesssim  \frac{\sqrt{\mathcal N_{\bar K_0}(\lambda)\log\left( 8n/\delta \right)2\log (8/\delta)}}{\sqrt n}\right) \ge 1-\delta/2. \label{eq:bound-term-t2} 
\end{align}
\item \tb{Combination.} We now union bound \eqref{eq:bound-term-t1} and \eqref{eq:bound-term-t2}, which gives with $\left( \P^n\otimes \Lambda^m \otimes \rho^n  \right)$-probability at least $1-\delta$ that
\begin{align}
	  B_n - \tilde B_n \lesssim  \frac{\sqrt{\lambda\mathcal N_{\bar K_0}(\lambda)\log\left( 8n/\delta \right)2\log (8/\delta)}}{\sqrt n}. \label{eq:c3-origin}
\end{align}
The conditions that we imposed along the way were (i) $0<\lambda\le\opnorm{\covc}$ and (ii)  $m\gtrsim \max\left\{\trace\left( C_{\P,\bar K_0} \right)\lambda^{-1},1\right\}\log(8/\delta)$. To conclude, we set $\lambda \asymp \frac{\log (m)}{m}$, which satisfies (i) for \hl{$m\gtrsim \opnorm{\covc}^{-1}\log(m)$} leveraging that $\covc\ne 0$ (by assumption), and \hl{(ii) for} 
\begin{align}
	m \gtrsim \max\left\{\left(\frac{8}{\delta}\right)^{c_2\trace\left(\covc\right)},\log\left(\frac{8}{\delta}\right)\right\},
\end{align}
where $c_2>0$ is an absolute constant; taking $m$ to be the maximum satisfying (i) and (ii) yields the stated requirement. To simplify the argument of the effective dimension, we use that $\lambda \mapsto \mathcal{N}_{\bar K_0}(\lambda)$ is decreasing\footnote{\label{fn:effective-dimension-explicit}For $0 < \lambda \le \lambda'$, one has that $\mathcal{N}_{\bar K_0}(\lambda) = \sum_i\tfrac{\lambda_i}{\lambda_i+\lambda} \ge \sum_i\tfrac{\lambda_i}{\lambda_i+\lambda'} = \mathcal{N}_{\bar K_0}(\lambda')$ with $(\lambda_i)_i$ denoting the eigenvalues of $\covc$.} in $\lambda >0$ to get the bound (holding for \hl{$m\ge 3$})
\begin{align}
    \mathcal{N}_{\bar K_0}\left(\frac{c_1\log (m)}{m}\right) \le \mathcal{N}_{\bar K_0}\left(\frac{c_1}{m}\right). \label{eq:effective-dim-relaxation}
\end{align}
Rearranging and noting that $\log(8n/\delta)\log(8/\delta)\le \log^2(8n/\delta)$ holds for $n\ge 1$ yields the stated claim.
\end{itemize}

\subsection{Proof of Theorem~\ref{thm:limit-nystroem-ksd}}\label{sec:proof-limit-nystroem-ksd}
We will show that along the stated sequences it holds asymptotically almost surely that
\begin{align}
    nB_n^2 - n\tilde B_n^2 = o_P(1) \label{eq:asymp-equiv}
\end{align} 
as $m,n\to \infty$ with the given conditions. The result is then implied by \hl{Theorem~\ref{theorem:bootstrap-v-stat}} and Slutsky's lemma, using that if \eqref{eq:asymp-equiv} holds,  $n\tilde B_n^2 = \left( n\tilde B_n^2 - nB_n^2 \right) + nB_n^2$ converges weakly to the limit of $nB_n^2$ \citep[p.~153]{vaart98asymptotic} for $\P_0^\infty$-almost every sequence $(X_i)_{i=1}^\infty$ and $\Lambda^\infty$-almost every sequence $(I_j)_{j=1}^\infty$.

Notice that by the definition of convergence in probability, \eqref{eq:asymp-equiv} is equivalent to the statement that for any $\epsilon > 0$ and any $\delta>0$ it holds that
\begin{align}
    \left(\P^n_0\otimes \Lambda^m\right)\left\{ \rho^n\left(\left|nB_n^2 - n\tilde B_n^2 \right| > \epsilon \;\big| \left(X_i\right)_{i=1}^n, \left(I_j\right)_{j=1}^m\right) < \delta\right\} \to 1,
\end{align}
which is also equivalent to
\begin{align}
    \left(\P^n_0\otimes \Lambda^m\right)\left\{ \rho^n\left(\left|nB_n^2 - n\tilde B_n^2 \right| > \epsilon \;\big| \left(X_i\right)_{i=1}^n, \left(I_j\right)_{j=1}^m\right) \ge \delta\right\} \to 0. \label{eq:pre-markov-reduction}
\end{align}
for $m,n \to \infty$ (with the stated conditions). This is what we prove in the following.

Let $\epsilon,\delta >0$ be arbitrary. By applying Markov's inequality to the l.h.s.\ of \eqref{eq:pre-markov-reduction}, we get that
\begin{align}
     \MoveEqLeft\left(\P^n_0\otimes \Lambda^m\right)\left\{ \rho^n\left(\left|nB_n^2 - n\tilde B_n^2 \right| > \epsilon \;\big| \left(X_i\right)_{i=1}^n, \left(I_j\right)_{j=1}^m\right) \ge \delta\right\} \\
     &\le \frac{\E_{\P_0^n\otimes \Lambda^m}\left[
     \rho^n\left(\left|nB_n^2 - n\tilde B_n^2 \right| > \epsilon \;\big| \left(X_i\right)_{i=1}^n, \left(I_j\right)_{j=1}^m\right)\right]}{\delta} \\
     &\overset{(a)}{=} \frac{\left(\P^n_0\otimes \Lambda^m\otimes \rho^n\right)\left(\left|nB_n^2 - n\tilde B_n^2 \right| > \epsilon\right)}{\delta}, \label{eq:markov-reduction}
\end{align}
by using the law of total expectation in (a); hence, it suffices to show that the numerator in \eqref{eq:markov-reduction} tends to zero. Observe that
\begin{align}
	\MoveEqLeft \left( \P_0^n\otimes \Lambda^m\otimes \rho^n \right)\left( \left|nB_n^2 - n\tilde B_n^2 \right| > \epsilon \right) \\
    &\stackrel{(a)}{=}
	\left(\P_0^n\otimes  \Lambda^m\otimes \rho^n \right)\left( n\rnorm{\kme\G - \projm \kme \G}^2 > \epsilon \right) \\
	&\stackrel{(b)}{=}
	\left( \P_0^n\otimes \Lambda^m\otimes \rho^n \right)\left( \rnorm{\kme\G - \projm \kme \G} > \sqrt{\frac \epsilon n}  \right) \\
    &\overset{(c)}{\le} 8n\exp\left\{-\frac{\sqrt{m\epsilon}}{c_4\sqrt{\log (m)\mathcal N_{\bar K_0}\!\left(\frac{c_1}{m}\right)}}\right\}.
\end{align}
where the Pythagorean theorem yields (a). Dividing by $n$ and taking the square root gives (b) as all terms are non-negative. For obtaining (c), we recall that the bound of \hl{Theorem~\ref{thm:consistency-nystroem-bootstrap}} (holding for $m$ large enough) is also a bound on $\rnorm{\kme \G - \projm\kme\G}$ by the decomposition used in its proof (Section~\ref{sec:proof-consistency-nystroem-bootstrap}). Hence, we obtain the claimed bound on the r.h.s.\ of (c) from Theorem~\ref{thm:consistency-nystroem-bootstrap} with $\P=\P_0$ and by solving for $\delta$ the equality
\begin{align}
c_4\frac{\sqrt{\log(m)}\log(8n/\delta)}{\sqrt{nm}}\sqrt{\mathcal N_{\bar K_0}\!\left(\frac{c_1}{m}\right)} & = \sqrt{\frac{\epsilon}{n}},
\end{align}
where $c_4>0$ is the constant implicit in the inequality \eqref{eq:c3-origin}.

Therefore, to guarantee that \eqref{eq:markov-reduction} tends to zero, it suffices to ensure that 
\begin{align}
\frac{\sqrt{m}}{\sqrt{\log (m)\mathcal N_{\bar K_0}\!\left(\frac{c_1}{m}\right)}} = \omega(\log (n)) \text{ as }n,m\to \infty.
\end{align}
In the following, we derive the assumptions on $n,m$ necessary to guarantee this behavior depending on  $\mathcal N_{\bar K_0}(\lambda)$.

\begin{enumerate}[label=(\roman*)]
	\item \tb{Polynomial decay: \hl{$\mathcal N_{\bar K_0}(\lambda)\lesssim \lambda^{-\gamma}$}.} In this case, we have that 
	$\mathcal N_{\bar K_0}\!\left(c_1/m \right)  \lesssim m^\gamma$, which, with $\log (m) \le \log (n)$, leads to the lower bounds
\begin{align}
    \frac{\sqrt{m}}{\sqrt{\log (m)\mathcal N_{\bar K_0}\!\left(\frac{c_1}{m}\right)}} \gtrsim \frac{m^{\frac{1-\gamma}{2}}}{\sqrt{\log (m)}} \gtrsim \frac{m^{\frac{1-\gamma}{2}}}{\sqrt{\log (n)}} = \omega(\log (n))
\end{align}
by using our choice of \hl{$m = \omega\!\left(\log^{3/(1-\gamma)}(n) \right)$} in the last step. 

\item \tb{Exponential decay: \hl{$\mathcal N_{\bar K_0}(\lambda) \lesssim \log\!\left( 1+\tilde c/\lambda \right)$}.} With this assumption, we have that $\mathcal N_{\bar K_0}\!\left(c_1/m \right) \lesssim \log\left( 1+c_3m \right)$ for some $c_3 $ $(=\tilde c/c_1)>0$. Choosing $m$ large enough\footnote{It is sufficient to guarantee that $1 \le c_3m$ and $2c_3 \le m$ by the monotonicity of the logarithm function.} allows to obtain the bound
\begin{align}
    \log( 1+c_3m ) \le \log(2c_3m) = \log(2c_3) + \log(m) \le 2\log(m). \label{eq:log1+cm-bound}
\end{align}
Hence, $\mathcal N_{\bar K_0}\!\left(c_1/m \right) \lesssim \log(m)$ and by using that $\log (m) \le \log (n)$, we get
\begin{align}
	\frac{\sqrt{m}}{\sqrt{\log (m)\mathcal N_{\bar K_0}\!\left(\frac{c_1}{m}\right)}} \gtrsim \frac{\sqrt{m}}{\log (m)} \gtrsim \frac{\sqrt m}{\log (n)} = \omega(\log (n)),
\end{align}
by taking \hl{$m= \omega\!\left(\log^4(n)\right)$}.
\end{enumerate}

\subsection{Proof of Theorem~\ref{thm:separation-n-ksd}} \label{sec:proof-separation-n-ksd}

The proof is similar to that of Theorem~\ref{thm:separation-ksd} in Section~\ref{sec:proof-separation-ksd} but we take \hl{Theorem~\ref{thm:concentration-aistats}} into account. Indeed, we first establish that the conditions of Theorem~\ref{thm:concentration-aistats} on $n$ and $m$ in \eqref{eq:condition-n-appendix} and \eqref{eq:condition-m-appendix} can be satisfied uniformly, that is, 
\begin{align}
     n_{0,1}(m) &\coloneq \sup_{\P\in\mathcal P_n} \max\left\{ \log(6/\delta),\frac{m^2\opnorm{\covc}}{\log (m)}\right\} < \infty \quad \text{and} \\
     m_0 &\coloneq \sup_{\P\in\mathcal P_n} \max\left\{\log(m)\opnorm{\covc}^{-1},\left(\tfrac{12}{\delta}\right)^{\tilde c_2\trace\left(\covc\right)},\log(12/\delta)\right\} < \infty,
\end{align} 
which is implied by showing that
\begin{align}
    \sup_{\P\in\mathcal P_n} \opnorm{\covc} &< \infty, &
    \sup_{\P\in\mathcal P_n} \opnorm{\covc}^{-1} &< \infty, & &\text{and} &
     \sup_{\P\in\mathcal P_n} \trace\left(\covc\right) &< \infty.
\end{align}
We verify these conditions one by one below.

\begin{itemize}

    \item \tb{Condition $\sup_{\P\in\mathcal P_n} \trace\left(\covc\right) < \infty$.}
    Notice that for any $\P \in\mathcal P_n$. it holds that
    \begin{align}
         \MoveEqLeft\trace\!\left(\covc\right)
         \overset{\eqref{eq:def-cov-operator}}{=}  \trace\!\left(\cov - \kme\P\otimes \kme \P\right) 
         \overset{(a)}{=}
         \trace\!\left(\cov\right) - \trace\left(\kme\P\otimes \kme \P\right) \\
         &\overset{(b)}{\le} \trace\!\left(\cov\right)
         \overset{\eqref{eq:def-cov-operator}}{=}\trace\!\left(\E_{\P}\!\left[\fm X \otimes \fm X \right]\right)
         \overset{(c)}{=} \E_{\P}\!\left[\trace\!\left(\fm X \otimes \fm X \right)\right] \\
         &\overset{(d)}{=} \E_\P\!\left[\rnorm{\fm X}^2\right] \overset{(e)}{=} \E_\P K_0(X,X) 
        \overset{(f)}{\le}\left(\E_\P K_0 ^2(X,X)\right)^{1/2} \overset{\eqref{eq:alternatives}}{\le} \sqrt{2c}\kappa,
    \end{align}
    where we use the linearity of the trace in (a). (b) holds by the positivity of $\kme\P\otimes\kme\P$. In (c), we flip the expectation and the trace; (d) holds by  Lemma~\ref{lemma:norm-equiv}. The fact that in a Hilbert space the norm is induced by its inner product, followed by the reproducing property of $K_0$ give (e). We apply Hölder's inequality in (f).
    Hence, we have that $\sup_{\P\in\mathcal P_n}\trace\left(\covc\right) \le \sqrt{2c}\kappa <  \infty$.

    \item \tb{Condition $ \sup_{\P\in\mathcal P_n} \opnorm{\covc} < \infty$.} 
    For any $\P\in\mathcal P_n$, one has $\opnorm{\covc} \le \left\| \covc\right\|_1 \stackrel{(a)}{=} \trace\!\left(\covc\right)$, where (a) holds by the positivity of $\covc$. Taking the supremum over $\P\in\mathcal P_n$ implies the result by the previous bullet.
    \item \tb{Condition $\sup_{\P\in\mathcal P_n} \opnorm{\covc}^{-1} < \infty$.} We have that
    \begin{align}
        \sup_{\P\in\mathcal P_n}\frac{1}{\opnorm{\covc}} = \frac{1}{\inf_{\P \in \mathcal P_n}\opnorm{\covc}} < \infty
    \end{align}
    as  \hl{$\inf_{\P\in\mathcal P_n}\opnorm{\covc} > 0$} was assumed. 
    
\end{itemize}

Having established these conditions, we continue with the proof of the statement. Starting from the definition of the type II error, we have that
\begin{align}
	\MoveEqLeft\beta\!\left(\tilde S_n,\mathcal P_n(\Delta_n,\theta)\times \{\Lambda^m\}\right) = \sup_{\P\in\mathcal P_n}\E_{\P^n\otimes\Lambda^m}\left[1-\tilde S_n(X_1,\ldots,X_n)\right] \\
	& \hspace{-0.4cm}\stackrel{(a)}{=} \sup_{\P\in\mathcal P_n}(\P^n\otimes\Lambda^m)\!\left(n\tilde D^2_{\P_0}\big(\hat \P_n\big) \le q_{W,1-\alpha}\right)           \\
	 & \hspace{-0.4cm} \stackrel{(b)}{=} \sup_{\P\in\mathcal P_n}(\P^n\otimes\Lambda^m)\!\left(\sqrt{n}\rnorm{\projm\mu_{K_0}\big(\hat \P_n\big)} \le \sqrt{q_{W,1-\alpha}}\right)                                                                                                                                                                                         \\
	 & \hspace{-0.4cm}\stackrel{(c)}{\le} \sup_{\P\in\mathcal P_n}(\P^n\otimes\Lambda^m)\!\left(\sqrt{n}D_{\P_0}(\P)-\sqrt{n}\rnorm{\kme{\P} - \projm\mu_{K_0}\big(\hat \P_n\big)} \le \sqrt{q_{W,1-\alpha}}\right)                                                                                                               \\
	 &\hspace{-0.4cm}\stackrel{(d)} {=} \sup_{\P\in\mathcal P_n}(\P^n\otimes\Lambda^m)\!\left(\sqrt{n}\rnorm{\kme{\P} - \projm\mu_{K_0}\big(\hat \P_n\big)} \ge \sqrt{n}D_{\P_0}(\P) -\sqrt{q_{W,1-\alpha}}\right).\label{eq:main-proof-eq1}       \quad     
\end{align}
The details are as follows. (a) is implied by \eqref{eq:test-definition}. In (b), we observed from \eqref{eq:def-limit-h0-w} that $W \ge 0 $ ensures that $q_{W,1-\alpha} \ge 0$, took the square root, and used the definition in \eqref{eq:nystroem-ksd-estimator}. The triangle inequality gives
\begin{align}
	D_{\P_0}(\P) = \rnorm{\kme{\P}} \le \rnorm{\kme{\P} - \projm\mu_{K_0}\big(\hat \P_n\big)} + \rnorm{\projm\mu_{K_0}\big(\hat \P_n\big)},
\end{align}
implying (c).
In (d), the terms were rearranged.

Let $r_n(\P) = D_{\P_0}(\P) -\sqrt{q_{W,1-\alpha}}/\sqrt n$ and $n_{0,2} \in \mathbb N_{>0}$ such that for all $n\ge n_{0,2}$ one has that $\inf_{\P\in\mathcal P_n}r_n(\P) > 0$ [possible by \eqref{eq:rn-positive}]. We continue as
\begin{align*}
	\eqref{eq:main-proof-eq1} &=  \sup_{\P\in\mathcal P_n}(\P^n\otimes\Lambda^m)\!\left(\rnorm{\kme{\P} - \projm\mu_{K_0}\big(\hat \P_n\big)} \ge r_n\right) \\
    &\overset{(a)}{\le}\sup_{\P\in\mathcal P_n}(\P^n\otimes\Lambda^m)\!\left(\rnorm{\kme{\P} - \projm\mu_{K_0}\big(\hat \P_n\big)} > \frac{r_n}{2}\right) \\
    &\overset{(b)}{\le} \sup_{\P\in\mathcal P_n}12n\exp\!\left( -\frac{\tilde c_1mr_n}{2\sqrt{\log(m)\mathcal N_{\bar K_0}\!\left( \frac 1m \right)}} \right) \\
	&\overset{(c)}{=} \sup_{\P\in\mathcal P_n} 12n\exp\!\left( -\frac{\tilde c_1 m  \left(\sqrt n D_{\P_0}(\P) -\sqrt{q_{W,1-\alpha}}\right)}{2\sqrt {n\log(m)}\sqrt{\mathcal N_{\bar K_0}\!\left( \frac 1m \right)}} \right) \\
	&\overset{(d)}{\le} 12n\exp\!\left( -\frac{\tilde c_1m  \inf_{\P\in\mathcal P_n} \left(\sqrt n D_{\P_0}(\P) -\sqrt{q_{W,1-\alpha}}\right)}{2\sqrt {n\log(m)}\sqrt{\sup_{\P\in\mathcal P_n}  \mathcal N_{\bar K_0}\!\left( \frac 1m \right)}} \right),
\end{align*}
where in (a) we scale the r.h.s.\ by $1/2$ to obtain a strict inequality. We apply \hl{Theorem~\ref{thm:concentration-aistats}} with $m\ge m_0$ and $n_0 \ge \max\{n_{0,1}(m_0),n_{0,2}\}$ in (b). (c) is by the definition of $r_n$, and (d) is by the monotonicity of the exponential function and the square root, and the properties of suprema/infima.
As \hl{$\inf_{\P\in\mathcal P_n} \left(\sqrt n D_{\P_0}(\P) -\sqrt{q_{W,1-\alpha}}\right) \to \infty$ as $n\to \infty$} (see bullet 1 of Section~\ref{sec:proof-separation-ksd}), it suffices to show that 
\begin{align}
\tfrac{m}{\sqrt {n\log(m)} \sqrt{\sup_{\P\in\mathcal P_n}  \mathcal N_{\bar K_0}\!\left( \frac 1m \right)}} = \omega(\log (n))
\end{align}
for our choice of $m$, as together with the previous chain of inequalities, this will imply that $\beta\!\left(\tilde S_n,\mathcal P_n(\Delta_n,\theta)\times \{\Lambda^m\}\right) \to 0$ for $n\to\infty$.

\begin{enumerate}[label=(\roman*)]
	\item \tb{Polynomial decay.} By the imposed assumption, \hl{$\sup_{\P\in\mathcal P_n}\mathcal N_{\bar K_0}(1/m) \lesssim m^{\gamma}$} for some $\gamma\in(0,1]$. Using this inequality and that $\log (m) \le \log (n)$,  
	\begin{equation*}
		\frac{m}{\sqrt {n\log(m)} \sqrt{\sup_{\P\in\mathcal P_n}  \mathcal N_{\bar K_0}\!\left( \frac 1m \right)}} \gtrsim \frac{m^{\frac{2-\gamma}{2}}}{\sqrt {n\log(m)}}\gtrsim \frac{m^{\frac{2-\gamma}{2}}}{\sqrt {n\log(n)}} = \omega(\log n),
	\end{equation*}
	by our choice of \hl{$m = \omega\!\left(n^{\frac{1}{2-\gamma}}\log^\frac{3}{2-\gamma} (n) \right)$}.
	\item \tb{Exponential decay.} By imposing the exponential decay assumption, it holds that \hl{$\sup_{\P\in \mathcal P_n}\mathcal N_{\bar K_0}(1/m) \lesssim \log(1+\tilde c m)$} for some $\tilde c>0$. Noticing, as in \eqref{eq:log1+cm-bound}, that also $\log(1+\tilde cm) \lesssim \log(m)$ and that $\log (m) \le \log (n)$, we obtain
	\begin{equation*}
		\frac{m}{\sqrt {n\log(m)} \sqrt{\sup_{\P\in\mathcal P_n}  \mathcal N_{\bar K_0}\!\left( \frac 1m \right)}} \gtrsim \frac{m}{\sqrt{n\log (n)}\sqrt{\log(n)}} \asymp \frac{m}{\sqrt n \log(n)} = \omega(\log (n)) 
	\end{equation*}
	by using our choice of \hl{$m= \omega\left(\sqrt n \log^2 (n)\right)$}.
\end{enumerate}

\subsection{Proof of Corollary~\ref{corr:n-ksd/n-bootstrap}} \label{sec:proof-corollary}
Define the probability measures in $\mathcal P_n$ satisfying the polynomial and exponential decay assumption, respectively, as 
\begin{align}
    \mathcal P_n^{(1)} &\coloneq \mathcal P^{(1)}_n(\Delta_n,\theta) \coloneq \left\{ \P \in \mathcal P_n(\Delta_n,\theta) : \mathcal N_{\bar K_0}(\lambda) \le c_1 \lambda^{-\gamma} \right\} \subseteq \mathcal P_n(\Lambda_n,\theta), \\
    \mathcal P_n^{(2)} &\coloneq \mathcal P^{(2)}_n(\Delta_n,\theta) \coloneq \left\{ \P \in \mathcal P_n(\Delta_n,\theta) : \mathcal N_{\bar K_0}(\lambda) \le  c_2 \log(1+\tilde c/\lambda) \right\} \subseteq \mathcal P_n(\Lambda_n,\theta)
\end{align}
for some constants $\tilde c, c_1, c_2> 0$ and $\gamma \in (0,1]$. Let also
\begin{align}
    d_{n,m} \coloneq d_{n,m}(x_1,\ldots,x_n,i_1,\ldots,i_m,R_1,\ldots,R_n) \coloneq |\mqn - \mq| + \mq 
\end{align}
and $I = [0,2q_{W,{1-\alpha}}]$. For any $\P\in\mathcal P_n^{(i)}$ ($i\in[2]$), we get the upper bound
\begin{align}
\MoveEqLeft\mjoint\left(n\tilde D_{\P_0}^2\left(\hat \P_n\right)\le \mqn\right) 
\\
&\overset{(a)}{=} \mjoint\left(n\tilde D_{\P_0}^2\left(\hat \P_n\right)\le \mqn-\mq+\mq\right)\\
&\stackrel{(b)}{\le}  \mjoint\left(n\tilde D_{\P_0}^2\left(\hat \P_n\right)\le \left|\mqn-\mq+\mq\right|\right)
\\
&\overset{(c)}{\le} \mjoint\Big(n\tilde D_{\P_0}^2\left(\hat \P_n\right)\le \underbrace{|\mqn-\mq| + \underbrace{|\mq|}_{=\mq}}_{=d_{n,m}}\Big) \\
    &\overset{(d)}{=} \mjoint\left(n\tilde D_{\P_0}^2\left(\hat \P_n\right)\le d_{n,m} \mid d_{n,m} \in I\right)\rho^n(d_{n,m} \in I) \\
    &\quad+  \mjoint\left(n\tilde D_{\P_0}^2\left(\hat \P_n\right)\le d_{n,m} \mid d_{n,m} \notin I\right)\rho^n(d_{n,m} \notin I) \\
    &\overset{(e)}{\le} \mjoint\left(n\tilde D_{\P_0}^2\left(\hat \P_n\right)\le d_{n,m} \mid d_{n,m} \in I\right) + \rho^n(d_{n,m} \notin I) \\
    &\overset{(f)}{\le} \left(\P^n\otimes\Lambda^m\right)\left(n\tilde D_{\P_0}^2\left(\hat \P_n\right)\le 2\mq\right) + \rho^n(d_{n,m} \notin I). \label{eq:upper-bound-corollary}
\end{align}
In (a), we consider $\pm \mq$. (b) holds by the monotonicity of probability measures and $z\le |z|$ ($z\in \R$), (c) is by the triangle inequality, by noticing that $W \ge 0 $ ensures that $q_{W,1-\alpha} \ge 0$ (see Theorem~\ref{thm:serfling-v-stat}\ref{thm:conv:null}), and by the definition of $d_{n,m}$. (d) is by the law of total probability and by using the fact that two factors have no randomness in $\P^n$ and $\Lambda^m$. In (e), we use the fact that probabilities are bounded by one, and in (f), by considering the worst case of $d_{n,m} \in I$.

We now prove the claim. Indeed, notice that by the definition of the type II error
\begin{align}
	\MoveEqLeft\beta(\tilde S_n',\mathcal P_n^{(i)} \times \{\Lambda^m\} \times \{\rho^n\}) = \sup_{\P\in\mathcal P_n^{(i)}}\left( \P^n\otimes\Lambda^m\otimes \rho^n \right) \!\left( n\tilde D_{\P_0}^2\!\left( \hat \P_n \right) \le \tilde q_{W,1-\alpha,n} \right) \\
    &\overset{\eqref{eq:upper-bound-corollary}}{\le} \sup_{\P\in\mathcal P_n^{(i)}}\left\{\left(\P^n\otimes\Lambda^m\right)\left(n\tilde D_{\P_0}^2\left(\hat \P_n\right)\le 2\mq\right) + \rho^n(d_{n,m} \notin I)\right\} \\
    &\overset{(a)}{=} \sup_{\P\in\mathcal P_n^{(i)}}\left(\P^n\otimes\Lambda^m\right)\left(n\tilde D_{\P_0}^2\left(\hat \P_n\right)\le 2\mq\right) + \rho^n(d_{n,m} \notin I),
\end{align}
where (a) holds as the second term does not depend on $\P$.
Replacing $q_{W,1-\alpha}$ in the proof of \hl{Theorem~\ref{thm:separation-n-ksd}} (Section~\ref{sec:proof-separation-n-ksd}) with $2q_{W,1-\alpha}$ yields that the supremum tends to zero as $m,n\to\infty$ with the stated conditions. For the second term, observe that 
\begin{align}
    \rho^n(d_{n,m} \notin I) 
    &\overset{(a)}{=} 1-\rho^n(d_{n,m} \in I) \stackrel{(b)}{=} 1-\rho^n(|\mqn - \mq| + \mq \in [0,2\mq]) \\
    &\overset{(c)}{=} 1-\rho^n(|\mqn - \mq| \in [-\mq,\mq]) \\
    &\overset{(d)}{=} 1-\rho^n(|\mqn - \mq| \le \mq) \\
    &\overset{(e)}{=} \rho^n(|\mqn - \mq| >\mq) \overset{(f)}{\to} 0    
\end{align}
conditional on $(X_i)_{i=1}^\infty$ and $(I_j)_{j=1}^\infty$ as $n,m\to\infty$ with the stated conditions.
(a) follows by considering the complementary event, (b) is by the definitions of $d_{n,m}$ and $I$, in (c) we subtract $\mq$ from both sides, in (d) we notice that the l.h.s.\ is non-negative, and (e) follows by again considering the complement. To obtain the limit in (f), notice that $|\mqn - \mq| = o_P(1)$ for almost all $(X_i)_{i=1}^\infty$ and $(I_j)_{j=1}^\infty$ sequences in the assumed setting of \hl{Theorem~\ref{thm:limit-nystroem-ksd}} (making use of the assumption that \hl{$\covcn \neq 0$}) by Remark~\ref{remark:tilde-qw-op1}. The requirements imposed on $m$ in Theorem~\ref{thm:limit-nystroem-ksd} are implied by those imposed in Theorem~\ref{thm:separation-n-ksd} in (i) and (ii), respectively. 

This concludes the proof.

\acks{This work was supported by the pilot program Core-Informatics of the Helmholtz Association (HGF). BKS is partially supported by the National Science Foundation (NSF) CAREER award DMS-1945396 and NSF-DMS-2413425.}

\appendix
\setcounter{equation}{0}
\renewcommand{\theequation}{\thesection.\arabic{equation}} %

\section{Auxiliary Results} \label{sec:auxiliary-results}

This appendix collects our auxiliary results. In Lemma~\ref{lemma:triangle-ineq}, we show that $D_{\P_0}$ satisfies a triangle inequality. In Lemma~\ref{lemma:ksd-inequality}, we give a lower bound on $D^2_{\P_0}(\P)$. Lemma~\ref{lemma:bernstein-implies-psi1} shows that a random variable satisfying the Bernstein conditions enjoys finite $\psi_1$-norm. Lemma~\ref{lemma:equivalent-conditoin-C-zero} gives an equivalent characterization of when a covariance operator is not zero.

The following simple observation allows us to substantially shorten the proof of Theorem~\ref{thm:separation-ksd}.

\begin{lemmaA}[Triangle inequality for $D_{\P_0}$] \label{lemma:triangle-ineq} Let $\H_{K_0}$ be an RKHS on a set $\X$ with reproducing kernel $K_0$, $\P,\Q \in \mathcal M_1^+\left(\X\right)$, and suppose that $D_{\P_0}(\P) \coloneq \norm{\int_\X K_0(\cdot,x)\d \P(x)}{\H_{K_0}}< \infty$. Then it holds that $
		D_{\P_0}(\P) \le D_{\P_0}(\P-\Q) + D_{\P_0}(\Q)$.
\end{lemmaA}
\begin{proof}
	The kernel mean embedding can be defined on the finite signed measure $\P-\Q$ \citep{sejdinovic13kernel}. The statement follows by considering $\P = \left( \P - \Q \right) + \Q$, the linearity of the Bochner integral, and the triangle inequality holding for $\norm{\cdot}{\H_{K_0}}$.
\end{proof}

The following result states \citet[Lemma A.19]{hagrass22spectral} in terms of KSD.
\begin{lemmaA}[KSD lower bound] \label{lemma:ksd-inequality} Let $u_\P$, $D^2_{\P_0}$, and $\integralop$ be defined as in the main text and assume that $u_\P\in \range\!\left( \integralop^\theta \right)$. Then, for any $\theta>0$, it holds that
	\begin{align}
		D_{\P_0}^2(\P) \ge \norm{u_\P}{L^2(\X,\P_0)}^{\frac{2\theta+1}{\theta}}\norm{\integralop^{-\theta}u_\P}{L^2(\X,\P_0)}^{-\frac{1}{\theta}}.
	\end{align}
\end{lemmaA}
\begin{proof} 
    Observe that for $\P=\P_0$ the inequality holds. Hence, w.l.o.g.\ we assume that $\P\neq \P_0$ in the following.
    Let $\integralop = \sum_{j\in J} \lambda_j \tilde \phi_j\otimes_{L^2(\X,\P_0)} \tilde \phi_j$ be the spectral decomposition of $\integralop$ (in line with Lemma~\ref{lemma:range-space-equivalences}), $\big(\tilde \phi_j\big)_{j\in J}$ forming an $L^2(\X,\P_0)$-ONB of $\overline{\range\!\left(\integralop\right)}$ by the self-adjointness of $\integralop$; $\big(\tilde \phi_j, \lambda_j^\theta\big)_{j\in J}$ is an eigensystem of $\integralop^\theta$.
	Observe that we have
	\begin{align}
		\norm{u_\P}{L^2(\X,\P_0)}^2  &\overset{(a)}{=} \sum_{j\in J}\ip{u_\P,\tilde \phi_j}{L^2(\X,\P_0)}^2 \overset{(b)}{=} \sum_{j\in J}\ip{\integralop^\theta v,\tilde \phi_j}{L^2(\X,\P_0)}^2 
        \overset{(c)}{=} \sum_{j\in J}\ip{ v,\integralop^\theta\tilde \phi_j}{L^2(\X,\P_0)}^2\\
        &
        \overset{(d)}{=} \sum_{j\in J}\lambda_j^{2\theta}\ip{v,\tilde \phi_j}{L^2(\X,\P_0)}^2          \\
		 & \overset{(e)}{\le} \left(\sum_{j\in J}\lambda_j^{2\theta+1}\ip{v,\tilde \phi_j}{L^2(\X,\P_0)}^2\right)^{\frac{2\theta}{2\theta+1}}\left(\sum_{j\in J}\ip{v,\tilde \phi_j}{L^2(\X,\P_0)}^2\right)^{\frac{1}{2\theta +1}}\\
        &\overset{(f)}{=}\left(\sum_{j\in J}\ip{v,\integralop^{\theta+1/2}\tilde \phi_j}{L^2(\X,\P_0)}^2\right)^{\frac{2\theta}{2\theta+1}}\left(\sum_{j\in J}\ip{v,\tilde \phi_j}{L^2(\X,\P_0)}^2\right)^{\frac{1}{2\theta +1}}\\
        &\overset{(g)}{=}\left(\sum_{j\in J}\ip{\integralop^{\theta+1/2}v,\tilde \phi_j}{L^2(\X,\P_0)}^2\right)^{\frac{2\theta}{2\theta+1}}\left(\sum_{j\in J}\ip{v,\tilde \phi_j}{L^2(\X,\P_0)}^2\right)^{\frac{1}{2\theta +1}}\\
        &\overset{(h)}{=}\Bigg(\sum_{j\in J}\ip{\integralop^{1/2}\underbrace{\integralop^{\theta}v}_{u_\P},\tilde \phi_j}{L^2(\X,\P_0)}^2\Bigg)^{\frac{2\theta}{2\theta+1}}\left(\sum_{j\in J}\ip{v,\tilde \phi_j}{L^2(\X,\P_0)}^2\right)^{\frac{1}{2\theta +1}}\\
		&\overset{(i)}{=} \norm{\integralop^{1/2}u_\P}{L^2(\X,\P_0)}^{\frac{4\theta}{2\theta +1}}\norm{\integralop^{-\theta}u_\P}{L^2(\X,\P_0)}^{\frac{2}{2\theta + 1}}\\
        &\overset{(j)}{=} D^{\frac{4\theta}{2\theta +1}}_{\P_0}(\P)\norm{\integralop^{-\theta}u_\P}{L^2(\X,\P_0)}^{\frac{2}{2\theta + 1}}, \label{eq:u-bound}
	\end{align}
	where the details are as follows. In (a), we use Parseval's identity, which holds as $u_\P\in\range \!\left(\integralop^\theta\right)$ which implies that $u_\P\in\Span\big(\tilde \phi_j : j\in J\big)$. As $u_\P \in\range \!\left(T_{\P_0}^\theta\right)$, there exists $v\in L^2(\X,\P_0)$ such that $u_\P=\integralop^\theta v$, which gives (b). (c) comes from the definition of the adjoint operator and the self-adjointness of $\integralop^{\theta}$ following from that of $\integralop$. In (d), we use that $\big(\tilde \phi_j, \lambda_j^\theta\big)_{j\in J}$ is an eigensystem of $\integralop^{\theta}$, and the linearity of the inner product. 
	Recall that for real-valued sequences $(a_j)_{j}$, $(b_j)_{j}$  and $p,q\in[1,\infty]$ with $1/p + 1/q = 1$, by Hölder's inequality
	\begin{align}
		\sum_j|a_jb_j| \le \left( \sum_j |a_j|^p \right)^{1/p}\left( \sum_j |b_j|^q \right)^{1/q}. \label{eq:hoelder-inequality}
	\end{align}
	Setting $p = \frac{2\theta+1}{2\theta}$, $q=2\theta+1$, $a_j=\left( \lambda_j^{2\theta+1}\ip{v,\tilde \phi_j}{L^2(\X,\P_0)}^2 \right)^{\frac{2\theta}{2\theta+1}}$, and $b_j =  \ip{v,\tilde \phi_j}{L^2(\X,\P_0)}^{\frac{2}{2\theta +1}}$ in \eqref{eq:hoelder-inequality} yields (e).
    (f) follows from the fact that $\big(\tilde \phi_j,\lambda_j^{\theta+1/2}\big)_{j\in J}$ forms an eigensystem of $\integralop^{\theta+1/2}$ and the linearity of the inner product, (g) is implied by the definition of the adjoint operator and the self-adjointness of $\integralop^{\theta+1/2}$ following from that of $\integralop$, in (h) the definition of $u_\P$ is leveraged, (i) comes from Parseval's identity and the definition of $v$,
    (j) follows from the fact that $D^2_{\P_0}(\P) = \norm{\integralop^{1/2}u}{L^2(\P_0)}^2$ by \eqref{eq:explicit-statistic}.
    
    Raising the resulting inequality \eqref{eq:u-bound} to the power of $\frac{2\theta+1}{2\theta}$ gives
    \begin{align}
    \norm{u_\P}{L^2(\X,\P_0)}^{\frac{2\theta+1}{\theta}} &\le D_{\P_0}^2(\P) \norm{\integralop^{-\theta}u_\P}{L^2(\X,\P_0)}^{\frac{1}{\theta}},
    \end{align}
    which dividing by $\norm{\integralop^{-\theta}u_\P}{L^2(\X,\P_0)}^{\frac{1}{\theta}}$ (positive as $\P\neq\P_0$ was assumed; hence $u_\P\neq 0$) gives the claim.
\end{proof}

The next result gives a converse to Lemma~\ref{lemma:sub-exp-bernstein}.
\begin{lemmaA}[Bernstein condition implies finite $\psi_1$-norm] \label{lemma:bernstein-implies-psi1}
    If $\mathcal P \subset \mathcal M_1^+(\R)$ is such that $\sup_{\P\in\mathcal P} \E_\P |X|^r \le cr!\kappa^r$ for all $r\ge 2$ and some $c,\kappa >0$ (independent of $\P \in \mathcal P$), then $\sup_{\P\in\mathcal P}\psione{X} \le \kappa(1+\max((2c)^{1/2},c))$.
\end{lemmaA}
\begin{proof}
    We first show that $\sup_{\P\in\mathcal P} \E_{\P}|X| \le c_01!\kappa^1$ with $c_0 \coloneq \max((2c)^{1/2},c)$, which then implies that 
    \begin{align}
    \sup_{\P\in\mathcal P} \E_\P |X|^r \le c_0r!\kappa^r \text{ for all } r\ge 1, \label{eq:sub-exp-psi-one}
    \end{align}
    that is, we extend the range of $r$. Indeed, by Hölder's inequality,
    \begin{align*}
        \sup_{\P\in\mathcal P}\E_{\P}|X| \le \sup_{\P\in\mathcal P}\left(\E_{\P}|X|^2\right)^{1/2} \le  \sup_{\P\in\mathcal P}\left(c2!\kappa^2\right)^{1/2} \le c_01!\kappa^1.
    \end{align*}

    To now prove the claim, we recall that $\psione{X} = \inf\{t>0 : \E_\P\exp(|X|/t) \le 2\}$. Hence, to ensure that $\sup_{\P\in\mathcal P}\psione{X} \le t_0$, it suffices to find an absolute constant $t_0 > 0$ such that $\sup_{\P\in\mathcal P} \E_\P\exp(|X|/t_{0}) \le 2$. By the series representation of the exponential function, the sub-additivity of the supremum, and \eqref{eq:sub-exp-psi-one}, we have that
                \begin{align}
                    \sup_{\P\in\mathcal P}\E_\P\exp(|X|/t_0) &= 1+ \sup_{\P\in\mathcal P}\sum_{r\ge1}\frac{\E_\P|X|^r}{t_0^rr!} \le 1+\sum_{r\ge1}\frac{c_0r!\kappa^r}{t_0^rr!} = 1+c_0\sum_{r\ge1}\frac{\kappa^r}{t_0^r} \\
                    &= 1+\frac{c_0}{1-\kappa/t_0}-c_0, \label{eq:sup-bound2}
                \end{align}
                where we assumed that $\kappa/t_0 < 1$ for the geometric series to converge. Choosing $t_0 = \kappa(1+c_0)$, the requirement $\kappa/t_0 = \frac{1}{1+c_0}< 1$ holds since $c_0>0$. The r.h.s.\ of \eqref{eq:sup-bound2} is upper bounded by two as 
                \begin{align}
                1+\frac{c_0}{1-\kappa/t_0}-c_0\,\Big|_{t_0=\kappa(1+c_0)} = 2,
                \end{align}
                 proving the claim.
\end{proof}

We state a necessary and sufficient condition for the covariance operator to be non-zero in the following result.

\begin{lemmaA}[Characterization of non-zero covariance operator] \label{lemma:equivalent-conditoin-C-zero}
    Let $(\X,\tau_\X)$ be a topological space equipped with a kernel $k : \X \times \X \to  \R$, $X\sim\P\in \mathcal M_1^+(\X)$, and $C = \E_\P\!\left[ k(\cdot,X) \otimes k(\cdot,X)\right]$. Then $C \neq 0$ iff there exists $A\in \mathcal B(\tau_\X)$ with $\P(A) > 0$ such that $ k(x,x) > 0$ for all $x\in A$.
\end{lemmaA}
\begin{proof}
    We have the chain of equivalences
\begin{align}
    C = 0 \iff 0 &= \left\|C\right\|_1 \overset{(a)}{=} \trace\!\left(C\right) 
    = \trace \!\left(\E_{\P}\!\left[k(\cdot, X) \otimes k(\cdot, X)\right]\right)\\
    &\overset{(b)}{=}\E_{\P}\!\left[\trace \left(k(\cdot, X) \otimes k(\cdot, X)\right)\right] 
    \overset{(c)}{=} \E_\P \left[ k(X,X)\right] 
    \overset{(d)}{\iff} 
    k(X,X) = 0\quad\P\text{-a.s.},
\end{align}
by using the positivity of $C$ in (a), swapping the expectation and the trace in (b), invoking Lemma~\ref{lemma:norm-equiv}, the fact that in a Hilbert space the norm is induced by its inner product, and the reproducing property of $k$ in (c), and by using that  $k(x,x) = \norm{k(\cdot,x)}{\H_k}^2 \ge 0$ for all $x\in\X$ together with a property of the Lebesgue integral in (d). Considering the complement gives the claim.
\end{proof}

\section{Additional Results} \label{sec:external-statements}
In this appendix, we collect all the additional results that are needed to prove the main results of the paper. 

The asymptotic behavior of V-statistics is captured in Theorems~\ref{thm:serfling-A1} and \ref{thm:serfling-A2}. The weak convergence of the i.i.d.\ weighted (a.k.a.\  wild) bootstrap\textsuperscript{\ref{fn:dehling-mikosch}} is provided in Theorem~\ref{thm:dehling-v-stat-bootstrap}.
Lemma~\ref{lemma:norm-equiv} lists the equality of various norms and the trace of $f\otimes f$.
Lemma~\ref{lemma:bound-projection} gives a concentration result for the projection of the covariance operator. Lemma~\ref{lemma:sub-exp-bernstein} shows that sub-exponential random variables have Bernstein-type moment decay.
Lemma~\ref{lemma:sub-gauss-norm} bounds a sub-Gaussian norm.
Lemma~\ref{lemma:max-of-sub-gauss} bounds the maximum of sub-Gaussian random variables.   
We recall a few properties of $\psi_1$- and $\psi_2$-norms in Lemma~\ref{lemma:orlicz-properties}.
Theorem~\ref{thm:concentration-aistats} is a slight modification of \citet[Theorem~2]{kalinke25nystromksd}, following at once from its proof.
A concentration result for bounded random variables taking values in a separable Hilbert space is provided in Theorem~\ref{thm:bernstein-bounded}; it is a corollary to \citet[Theorem~3.5]{pinelis94concentration}. The complementary result (Theorem~\ref{thm:bernstein-unbounded}) for the unbounded case is quoted from \citet{sriperumbudur22approximate}.

We start by recalling the definitions used to obtain asymptotic results for V-statistics \citep[Chapter~6]{serfling80approximation} specialized to core functions of degree $2$.\footnote{The definitions in \citet[Ch.~6]{serfling80approximation} are stated in terms of distribution functions, but the results hold for distributions. We translate the definitions accordingly and refer to \citet{vaart98asymptotic} for a treatment of the von Mises calculus in this more general setting.} Notice that these allow a calculus inspired by Taylor approximations but permit handling probability distributions.
Let $h : \X \times \X \to \R$ be symmetric, $\P, \Q \in \mathcal M_1^+(\X)$, $T(\P) = \int_\X\int_X h(x,y)\d\P(x)\d\P(y)$, $(X_i)_{i=1}^\infty \overset{\text{i.i.d.}}{\sim} \P$,
\begin{align}
	d_1T(\P;\Q-\P) := \dfrac{\d }{\d \lambda} T(\P+\lambda(\Q - \P)) \Big |_{\lambda\downarrow 0}, &&
	d_2T(\P;\Q-\P) := \dfrac{\d^2 }{\d \lambda^2} T(\P+\lambda(\Q - \P)) \Big |_{\lambda\downarrow 0},
\end{align}
and $h(\P;x) = d_1T(\P;\delta_x-\P)$ for $x\in\X$. Further, one can guarantee that the following implicit definition of  $h(\P;x,y)$ is sensible:
\begin{equation}
	d_1T\!\left(\P;\hat \P_n - \P\right) + \frac{1}{2}d_2T\!\left(\P;\hat \P_n - \P\right) \eqcolon \frac{1}{n^2}\sum_{i,j=1}^n h(\P;X_i,X_j). \label{eq:h-def}
\end{equation}

Then, we say condition $\mathcal{A}_1$ holds if (i) $0 < \Var_\P(h(\P;X_1)) < \infty$ and (ii) $\sqrt n\big(T\big(\hat\P_n\big) - T(\P) - d_1T\big(\P;\hat \P_n-\P\big)\big) = o_P(1)$, and
condition $\mathcal{A}_2$ holds if (i) $\Var_\P(h(\P;X_1)) =0$, (ii) $\Var_{\P^2}(h(\P;X_1,X_2)) >0$, and (iii) $n\big(T\big(\hat\P_n\big) - T(\P) - n^{-2}\sum_{i,j=1}^n h(\P;X_i,X_j)\big) = o_P(1)$. 

These definitions are used in the next two statements.

\begin{theoremA}[Theorem A; Section 6.4.1; \citealt{serfling80approximation}] \label{thm:serfling-A1}
	Suppose that condition $\mathcal{A}_1$ holds. Let $\mu(T,\P) \coloneq \E_\P h(\P;X_1)$ and $\sigma^2(T,\P)\coloneq \Var_\P(h(\P;X_1))$. Then
	\begin{equation}
		\sqrt{n}\left(T\big( \hat \P_n \big) - T(\P) - \mu(T,\P)\right) \dconv \mathcal N(0,\sigma^2(T,\P)).
	\end{equation}
\end{theoremA}

\begin{theoremA}[Theorem B; Section 6.4.1; \citealt{serfling80approximation}] \label{thm:serfling-A2}
	Assume that condition $\mathcal{A}_2$ holds, $h(\P;x,y) = h(\P;y,x)$, $\E_{\P^{2}} h^2(\P;X_1,X_2) < \infty$, $\E_\P|h(\P;X_1,X_1)| < \infty$, and that $\E_\P h(\P;x,X_1)$ is constant (in $x$). Denote by $(\lambda_j)_{j\in J}$ the eigenvalues of the operator $A$ defined on $L^2(\X,\P)$ by
	\begin{equation}
		(Ag)(x) = \int_\X [h(\P;x,y)-\mu(T,\P)]g(y)\d \P(y)\quad\text{for } x\in \X,\; g\in L^2(\X,\P),
	\end{equation}
	where $\mu(T,\P) \coloneq \E_{\P^{2}} h(\P;X_1,X_2)$. Then
	\begin{equation}
		n\left(T\big(\hat \P_n\big) - T(\P) - \mu(T,\P)\right) \dconv \sum_{j\in J}\lambda_jZ_j^2,
	\end{equation}
	where $Z_1,Z_2,\ldots$ are i.i.d.\ standard normal.
\end{theoremA}

\begin{theoremA}[Theorem 3.1(Remark); \citealt{dehling94bootstrap}] \label{thm:dehling-v-stat-bootstrap}
	Denote by $\X$ a separable metric space. Let $X,X_1,X_2,\ldots \stackrel{\text{i.i.d.}}{\sim}\P \in \mathcal M_1^+(\X)$, $R_1,R_2,\ldots \stackrel{\text{i.i.d.}}{\sim} \rho$, where $\rho$ is the Rademacher distribution, and $h : \X\times \X \to \R$ symmetric and degenerate, in other words, $\operatorname{Var}_{X_1\sim\P}\left[\E_{X_2\sim\P}h(X_1,X_2)\right] = 0$. Assume that  $\E_{\P^{2}} h^2(X_1,X_2) < \infty$ and $\E_\P |h(X,X)| < \infty$. Then, for almost every realization $(x_n)_{n=1}^\infty$, it holds that
	\begin{align}
		\frac1n\sum_{i,j=1}^nR_iR_jh(x_i,x_j) \dconv \sum_{i=1}^\infty \lambda_iZ_i^2,
	\end{align}
	where $Z_1,Z_2,\ldots$ are i.i.d.\ standard normal and $(\lambda_i)_{i=1}^\infty$ are the eigenvalues of the Hilbert-Schmidt integral operator on $L^2(\P)$ given by 
		$Tf = \int_{\X}h(\cdot,y)f(y)\d \P(y)$. 
	\end{theoremA}

\begin{lemmaA}[Lemma B.8; \citealt{sriperumbudur22approximate}] \label{lemma:norm-equiv} Let $B = f \otimes f$, where $f\in\H$ and $\H$ is a separable Hilbert space. Then $\opnorm{B}=\hsnorm{B}{\H} = \trace (B) = \norm{f}{\H}^2$.
\end{lemmaA}

\begin{lemmaA}[Lemma B.1; \citealt{kalinke25nystromksd}] \label{lemma:bound-projection}
  Let Assumptions~\ref{ass:integrability} and \ref{ass:sub-gaussian} hold, and assume $0<\lambda\le\opnorm{C_{\P,\bar K_0}}$. Then, for any $\delta \in (0,1)$, it holds that
    \begin{align}
        \left(\P^n \otimes \Lambda^m\right)\left(\opnorm{\left(I-P_{\H_{K_0,m}}\right)C_{\P,\bar K_0,\lambda}^{1/2}}^2 \lesssim \lambda\right) \ge 1-\delta,
    \end{align}
    provided that $m\gtrsim \max\left\{ \frac{\trace\left(C_{\P,\bar K_0}\right)}{\lambda},1  \right\}\log\left(4/\delta\right)$.
\end{lemmaA}

\begin{lemmaA}[Lemma B.2; \citealt{kalinke25nystromksd}]\label{lemma:sub-exp-bernstein}
	Let $Y$ be a real-valued random variable which is sub-exponential, i.e., $\norm{Y}{\psi_1} < \infty$. Let
	  $ \sigma \coloneq \sqrt 2 \norm{Y}{\psi_1}$, $ B \coloneq \norm{Y}{\psi_1} >0$. Then the Bernstein condition
	  \begin{align}
		\E|Y|^p \le \frac{1}{2}p!\sigma^2B^{p-2} < \infty \label{eq:Bernstein}
		\end{align}
	  holds for any $p\ge 2$.
	\end{lemmaA}

\begin{lemmaA}[Lemma B.3; \citealt{kalinke25nystromksd}] \label{lemma:sub-gauss-norm} 
	Let $\H$ be a separable Hilbert space, $Y$ distributed with $\P \in \mathcal M_1^+(\H)$, and $A \in \mathcal L(\H)$ invertible and positive. Assume that $Y$ is sub-Gaussian, that is, $\norm{\ip{Y,u}{\H}}{\psi_2} \lesssim \norm{\ip{Y,u}{\H}}{L^2(\X,\P)}$ holds for all $u\in\H$. Then
	\begin{align}
	  \norm{\norm{A^{1/2}Y}{\H}}{\psi_2}^2 \lesssim \trace \left(A\E_{Y\sim\P} \left(Y\otimes Y\right)\right).
	\end{align}
	As immediate consequence under Assumption~\ref{ass:sub-gaussian}, choosing $A\coloneq  I$ and $Y\coloneq \fmc X$, and $A\coloneq C^{-1}_{\P,\bar K_0,\lambda}$ ($\lambda > 0$) and $Y\coloneq \fmc X$, respectively, it holds that
	\begin{align}
	  \norm{\norm{\fmc X}{\H_{K_0}}}{\psi_2} < \infty, && \text{and} && \norm{\norm{C_{\P,\bar K_0,\lambda}^{-1/2}\fmc X}{\H_{K_0}}}{\psi_2}^2 \lesssim \trace \left( C_{\P,\bar K_0,\lambda}^{-1}C_{\P,\bar K_0} \right) < \infty, \notag
	\end{align}
    that is, both  $\norm{\fmc X}{\H_{K_0}}$ and $\norm{C_{\P,\bar K_0,\lambda}^{-1/2}\fmc X}{\H_{K_0}}$ are sub-Gaussian.
 \end{lemmaA}

\begin{lemmaA}[Lemma B.5; \citealt{kalinke25nystromksd}]\label{lemma:max-of-sub-gauss}
    Let $\left(X_i\right)_{i=1}^n \stackrel{\text{i.i.d.}}{\sim} \P$ be real-valued sub-Gaussian random variables. Then $\P^n\!\left(\max_{i\in[n]}|X_i| \lesssim \sqrt{\norm{X_1}{\psi_2}^2\log(2n/\delta)}\right) \ge 1-\delta$ holds for any $\delta\in(0,1)$.
\end{lemmaA}

We refer to the following sources for the items in Lemma~\ref{lemma:orlicz-properties}, taken from and collected by \citet[Lemma~C.2]{kalinke25nystromksd}. Item 1 is \citet[Lemma~2.6.8]{vershynin18highdimensional}, Item 2 is \citet[Exercise~2.7.10]{vershynin18highdimensional}, Item 3 recalls \citet[p.~95]{vaart96weak}, and Item 4 is \citet[Lemma~2.7.6]{vershynin18highdimensional}.

\begin{lemmaA}[Collection of Orlicz properties]\label{lemma:orlicz-properties}
Let $X$ be a real-valued random variable.
  \begin{enumerate}
    \item If $X$ is sub-Gaussian, then $X-\E X$ is also sub-Gaussian, and
    \begin{align}
      \norm{X-\E X}{\psi_2} \le \norm{X}{\psi_2} + \norm{\E X}{\psi_2} \lesssim \norm{X}{\psi_2}.
    \end{align}
    \item  If $X$ is sub-exponential, then $X-\E X$ is also sub-exponential, and satisfies
    \begin{align}
      \norm{X-\E X}{\psi_1} \le \norm{X}{\psi_1} + \norm{\E X}{\psi_1} \lesssim \norm{X}{\psi_1}.
    \end{align}
    \item If $X$ is sub-Gaussian, it is sub-exponential. Specifically, $\norm{X}{\psi_1} \le \sqrt{\log (2)}\norm{X}{\psi_2}$.
    \item $X$ is sub-Gaussian if and only if $X^2$ is sub-exponential. Moreover,
    \begin{align}
      \norm{X^2}{\psi_1} = \norm{X}{\psi_2}^2.
    \end{align}
  \end{enumerate}
\end{lemmaA}

\begin{theoremA}[Theorem 2; \citealt{kalinke25nystromksd}] \label{thm:concentration-aistats} Suppose that Assumption~\ref{ass:integrability} and Assumption~\ref{ass:sub-gaussian} hold and that $\covc \neq 0$ (see Remark~\ref{remark:consistency-nksd}\ref{remark:item:cov-op-non-zero}). Then, for any $\delta \in (0,1)$, it holds that
	\begin{align}
		\left( \P^n\otimes\Lambda^m \right)\!\left( \rnorm{\kme{\P} - \projm\kme {\hat\P_n} } \gtrsim  \frac{\sqrt{\log(m) \mathcal{N}_{\bar K_0}\!\left(\frac{1}{m}\right)}\log(12n/\delta)}{m} \right) \le \delta,
	\end{align}
    given that
    \begin{align}
        m &\gtrsim \max\left\{\log(m)\opnorm{\covc}^{-1},\left(\tfrac{12}{\delta}\right)^{\tilde c_2\trace\left(\covc\right)},\log(12/\delta)\right\}, \text{ and} \label{eq:condition-m-appendix} \\
        n&\gtrsim \max\left\{ \log(6/\delta),\frac{m^2\opnorm{\covc}}{\log m}\right\};\label{eq:condition-n-appendix}
    \end{align}
    equivalently, it holds for any $t>0$ that
    \begin{align}
        \left( \P^n\otimes\Lambda^m \right)\!\left( \rnorm{\kme{\P} - \projm\kme {\hat\P_n} } > t \right) \le 12n\exp\!\left( -\frac{\tilde c_1mt}{\sqrt{\log(m)\mathcal N_{\bar K_0}\!\left( \frac 1m \right)}} \right),
    \end{align}
    where $\tilde c_1, \tilde c_2 > 0$ are constants that may depend on $K_0$ but are independent of $\P$.
\end{theoremA}
\begin{proof}
    In their proof, the authors choose the parameter $\lambda>0$ $\P$-dependent. To streamline our proofs, we simplify their result by imposing conditions on $n$ and adjust their result by setting $\lambda = \tfrac{\log (m)}{m}$, that is, our choice of $\lambda$ does \emph{not} depend on $\P$.

    We start with the \tb{simplification}. Let us define the terms
    \begin{align*}
        t_1 &\coloneq t_1(\P,K_0,n,\delta) \coloneq \frac{\sqrt{\trace\left(\covc\right)}\log(6/\delta)}{n}, \\
        t_2&\coloneq t_2(\P,K_0,n,\delta) \coloneq \sqrt{\frac{\trace\left(\covc\right)\log(6/\delta)}{n}}, \\
        t_3&\coloneq t_3(\P,K_0,n,\delta,\lambda,m) \coloneq \sqrt{\frac{\lambda\effdim(\lambda)\log^2(12n/\delta)}{m}}.
    \end{align*}
    Using these shorthands and that $\log(12n/\delta)\log(12/\delta) \le \log^2(12n/\delta)$, their result in the middle of page~17, below (18) in the cited work, shows that for any $\delta \in (0,1)$, it holds for some absolute constant $\tilde c>0$ that\footnote{While the authors state the result in terms of $|D_{\P_0}(\P)-\tilde D_{\P_0}\big(\hat \P_n\big)|$, the version stated here is implied by their bound using their decomposition (13).}
    \begin{align}
        (\P^n\otimes\Lambda^m)\left(\rnorm{\kme{\P} - \projm\kme {\hat\P_n} } \le \tilde c(t_1+t_2+t_3) \right) \ge 1-\delta,
    \end{align}
    given that $0<\lambda\le \opnorm{\covc}$ and $m\gtrsim\max\left\{\tfrac{\trace\big(\covc\big)}{\lambda},1\right\}\log(12/\delta)$.

    Our first goal is to derive conditions on $n$ such that the simplification $\tilde c(t_1+t_2+t_3) \le 3\tilde ct_3$ holds. Indeed, in step 1, we show that there exists $n_{0,1} \in \mathbb N_{>0}$ such that for any $n \ge n_{0,1}$, we have that $t_1 \le t_2$. Step 2 shows that there exists $n_{0,2}\in\mathbb N_{>0}$ such that for any $n\ge n_{0,2}$ it holds that $t_2\le t_3$. The combination of both results yields the simplification.
    \begin{itemize}
        \item \tb{Step 1.} Notice that 
        \begin{align}
            t_1 \le t_2 &\iff \frac{\sqrt{\trace\left(\covc\right)}\log(6/\delta)}{n} \le \sqrt{\frac{\trace\left(\covc\right)\log(6/\delta)}{n}} \iff \log(6/\delta) \le n,
        \end{align}
        that is, the inequality holds for all $n\ge n_{0,1} \coloneq \left\lceil \log(6/\delta)\right\rceil$.
        \item \tb{Step 2.} We have 
        \begin{align}
            t_2\le t_3 &\iff \sqrt{\frac{\trace\left(\covc\right)\log(6/\delta)}{n}} \le \sqrt{\frac{\lambda\effdim(\lambda)\log^2(12n/\delta)}{m}} \\
            &\iff \frac{\trace\left(\covc\right)\log(6/\delta)}{n} \le \frac{\lambda\effdim(\lambda)\log^2(12n/\delta)}{m} \\
            &\overset{(a)}{\iff}   \frac{m\trace\left(\covc\right)\log(6/\delta)}{\lambda\effdim(\lambda)\log^2(12n/\delta)} \le n \\
            &\overset{(b)}{\impliedby}  \frac{m\trace\left(\covc\right)}{\lambda\effdim(\lambda)} \le n \overset{(c)}{\impliedby}  \frac{m\trace\left(\covc\right)}{\lambda\frac{\trace\left(\covc\right)}{2\opnorm{\covc}}} \le n \iff \frac{2m\opnorm{\covc}}{\lambda} \le n,
        \end{align}
        where we used in (a) that $\effdim(\lambda)>0$ for all $\lambda>0$ (implied by the explicit form of $\effdim(\lambda)$ stated in footnote~\ref{fn:effective-dimension-explicit} and the condition \hl{$\covc \neq 0$} imposed),  (b) holds as $\log(6/\delta) < \log(12n/\delta)$ and as $\log(12n/\delta) \ge 1$, and (c) follows from 
        \begin{align}
            \frac{\trace\left(\covc\right)}{2\opnorm{\covc}} \le \effdim(\lambda)
        \end{align}
        holding for any $0 < \lambda \le \opnorm{\covc}$ \citep[Lemma~B.7(1.)]{kalinke25nystromksd}. Hence, $t_2\le t_3$ is guaranteed by choosing $n\ge n_{0,2} \coloneq \left\lceil \tfrac{2m\opnorm{\covc}}{\lambda} \right\rceil$.
        \item \tb{Combination.} The combination of both steps yields that 
    \begin{align}
        (\P^n\otimes\Lambda^m)\left(\rnorm{\kme{\P} - \projm\kme {\hat\P_n} } \le 3\tilde ct_3 \right) \ge 1-\delta, \label{eq:aistats-simplification}
    \end{align}
    given that all of (i) $0<\lambda\le \opnorm{\covc}$, (ii) $m\gtrsim\max\left\{\tfrac{\trace\big(\covc\big)}{\lambda},1\right\}\log(12/\delta)$, and (iii) $n\ge \max\{n_{0,1},n_{0,2}\}$ are satisfied.
    \end{itemize}

    It remains to verify that for large enough $m$ the conditions (i) and (ii) are satisfied with \tb{our choice of $\lambda = \tfrac{\log (m)}{m}$}. (i) is equivalent to  $m \ge \log(m)\opnorm{\covc}^{-1}$ using that by assumption \hl{$\covc \neq 0$}; an asymptotic consideration shows that the inequality can always be satisfied by choosing $m$ large enough. To satisfy (ii), one can take $m\gtrsim \max\left\{\left(\tfrac{12}{\delta}\right)^{\tilde c_2\trace\left(\covc\right)},\log(12/\delta)\right\}$ for some absolute constant $\tilde c_2 > 0$.
    Therefore, $m \gtrsim \max\left\{\log(m)\opnorm{\covc}^{-1},\left(\tfrac{12}{\delta}\right)^{\tilde c_2\trace\left(\covc\right)},\log(12/\delta)\right\}$ satisfies both conditions. Requirement (iii) for the chosen $\lambda$ becomes 
    \begin{align}
     n\gtrsim \max\left\{ \log(6/\delta),\frac{m^2\opnorm{\covc}}{\log (m)}\right\}.
    \end{align}
    Hence, with our choice of $\lambda$, and relaxing $\mathcal{N}_{\bar K_0}\!\left(\tfrac{\log (m)}{m}\right) \le \mathcal{N}_{\bar K_0}\!\left(\tfrac{1}{m}\right)$ as in \eqref{eq:effective-dim-relaxation}, we get from \eqref{eq:aistats-simplification} that
    \begin{align}
    (\P^n\otimes\Lambda^m)\left(\rnorm{\kme{\P} - \projm\kme {\hat\P_n} } \le 3\tilde c \frac{\sqrt{\log(m) \mathcal{N}_{\bar K_0}\!\left(\frac{1}{m}\right)}\log(12n/\delta)}{m}\right) \ge 1-\delta,
    \end{align}
    which gives the first stated result by considering the complement. 
    Solving the equation 
    \begin{align}
     t & = 3\tilde c \frac{\sqrt{\log(m) \mathcal{N}_{\bar K_0}\!\left(\frac{1}{m}\right)}\log(12n/\delta)}{m} \iff \delta = 12 n \exp\left(-\frac{t m}{3 \tilde c \sqrt{\log(m) \mathcal{N}_{\bar K_0}\!\left(\frac{1}{m}\right)}}\right)
    \end{align}
    gives the second stated result after defining $\tilde{c}_1 \coloneq \frac{1}{3\tilde c}>0$.
\end{proof}

\begin{theoremA}[Corollary A.5.2; \citealt{mollenhauer21statistical}]
	\label{thm:bernstein-bounded}
	Let $\left( X_i \right)_{i=1}^n$ be centered independent random variables taking values in a separable Hilbert space $\left(\H,\norm{\cdot}{\H}\right)$ such that $\max_{i\in[n]}\norm{X_i}{\H} \le b$ almost surely, for some  $b > 0$. Then for any $\delta \in (0,1)$, it holds with probability at least $1-\delta$ that
	\begin{align*}
	  \norm{\frac 1n \sum_{i=1}^nX_i}{\H}\le b\frac{\sqrt{2\log (2/\delta)}}{\sqrt n}.
	\end{align*}
  \end{theoremA}

\begin{theoremA}[Theorem 3.3.4; \citealt{yurinsky95sums}] \label{thm:bernstein-unbounded}
  Let $\left( \Omega,\mathcal A,\P \right)$ be a probability space, $\H$ a separable Hilbert space, $B> 0$, $\sigma >0$, and $\eta_{1},\ldots,\eta_{n} : \Omega\to \H$ centered i.i.d.\ random variables that satisfy
  \begin{align*}
    \E\norm{\eta_{1}}{\H}^{p} \le \frac12p!\sigma^2B^{p-2}
  \end{align*}
  for all $p\ge 2$. Then, for any $\delta \in (0,1)$ it holds with probability at least $1-\delta$ that
  \begin{align*}
    \norm{\frac1n\sum_{i=1}^{n}\eta_{i}}{\H} \le \frac{2B\log(2/\delta)}{n}+\sqrt{\frac{2\sigma^{2}\log(2/\delta)}{n}}.
  \end{align*}
\end{theoremA}

\vskip 0.2in
\bibliography{bib/collected_Florian,bib/publications}

\end{document}

%% file: preamble.tex
\usepackage{mathtools}
\mathtoolsset{showonlyrefs}
\usepackage{amsmath,amssymb}
\usepackage{bm}
\usepackage{dsfont}
\usepackage{enumitem}
\usepackage[final]{showlabels} %
\usepackage[dvipsnames]{xcolor}         %
\usepackage{booktabs}
\usepackage{tablefootnote} %

\newcommand{\hl}[1]{#1}
\newcommand*{\T}{^{\mkern-1.5mu\mathsf{T}}} %
\newcommand{\E}{\mathbb{E}}   %
\newcommand{\D}{D_{\mathbb{P}_0}(\P)} %
\newcommand{\I}{\mathfrak{I}}   %
\newcommand{\G}{G_{\rho^n,\P^n}}%
\newcommand{\Q}{\mathbb{Q}}   %
\newcommand{\R}{\mathbb{R}}   %
\newcommand{\X}{\mathcal{X}}  %
\newcommand{\ip}[2]{\left\langle{#1}\right\rangle_{#2}} %
\newcommand{\rip}[1]{\left\langle{#1}\right\rangle_{\H_{K_0}}} %
\newcommand{\norm}[2]{\left\|{#1}\right\|_{#2}} %
\newcommand{\opnorm}[1]{\left\|{#1}\right\|_{\mathrm{op}}} %
\newcommand{\hsnorm}[2]{\left\|{#1}\right\|_{#2\otimes#2}} %
\newcommand{\rnorm}[1]{\left\|{#1}\right\|_{\H_{K_0}}} %
\newcommand{\psione}[1]{\left\|{#1}\right\|_{\psi_1}} %
\newcommand{\psitwo}[1]{\left\|{#1}\right\|_{\psi_2}} %
\newcommand{\tb}{\textbf}    %
\newcommand{\fm}[1]{K_0\!\left(\cdot,{#1}\right)} %
\newcommand{\fmc}[1]{\bar K_0\!\left(\cdot,{#1}\right)} %
\newcommand{\kme}[1]{\mu_{K_0}\!\left({#1}\right)}

\renewcommand{\H}{\mathcal{H}} %
\renewcommand{\O}{\mathcal{O}} %
\renewcommand{\P}{\mathbb{P}} %
\renewcommand{\b}{\mathbf}    %
\renewcommand{\d}{\mathrm{d}} %
\renewcommand{\L}{\mathcal{L}} %
\newcommand{\dconv}{\leadsto}

\newcommand{\mq}{q_{W,1-\alpha}}
\newcommand{\mqn}{\tilde q_{W,1-\alpha,n,m}}
\newcommand{\mjoint}{\left(\P^n\otimes \Lambda^m\otimes \rho^n\right)}
\newcommand{\effdim}{\mathcal N_{\bar K_0}}

\DeclareMathOperator{\trace}{tr}
\DeclareMathOperator{\range}{ran}

\DeclareMathOperator{\Var}{Var}

\newcommand{\fmi}{K_0\!\left(\cdot,X_i\right)}
 
\newcommand{\fmij}{K_0\!\left(\cdot,X_{I_j}\right)}
\newcommand{\fmxi}{K_0\!\left(\cdot,x_i\right)}

\newcommand{\Knn}{\mathbf {G}_{n,n}}
\newcommand{\Knm}{\mathbf {G}_{n,m}}
\newcommand{\Kmm}{\mathbf {G}_{m,m}}
\newcommand{\Kmn}{\mathbf {G}_{m,n}}
\newcommand{\covc}{C_{\P,\bar K_0}}
\newcommand{\covcn}{C_{\P_0,\bar K_0}}
\newcommand{\cov}{C_{\P,K_0}}
\newcommand{\covn}{C_{\P_0,K_0}}
\newcommand{\covcr}{C_{\P,\bar K_0,\lambda}}

\newcommand{\projm}{P_{0,m}}
\newcommand{\integralop}{T_{\P_0}}
\newcommand{\mC}{\mathcal C}

\newcommand{\Span}{\mathrm{span}} %

\newtheorem{assumption}{Assumption}

\newtheorem{theoremA}{Theorem}[section] %
\newtheorem{lemmaA}[theoremA]{Lemma} %

\setenumerate{labelindent=0em,leftmargin=2em,topsep=0cm,partopsep=2mm,parsep=0cm,itemsep=1mm}
\setitemize{labelindent=0em,leftmargin=1.3em,topsep=0cm,partopsep=0cm,parsep=0cm,itemsep=1mm}

\newcommand{\Np}{\mathbb{N}_{>0}} %